\documentclass[a4paper, 11pt]{article}
\usepackage[margin=2.5cm,tmargin=2.5cm,bmargin=2.5cm]{geometry}

\usepackage{mystyle}

\usepackage{todonotes}
\usepackage{caption}
\usepackage{subcaption}
\usepackage[utf8]{inputenc}
\usepackage[T1]{fontenc}
\usepackage{adjustbox}
\usepackage[toc,page]{appendix}
\usepackage{comment}

\usepackage[backend=bibtex,style=alphabetic,]{biblatex} 
\newcommand{\fmap}{\phi}

\newcommand{\activation}{\sigma}

\newcommand{\loc}{{loc}}

\newcommand{\diag}{\mathrm{diag}}

\newcommand{\Leb}{\mathrm{Leb}}

\newcommand{\relu}{\mathrm{ReLU}}
\newcommand{\gelu}{\mathrm{GeLU}}
\newcommand{\swish}{\mathrm{Swish}}

\newcommand{\NODE}{\ensuremath{\mathrm{NODE}}}
\newcommand{\NODEs}{\ensuremath{\mathrm{NODEs}}}

\newcommand{\Attention}{\ensuremath{\mathrm{Attention}}}

\title{Understanding the training of infinitely deep and wide ResNets \\ with Conditional Optimal Transport}

\author{
    Raphaël Barboni\\
    ENS -- PSL Université \\
    \texttt{raphael.barboni@ens.fr}
    \and
    Gabriel Peyré\\
    CNRS and ENS -- PSL Univ.\\
    \texttt{gabriel.peyre@ens.fr}
    \and
    Fran\c{c}ois-Xavier Vialard\\
    LIGM, Univ. Gustave Eiffel, CNRS\\
    \texttt{francois-xavier.vialard@univ-eiffel.fr}
}
\date{}

\begin{document}

\maketitle

\begin{abstract}
We study the convergence of gradient flow for the training of deep neural networks.
While Residual Neural Networks are a popular example of very deep architectures, their training constitutes a challenging optimization problem notably due to the non-convexity and the non-coercivity of the objective.
Yet, in applications, such tasks are successfully solved by simple optimization algorithms such as gradient descent.
To better understand this phenomenon, we focus here on a ``mean-field'' model of infinitely deep and arbitrarily wide ResNet, parameterized by probability measures on the product set of layers and parameters and with constant marginal on the set of layers.
Indeed, in the case of shallow neural networks, mean field models have been proven to benefit from simplified loss landscapes and good theoretical guarantees when trained with gradient flow for the Wasserstein metric on the set of probability measures.
Motivated by this approach, we propose to train our model with gradient flow w.r.t. the conditional Optimal Transport distance: a restriction of the classical Wasserstein distance which enforces our marginal condition.
Relying on the theory of gradient flows in metric spaces we first show the well-posedness of the gradient flow equation and its consistency with the training of ResNets at finite width.
Performing a local Polyak-\L{}ojasiewicz analysis, we then show convergence of the gradient flow for well-chosen initializations: if the number of features is finite but sufficiently large and the risk is sufficiently small at initialization, the gradient flow converges to a global minimizer.
This is the first result of this type for infinitely deep and arbitrarily wide ResNets.
In addition, this work is an opportunity to study in more detail the conditional Optimal Transport metric, particularly its dynamic formulation. Some of our results in this direction might be interesting on their own.\footnote{%
Acknowledgements: The work of G. Peyr\'e was supported by the French government under management of Agence Nationale de la Recherche as part of the
``Investissements d'avenir'' program, reference ANR19-P3IA-0001 (PRAIRIE 3IA Institute). 
}

\noindent
\textit{Keywords:} Residual neural networks, Optimal Transport, gradient flows in Wasserstein space, Polyak-\L{}ojasiewicz inequality, mean-field analysis.
\end{abstract}

\section{Introduction}

Understanding the training dynamic of neural networks is an important problem in Machine Learning as it brings the hope of understanding the good performance of these models.
This training is however a complex optimization problem, usually solved by performing (stochastic) gradient descent for the training risk, an optimization procedure that, though simple, often manages to find a global minimum of the risk despite its non-convexity.
This phenomenon is now correctly understood in some simple cases such as the one of linear networks~\parencite{hardt_identity_2018,bartlett2018gradient,zou2019global,bah2022learning}.
In the more realistic case of non-linear architectures, most works have focused on  \emph{Multi-Layer Perceptrons (MLP)}~\parencite{li2017convergence,du2019gradient,allen2019convergence,zou_gradient_2020,lee2019wide,chen2020much,nguyen_proof_2021} and convergence to a minimizer of the risk can be obtained with great probability over a random initialization provided that the network is sufficiently wide, a regime referred to as ``overparameterization''.
Taking the limit of infinite width, many works have also studied the convergence of gradient descent for the training of neural networks in the limit of an infinite number of parameters~\parencite{chizat2018global,mei2018mean,javanmard_analysis_2020,wojtowytsch2020convergence,nguyen2023rigorous}. In those works, the neural network is trained by modeling the parameters as a probability measure on the parameter space and performing a Wasserstein gradient flow over the set of probability measures. Notably, \textcite{chizat2018global} establish a result of optimality at convergence: if the gradient flow converges then its limit is a global minimizer of the training risk.

We will focus in this work on the case of the \emph{Residual Neural Network (ResNet)} architecture.
ResNets were first introduced by~\textcite{he2016deep} for applications in computer vision but the architecture has since distinguished itself by obtaining state-of-the-art results in several other machine-learning applications. A key feature of ResNets is the extensive use of \emph{skip connections}: each layer consists of the addition of a 
perturbation (called \emph{residual}) to the output of the previous layer.
The presence of skip connections has indeed been shown to ease the training of deeper neural networks~\parencite{raiko2012deep,szegedy2017inception} by mitigating the \emph{vanishing / exploding gradient} phenomena, a common problem encountered when training deep neural networks~\parencite{Bengio94learning,glorot2010understanding}. The ResNet architecture has thus permitted the training of neural networks of almost arbitrary depth~\parencite{he2016identity}.
Considering the limit where the depth tends to infinity~\textcite{chen2018neural} introduced the \emph{Neural Ordinary Differential Equation (\NODE{})} architecture: ResNets being seen as the \emph{explicit Euler} discretization scheme of an ODE, passing to the limit of infinite depth leads to a model performing the integration of an ODE with a parametric velocity field.
An important contribution of \NODEs{} is to provide a theoretical framework upon which many other works have been based to study very deep neural network architectures.
\textcite{chen2018neural} proposed a method based on \emph{adjoint sensitivity analysis} to compute the gradient of \NODEs{} efficiently without automatic differentiation.
\textcite{sander_momentum_2021} proposed a new architecture based on a second-order ODE which can be trained with reduced computational complexity.
Inspired by methods from medical imaging and shape analysis, \textcite{vialard2020shooting} proposed a new algorithm for the training of deep ResNets.
\textcite{e_mean-field_2019,e_barron_2021} studied the training and generalization properties of deep ResNets borrowing tools from the mathematical theory of \emph{Optimal Control}.

\paragraph{Notations}

For a metric space $X$,  $\Pp(X)$ is the set of Borel probability measures on $X$.
This set is endowed with the \emph{narrow topology}, i.e. the topology of convergence against the set $\Cc_b(X)$ of bounded continuous functions.
For $x \in X$, we denote by $\delta_x \in \Pp(X)$ the Dirac measure at $x$. For $p \geq 1$, $\Pp_p(X)$ is the subset of $\Pp(X)$ of probability measures with finite $p$-order moment, endowed with the \emph{Wasserstein distance} $\Ww_p$~\parencite{villani2009optimal,santambrogio2015optimal}.
When $X$ is a Hilbert space, we define on $\Pp_p(X)$ the $p$-Energy $\Ee_p(\mu) \eqdef \int_X |x|^p \d \mu(x)$.
If $\mu \in \Pp(X)$ and $f : X \mapsto Y$ is a measurable map between topological spaces we denote by $f_\# \mu \in \Pp(Y)$ the pushforward of $\mu$ by $f$.
If $\lbrace f_i : X_i \to Y_i \rbrace_{1 \leq i \leq n}$ is a family of mappings then $(f_1, ..., f_n)$ designates the product map $(f_1, ..., f_n) : (x_1, ..., x_n) \mapsto (f_1(x_1), ..., f_n(x_n))$ and if $X = X^1 \times ... \times X^n$ is a product space we designate by $\pi^i$ the projection $\pi^i : (x^1, ..., x^n) \in X \mapsto x^i \in X^i$.

\subsection{Mean-field models of neural networks}

We consider in this paper neural networks parameterized by probability measures.
Such ``mean-field models'' can be thought of as neural networks of arbitrary width and were studied before by \textcite{chizat2018global,mei2018mean,rotskoff2018parameters,wojtowytsch2020convergence,nguyen2023rigorous}.
Provided with the \emph{parameter space} $\Om \subset \RR^p$, the input space $\RR^d$ and with a Borel map $\fmap : \Om \times \RR^d \to \RR^d$ (the \emph{feature map}), we consider mappings $F_\mu : \RR^d \to \RR^d$ parameterized by measures $\mu \in \Pp(\Om)$ and defined by:
\begin{equation} \label{measure_parameterization}
    \forall \mu \in \Pp(\Om), \quad F_\mu : x \in \RR^d \mapsto \int_\Om \fmap(\om, x) \d \mu(\om) .
\end{equation}

\paragraph{Single hidden layer perceptron}

The above definition encompasses as a particular case standard neural network architectures. For example, for $\Om = \RR^d \times \RR^d \times \RR$ and $\fmap : ((u,w,b), x) \mapsto u \activation (w^\top x + b)$ with a real-valued function $\activation : \RR \to \RR$ (called \emph{activation}), considering the atomic measure $\mu = \frac{1}{m} \sum_{i=1}^m \delta_{(u_i, w_i, b_i)}$ one recovers the classical model of a \emph{Single Hidden Layer (SHL) Perceptron} of width $m \geq 1$:
\begin{equation*}
    F_\mu : x \mapsto \frac{1}{m} \sum_{i=1}^m u_i \activation (w_i^\top x + b_i).
\end{equation*}

\paragraph{Convolutional Layer}

Closer to applications, \cref{measure_parameterization} also encompasses the convolutional architecture originally used in ResNet for image recognition~\parencite{he2016deep}. Consider integers $n,c,m, k \geq 0$ and $\Om = \RR^{c \times 1 \times k \times k} \times \RR^{1 \times c \times k \times k} \times \RR^{1 \times n \times n}$, the set of parameters of the form $(u,w, b)$ where $u$ and $w$ are convolutional kernels of size $c \times 1 \times k \times k$ and $1 \times c \times k \times k$ respectively and $b$ is a bias term of size $1 \times n \times n$. Then for an image input $x \in \RR^{c \times n \times n}$ of size $n \times n$ with $c$ channels, and $(u,w, b) \in \Om$ consider the feature map $\varphi : ((u,w,b),x) \mapsto u \star \activation ( w \star x + b) \in \RR^{c \times n \times n}$ where $\star$ denotes the discrete convolution operator and the activation $\activation$ is applied component-wise. Then for an empirical measure $\mu = \frac{1}{m} \sum_{i=1}^m \delta_{(u_i,w_i, b_i)}$, \cref{measure_parameterization} gives on input $x$:
\begin{align*}
    F_\mu(x) = \frac{1}{m} \sum_{i=1}^m u_i \star \activation (w_i \star x + b_i) ,
\end{align*}
which is the output of a ResNet residual with $m$ intermediary channels (without normalization)~\parencite{he2016deep}. However, the definition in~\cref{measure_parameterization} does not model some popular architectures such as normalization layers which play an important role in the success of ResNets.

\paragraph{Attention layers}

Finally, \cref{measure_parameterization} also models \emph{attention layers} at the heart of \emph{Transformers} architectures~\parencite{vaswani2017attention}. Consider as parameter space $\Om = \RR^{c \times d} \times \RR^{c \times d} \times \RR^{d \times d}$, the set of triplets $(K, Q, V)$ where $K \in \RR^{c \times d}$ is the \emph{key} matrix, $Q \in \RR^{c \times d}$ is the \emph{query} matrix and $V \in \RR^{d \times d}$ is the \emph{value} matrix. One can define the feature map $\fmap$ as an \emph{attention} layer with parameter $(K, Q, V) \in \Om$. For an input sequence of \emph{tokens} $\xx = (x^i)_{1 \leq i \leq N} \in (\RR^d)^N$ of length $N \geq 0$:
\begin{align*}
    \fmap((K, Q, V), \xx) \eqdef \Attention((K, Q, V), \xx) =  \left( \sum_{j=1}^N \frac{e^{\langle K x^i, Q x^j \rangle}}{ \sum_{j=1}^N e^{\langle K x^i, Q x^j \rangle}} V x^j \right)_{1 \leq i \leq N} \in (\RR^d)^N\,.
\end{align*}
Then for an empirical measure $\mu = \frac{1}{m} \sum_{k=1}^m \delta_{(K_k, Q_k, V_k)}$, \cref{measure_parameterization} defines a \emph{multi-head attention} layer with $m$ heads:
\begin{align*}
    F_\mu(\xx) = \frac{1}{m} \sum_{k=1}^m \Attention((K_k, Q_k, V_k), \xx) \, .
\end{align*}
Note that with this definition, we are only able to describe Transformer architectures taking as input sequences of tokens of fixed length $N$.
However, our setting could be adapted to model Transformer architecture taking as input sequences of tokens of various finite lengths by considering different feature maps depending on the length of the input sequence.

\paragraph{Functional properties}

\textcite{weinan2022representation} study the class of mappings defined by~\cref{measure_parameterization} from a functional analysis viewpoint.
More precisely, they study the ``Barron space''
$$\Bb \eqdef \lbrace f: x \mapsto \int u \activation (w^\top x + b) \d \mu (u,w,b), \, \mu \in \Pp_2(\RR \times \RR^d \times \RR) \rbrace$$
in the case $\activation = \relu$ and show $\Bb$ can be endowed with a norm
\begin{align*}
    \forall f \in \Bb, \quad \| f \|_\Bb \eqdef \inf \lbrace \int | u | ( \| w \| + |b| ) \d \mu, \, \mu \in \Pp_2(\RR \times \RR^d \times \RR), \, f = F_\mu \rbrace ,   
\end{align*}
for which it is a Banach space, continuously embedded in the space of Lipschitz functions.
In fact, $\Bb$ is the smallest Banach space containing all single hidden layer neural networks of finite width~\parencite[Thm.3.7 and Thm.3.8]{weinan2022representation} and it is thus dense in the set of continuous functions equipped with the compact-open topology~\parencite{cybenko1989approximation}. Note however that the case of convolutional or attention layers is different as the maps $F_\mu$ defined by~\cref{measure_parameterization} are then equivariant by translation of the input images or permutation of the tokens respectively.

In our case, due to~\cref{fmap_assumption1} on $\fmap$, we have that if $f = F_\mu$ is given by~\cref{measure_parameterization} then $\| f \|_\Bb \leq \Ee_2(\mu)$ and the second-order moment $\Ee_2(\mu)$ will be a key ingredient in controlling the training dynamic of our model (\emph{e.g.} by~\cref{flow_growth}). Finally, note that~\cref{measure_parameterization} is not the only way to represent neural networks of arbitrary width and \textcite{weinan2022representation} propose various equivalent representations of the functional space $(\Bb, \|.\|_\Bb)$.

\subsection{Mean-field \NODEs{}}

We then proceed to the definition of \emph{Neural ODEs (\NODEs{})} modeling ResNets whose depth tends to infinity with a proper rescaling of the residual layers~\parencite{chen2018neural}.
Our \NODE{} model is an ODE whose velocity field (or \emph{residual}) belongs to the class of mappings parameterized by measures defined in~\cref{measure_parameterization}.
Similar models of ``mean-field \NODEs{}'' or ``mean-field limit of ResNets'' were studied by~\textcite{lu2020mean,ding2022overparameterization,isobe2023convergence}.

\begin{defn}[\NODE{}] \label{def:NODE}
    For a family of probability measures $\mu = \lbrace \mu(.|s) \rbrace_{s \in [0, 1]} \in \Pp(\Om)^{[0,1]}$ and input $x \in \RR^d$ we define the output of the \NODE{} model as $\NODE_\mu(x) \eqdef x_\mu(1)$ where $(x_\mu(s))_{s \in [0,1]}$ satisfies the \emph{Forward ODE}:
    \begin{equation} \label{forward}
        \frac{\d}{\d s} x_\mu(s) = F_{\mu(.|s)}(x_\mu(s)), \quad x_\mu(0) = x.
    \end{equation}
    When there is no ambiguity, we simply write $x(s)$.
\end{defn}

\paragraph{The parameter set $\Pp_2^\Leb([0,1] \times \Om)$}

To justify the well-posedness of~\cref{forward} it is first necessary to define the adequate set of parameters we will consider. Given a topological space $Z$, we define $\Pp^\Leb_2([0,1] \times Z)$ as the set of probability measures $\mu \in \Pp_2([0,1] \times Z)$ whose marginal w.r.t. $[0,1]$ is the Lebesgue measure $\Leb([0,1])$:
\begin{align*}
    \Pp^\Leb_2([0,1] \times Z) \eqdef \enscond{\mu \in \Pp_2([0,1] \times Z)}{\pi^1_\#  \mu = \Leb([0,1])} .
\end{align*}
Given $\mu \in \Pp_2^\Leb([0,1] \times \Om)$, using a result of disintegration of measures~\parencite[Thm.4.2.4]{attouch2014variational}, there exists a $\d s$-a.e. uniquely determined family of probability measure $\mu(.|s) \in \Pp_2(Z)$ such that for every measurable $f : [0, 1] \times Z \to \RR$:
$$
     \text{$s \in [0,1] \mapsto \int_Z f(s, z) \d \mu(z|s)$ is measurable and $\int_{[0, 1] \times Z} f(s, z) \d \mu(s,z) = \int_0^1 \int_Z f(s, z) \d \mu(z|s) \d s$.}
$$
In the following, we will consider as parameters probability measures $\mu \in \Pp^\Leb_2([0, 1] \times \Om)$. Therefore, every parameter $\mu \in \Pp^\Leb_2([0, 1] \times \Om)$ is naturally associated with a (almost everywhere uniquely defined) family of probability measures $\lbrace \mu(.|s) \rbrace_{s\in[0, 1]}$.
We will provide this set of parameters with a modification of the Wasserstein-$2$ distance~\parencite{villani2009optimal,santambrogio2015optimal} that takes into account the marginal constraint by considering a restriction of Kantorovich's original optimal coupling problem to the set of couplings that are the identity on the first variable $s \in [0,1]$. The solution of this new optimization problem induces the \emph{Conditional Optimal Transport} (COT) distance on the parameter set~\parencite{hosseini2023conditional}.

\paragraph{Well-posedness of \NODEs{}}

The following assumption on $\fmap$ will be sufficient to show the well-posedness of~\cref{forward} for any parameter $\mu \in \Pp_2^\Leb([0,1] \times \Om)$. This is the content of~\cref{prop:flow_wellposed}.

\begin{assumption} \label{fmap_assumption1}
    Assume $\fmap : \Om \times \RR^d \to \RR^d$ is measurable and 
    \begin{enumerate}
        \item (quadratic growth) grows at most quadratically w.r.t. $\om$ and linearly w.r.t. $x$: there exists a constant $C$ s.t.
        \begin{align*}
            \forall x \in \RR^d, \om \in \Om, \quad \|\fmap(\om, x) \| \leq C (1 + \|x\|) (1+\|\om\|^2)\,.
        \end{align*}
        \item (local Lipschitz continuity) is locally Lipschitz with respect to $x$ with a Lipschitz constant that grows at most quadratically with $\om$: for every $R \geq 0$, there exists a constant $C(R)$ s.t.
        \begin{align*}
            \forall x, x' \in B(0, R), \, \forall \om \in \Om, \quad \| \fmap(\om, x) - \fmap(\om, x') \| \leq C(R) (1+\|\om\|^2) \| x-x'\| \,. 
        \end{align*}
    \end{enumerate}
\end{assumption}

\begin{prop}[Well-posedness of the flow] \label{prop:flow_wellposed}
    Assume $\mu \in \Pp^\Leb_2([0, 1] \times \Om)$ and $\fmap$ satisfies~\cref{fmap_assumption1}. Then for every $x \in \RR^d$ there exists a unique weak solution to~\cref{forward}, that is an absolutely continuous path $(x(s))_{s \in [0, 1]}$ such that for every $s \in [0, 1]$:
    \begin{align} \label{forward_weak}
        x(s) = x + \int_0^s F_{\mu(.|r)} (x(r)) \d r .
    \end{align}
\end{prop}

\begin{proof}
    The result follows Caratheodory's theorem for the existence and uniqueness of absolutely continuous solutions~\cite[Sec.I.5]{hale1969ordinary}. Given $\mu \in \Pp_2^\Leb([0, 1] \times \Om)$, the map $(s, x) \mapsto F_{\mu(.|s)}(x)$ is measurable w.r.t. $s$ and, thanks to~\cref{fmap_assumption1} (local Lipschitz continuity), locally Lipschitz w.r.t. $x$ with a Lipschitz constant that is integrable w.r.t. $s$. Moreover, the solutions of~\cref{forward_weak} are defined up to time $s = 1$ thanks to the growth assumption in~\cref{fmap_assumption1}, and if $C$ is the growth constant we get the following bound on the solution:
    \begin{align} \label{flow_growth}
        \forall s \in [0, 1], \quad \| x(s) \| \leq \exp( C (1+ \Ee_2(\mu))) ( \| x(0) \| + C (1+ \Ee_2(\mu))).
    \end{align}
\end{proof}

\subsection{Supervised learning problem}

We consider training our \NODE{} model for a supervised learning task. That is, given a data distribution $\RR^d \times \RR^{d'} \ni (x, y) \sim \Dd$ and loss $\ell: \RR^d \times \RR^{d'} \to \RR_+$, we associate to a parameterization $\mu \in \Pp_2^\Leb([0,1] \times \Om)$ the risk:
\begin{align} \label{risk}
    L(\mu) \eqdef \ \EE_{x, y} \ell(\NODE_\mu(x), y) = \EE_{x, y} \ell(x_\mu(1), y) \,.
\end{align}
In the following, we assume the data distribution $\Dd$ has compact support and $\ell$ is a smooth loss.
In particular, remark that in the case of an empirical data distribution $\Dd = \frac{1}{N} \sum_{i=1}^N \delta_{(x_i, y_i)}$, we recover the empirical risk:
\begin{align*}
    L(\mu) = \frac{1}{N} \sum_{i=1}^N \ell(\NODE_\mu(x_i), y_i) \,.
\end{align*}
Training the \NODE{} model then amounts to finding a parameterization $\mu$ which minimizes the risk defined in~\cref{risk} \emph{i.e.} to solve the following \emph{risk minimization} task:
\begin{equation*}
    \text{Find} \quad \mu^* \in \argmin_{\mu \in \Pp_2^\Leb([0,1] \times \Om)} L(\mu)\,.
\end{equation*}
In this work, we consider solving the above risk minimization task by performing gradient flow over the parameter $\mu$. We are interested in providing sufficient conditions to ensure that such a gradient flow can ``efficiently train'' our \NODE{} \emph{i.e.} finds a minimizer $\mu^*$ of the risk.

\subsection{Related works and contributions}

Due to the popularity and performance of the ResNet architecture, many works have studied its training dynamic and the convergence of this training dynamic toward a global minimizer of the risk.

\paragraph{ResNets and \NODEs{} of finite width}
 
\textcite{allen2019convergence,du2019gradient,liu2020linearity} give convergence results for the training of deep neural networks with gradient descent and their results can be applied to ResNets. However, their condition for convergence depends on the ResNet's depth and their result can therefore not be applied to infinitely deep models of \NODEs{}. \textcite{marion2023implicit} give local convergence results for the training of \NODEs{} based on a local Polyak-\L{}ojasiewicz (P-\L{}) condition. They assume a model of finite width and their result therefore does not hold in the mean field limit where residuals are of the form~\cref{measure_parameterization}.
They also consider parameter initializations that are Lipschitz w.r.t. depth which is not consistent with applications where the parameters are initialized at random, independently at each layer.
As a comparison~\cref{measure_parameterization} models residuals of both finite and infinite width and we only assume that the family $\lbrace \mu(.|s) \rbrace_{s \in [0,1]}$ is measurable w.r.t. $s \in [0,1]$.

\paragraph{\NODEs{} with linear parameterization of the residuals}

Relying on a similar local P-\L{} analysis, \textcite{barboni2022global} show a local convergence result for \NODEs{} where the residuals are linear functions on a Hilbert space of arbitrary dimension. For SHL residuals this amounts to only training the outer weights and keeping the inner weights unchanged. While preserving the expressivity of the model, one can see that this assumption does not model the way ResNets are trained in practice.
It also simplifies the analysis by naturally providing the set of residuals with a Hilbert space structure.
In contrast,
our model allows for residuals that are not necessarily linear w.r.t. their parameters and encompasses (among other) the case of SHL residuals where both the outer and inner weights are trained.
This is done at the cost of considering the parameter set $\Pp_2^\Leb([0,1] \times \Om)$ which has no Hilbert space structure and is a positively curved metric space when equipped with the Conditional OT distance~\parencite[Sec.7.3]{ambrosio2008gradient}.
Another notable difference is that in~\parencite{barboni2022global} the risk is shown to admit no spurious critical points (saddles, local minima, ...) whereas in this work, the risk may admit saddles whenever $\D_\om \fmap$ is not surjective. This is for example happening for the SHL architecture defined by~\cref{fmap_SHL} when the feature distribution is not sufficiently spread.

\paragraph{\NODEs{} in the mean-field limit}

Some works have proposed models for ResNets of infinite depth similar to~\cref{def:NODE}. \textcite{e_barron_2021} study properties of the functional space induced by considering the flow of functions of the form~\cref{measure_parameterization} and define a notion of norm which they use to provide bounds on the Rademacher complexity of this class of function. \textcite{chen2023generalization} also provide bounds on this Rademacher complexity which they use to prove an upper bound on the generalization error of trained ResNets.

Closer to our work are the works of~\textcite{lu2020mean,ding2021global,ding2022overparameterization} studying gradient flow dynamics for the minimization of the risk $L$ for the ResNet model of~\cref{def:NODE}. \textcite{lu2020mean} consider gradient flows w.r.t. the true Wasserstein distance on the space of measures. While this point of view motivates a new training strategy, it is not consistent with the way ResNets are trained in practice, that is with a layer-wise-$L^2$ metric. \textcite{ding2021global,ding2022overparameterization} show existence of gradient flows similar to~\cref{def:gradient_flow}.
A regularization of the risk is also assumed in~\parencite{ding2022overparameterization}.
Those three works give an optimality result at convergence: if the parameter converges then its limit is a global minimizer of the risk, however, they do not provide proofs of convergence.
This convergence assumption seems hard to justify \emph{a priori}, as the loss landscape of ResNets can have non-compact subsets and cases where the gradient flow fails to converge have been identified for simple architectures~\parencite{bartlett2018gradient}.

As a comparison, a key contribution of our work is to provide the parameter set with the appropriate metric structure allowing us to identify the gradient flow equation, derived formally by \emph{adjoint sensitivity analysis} with a \emph{curve of maximal slope} of the risk.
Similarly, \textcite{isobe2023convergence} considers \NODEs{} parameterized on the space of $\Pp_2(\Om)$-valued functions equipped with a "$L^2$-Wasserstein" metric and trained with gradient flow.
Using methods from the study of non-linear evolution equations, he shows the risk satisfies functional inequalities similar to the Polyak-\L{}ojasiewicz inequality in the neighborhood of critical points and shows convergence of the gradient flow to a critical point.
Aside from technical differences, our work differs in at least two fundamental aspects.
First, \cite{isobe2023convergence} considers adding a regularization term to the risk. Such a regularization ensures gradient flow curves stay in strongly compact sets~\cite[Prop.5.4]{isobe2023convergence} and admit convergent sub-sequences.
Also, the obtained functional inequality does not rule out the presence of non-optimal critical points and the obtained limit is thus not necessarily a minimizer of the risk.
In contrast, we consider an unregularized risk whose level sets may be non-compact and show convergence of the gradient flow to a global minimum for well-chosen initializations.

\paragraph{ResNets as a discretization of \NODEs{}}

While it is not addressed in the present work, an interesting question is the one of the consistency of the \NODE{} model with ResNets of finite depth.
\textcite{marion2023implicit} shows the convergence of ResNets of finite width to \NODEs{}, at initialization and during training, when the depth tends to infinity. This convergence is uniform over finite training time intervals but can be made uniform over the whole training dynamic under a convergence condition. For ResNets of arbitrary width, with layers of the form~\cref{measure_parameterization}, \textcite{ding2021global,ding2022overparameterization} give a result of uniform convergence over finite training time intervals. Adding a regularization term, \textcite{thorpe2023deep} show the $\Gamma$-convergence of the risk associated with ResNets to the one associated with \NODEs{}.

\paragraph{Conditional Optimal transport}

In this work, we rely on the properties of the Conditional OT metric (\cref{sec:metric}) to define a notion of gradient flow for the training of ResNets in the mean-field limit.
Similar metrics have been used in recent works for other applications, for example \textcite{peszek2023heterogeneous} use gradient flow in the Conditional OT topology to study evolution PDEs with heterogeneities, \textcite{hosseini2023conditional} apply Conditional OT to the study of solutions to \emph{Bayesian Inverse Problems}, \textcite{chemseddine2024conditional} consider applications to \emph{Bayesian Flow Matching} and \textcite{kerrigan2024dynamic} consider applications to conditional generative modeling.
Important for our work are the dynamical properties of the Conditional OT metric. Analogously to the Wasserstein case~\parencite{ambrosio2008gradient}, we show that absolutely continuous curves are solutions to certain continuity equations~(\cref{prop:characterization_absolute_continuity}). Similar results were shown in~\cite{peszek2023heterogeneous}.

\paragraph{Contributions}

Our main contribution is to propose a model for ResNets of infinite depth and arbitrary width, together with a metric space structure that is consistent with the layer-wise-$L^2$-metric used in practice when training ResNets with gradient descent and automatic differentiation. Our model thus allows a rigorous analysis of the training of ResNets at infinite depth and arbitrary width.

In detail, the ResNet model of~\cref{def:NODE} is parameterized over the set $\Pp_2^\Leb([0, 1] \times \Om)$ of probability measures whose first marginal is the Lebesgue measure on $[0,1]$, which we provide in~\cref{sec:metric} with the metric structure of a $L^2$-Wasserstein (or \emph{Conditional Optimal Transport}) distance $d$ (\cref{prop:distance}).
In~\cref{sec:gradient_flow} we leverage results from the theory of gradient flows in metric spaces \parencite{ambrosio2008gradient,santambrogio2017euclidean} to define the gradient flow of the risk $L$. This gradient flow equation corresponds to both notions of \emph{curve of maximal slope} of the risk and the usual gradient flow of ResNets obtained by \emph{adjoint sensitivity analysis} \parencite{chen2018neural}. We conclude this part by showing well-posedness results for the gradient flow equation, that is existence in arbitrary time (\cref{thm:existence_curve}), uniqueness (\cref{thm:uniqueness_curve}) and stability w.r.t. initialization (\cref{thm:stability_curve}).
Finally, we study in~\cref{sec:convergence} the asymptotic behavior of gradient flow curves.
We show that the risk $L$ satisfies a Polyak-\L{}ojasiewicz (P-\L{}) property around well-chosen initializations. The risk has no saddles in these regions and decreases at a constant rate along the gradient flow. Based on previous works on the convergence of curves of maximal slope under the P-\L{} assumption \parencite{schiavo2023local,hauer2019kurdyka}, we can then formulate a convergence result: for initializations with a sufficiently large but finite number of features and sufficiently low risk the gradient flow converges to a global minimizer (\cref{thm:convergence_SHL}). 
Our results are to be compared with the ones of~\textcite{lu2020mean,ding2022overparameterization}. Both works give a result of optimality under a convergence assumption but do not give conditions guaranteeing convergence of the gradient flow. Moreover, their results hold under the assumption of an infinite number of features whereas our convergence conditions can be obtained with a finite number of features. 

In addition to this, we studied in~\cref{sec:metric} theoretical properties of the space $\Pp_2^\Leb([0,1] \times \Om)$ equipped with the Conditional Optimal Transport distance $d$.
The literature on this subject still being sparse, some of our results might be of their own interest.
In particular, we provide in~\cref{prop:characterization_absolute_continuity} a characterization of absolutely continuous curves analogous to the one in the Wasserstein space \parencite[Thm.8.3.1]{ambrosio2008gradient}.

\section{Metric structure of the parameter set $\Pp^\Leb_2([0, 1] \times \Om)$} \label{sec:metric}

We define here a notion of distance $d$ over the parameter set $\Pp_2^\Leb([0, 1] \times \Om)$ and study its properties.
Importantly, the characterization of absolutely continuous curves in the metric space $(\Pp_2^\Leb([0, 1] \times \Om), d)$ will be used in~\cref{sec:gradient_flow} to define the notion of gradient flow for the risk $L$.

In the rest of this work, we will assume for simplicity that $\Om$ is the Euclidean space $\RR^p$ for some $p \geq 1$. However, the presented results could probably be adapted to the case where $\Om$ is a smooth manifold embedded in $\RR^p$. In particular, we will extensively use the fact that $\Om$ is a complete, separable metric space.
We recall the definition of the \emph{Wasserstein} $\Ww_2$ distance over the space $\Pp_2(\RR^p)$,  obtained as the solution to Kantorovich's optimal transport problem:
\begin{align} \label{wasserstein_distance}
    \forall \mu, \mu' \in \Pp_2(\RR^p), \quad \Ww_2(\mu, \mu') \eqdef \min_{\gamma \in \Gamma(\mu, \mu')} \left( \int_{\RR^p \times \RR^p} \| \om - \om' \|^2 \d \gamma(\om, \om') \right)^{1/2} ,
\end{align}
where $\Gamma(\mu, \mu')$ is the set of \emph{couplings} between $\mu$ and $\mu'$, that is the set of probability  measures $\gamma \in \Pp(\RR^p \times \RR^p)$ respecting the marginal conditions $\pi^1_\# \gamma = \mu$ and $\pi^2_\# \gamma = \mu'$. We denote by $\Gamma_o(\mu, \mu') \subset \Gamma(\mu, \mu')$ the subset of \emph{optimal couplings} achieving the equality in~\cref{wasserstein_distance}. We refer to the books of~\textcite{villani2009optimal,santambrogio2015optimal} for further properties of the Wasserstein distance.

\subsection{Conditional Optimal Transport distance}

The Conditional Optimal Transport distance $d$ is a modification of the Wasserstein distance $\Ww_2$ with the supplementary constraint that the transport plan should preserve the marginal over $[0, 1]$. This constraint is introduced to closely model the training dynamic of ResNets where the gradients are computed over the weights of each layer independently.
For this purpose, it is natural to define a ``layer-wise-$L^2$'' Wasserstein distance, that is an $L^2$-distance over the set of families of probability measures in $\Pp_2(\Om)$ indexed over $s \in [0,1]$.

\begin{prop}[Conditional Optimal Transport distance] \label{prop:distance}
    \hfill \\
    Consider the function $d : \left( \Pp^\Leb_2([0, 1] \times \Om) \right)^2 \to \RR_+$ defined by:
    \begin{align*}
        d(\mu, \mu') \eqdef \left( \int_0^1 \Ww_2 \left( \mu(.|s), \mu'(.|s) \right)^2 \d s \right)^{1/2} .
    \end{align*}
    Then, $d$ defines a metric on $\Pp^\Leb_2([0, 1] \times \Om)$.
\end{prop}

\begin{proof}
    One essentially needs to justify the existence of the integral in the definition of $d$.
    That $d$ is a metric then follows from the properties of the Wasserstein and $L^2$ metrics respectively.

    For a complete separable metric space $\Om$ and Borel probability measures $\mu, \nu \in \Pp(\Om)$ it is known \parencite[Thm.5.10]{villani2009optimal} that the Monge-Kantorovich problem admits the dual formulation:
    \begin{align*}
        \Ww_2(\mu, \nu)^2 = \sup \lbrace \int_\Om \varphi \d \mu + \int_\Om \psi \d \nu \rbrace
    \end{align*}
    where the supremum is taken over all pairs $(\varphi, \psi) \in \Cc_b(\Om) \times \Cc_b(\Om)$ s.t.\@ $\varphi(x) + \psi(y) \leq \| x-y \|^2$.
    An alternative formulation is:
    \begin{align*}
        \Ww_2(\mu, \nu)^2 = \sup_{\varphi \in \Cc_b(\Om)} \lbrace \int_\Om \varphi \d \mu + \int_\Om  \varphi^c \d \nu \rbrace ,
    \end{align*}
    where for $\varphi : \Om \to \RR$ the $c$-transform $\varphi^c$ of $\varphi$ is defined as~\parencite[Def.1.10]{santambrogio2015optimal}:
    \begin{align*}
        \forall \om \in \Om, \quad \varphi^c(\om) \eqdef \inf_{\om' \in \Om} \| \om' - \om \|^2 - \varphi(\om').
    \end{align*}
    Consider $(\varphi_n)_{n \geq 0}$ a sequence of functions in $\Cc_b(\Om)$ such that, for any $\varphi \in \Cc_b(\Om)$, one can find a subsequence $m(n)$ with $\varphi_{m(n)} \to \varphi$ for the compact-open topology (uniform convergence on compact subsets) and $\| \varphi_{m(n)} \|_\infty$ is uniformly bounded.
    Then we also have $\varphi_{m(n)}^c \to \varphi^c$ uniformly on compact subsets with $\| \varphi^c_{m(n)} \|_\infty \leq \| \varphi_{m(n)} \|_\infty$ uniformly bounded, whence:
    \begin{align*}
        \Ww_2(\mu, \nu)^2 = \sup_{n \in \NN} \lbrace \int_\Om \varphi_n \d \mu + \int_\Om \varphi_n^c \d \nu  \rbrace .
    \end{align*}
    Thus, for $\mu, \mu' \in \Pp_2^\Leb([0, 1] \times \Om)$, the application $s \mapsto \Ww_2 \left( \mu(.|s), \mu'(.|s) \right)^2$ is measurable as it can be expressed as the supremum of countably many measurable functions.
\end{proof}

Alternatively, the distance $d$ can be viewed as an optimal transport distance with the additional constraint that the transport plans should be the identity on the first marginal.
This new formulation proves itself convenient for calculations and, in particular, allows easily estimating the distance $d$ from above.
Given $\mu, \mu' \in \Pp_2^\Leb([0, 1] \times \Om)$, we define the two following sets of ``couplings'' between $\mu$ and $\mu'$:
\begin{align*}
    \Pi^\Leb(\mu, \mu') & \eqdef \lbrace \gamma \in \Pp_2^\Leb ([0, 1] \times \Om^2), \ \gamma(.|s) \in \Gamma(\mu(.|s), \mu'(.|s)) \ \text{for $\d s$-a.e. $s \in [0, 1]$} \rbrace, \\
    \Gamma^\diag(\mu, \mu') & \eqdef \lbrace \gamma \in \Gamma(\mu, \mu'), \ \int f(s,s') \d \gamma(s,\om,s', \om') = \int_0^1 f(s,s) \d s  \ \forall f \in \Cc([0,1]^2) \rbrace . 
\end{align*}
Note that these two sets are closely related as, if $\gamma \in \Pi^\Leb(\mu, \mu')$, then $\Tilde{\gamma} \eqdef (\pi^1,\pi^2,\pi^1, \pi^3)_\# \gamma \in \Gamma^\diag(\mu, \mu')$ and conversely, if $\Tilde{\gamma} \in \Gamma^\diag(\mu, \mu')$, then $\gamma \eqdef (\pi^1, \pi^2, \pi^4)_\# \Tilde{\gamma} \in \Pi^\Leb(\mu, \mu')$. In both cases, we have for any measurable $f : \Om^2 \to \RR$: 
\begin{align} \label{coupling_correspondence}
    \int_{[0,1] \times \Om^2} f(\om,\om') \d \gamma(s, \om, \om') = \int_{([0, 1] \times \Om)^2} f(\om, \om') \d \Tilde{\gamma}(s,\om,s',\om') .
\end{align}
In the same way the Wasserstein distance $\Ww_2(\mu, \mu')$ can be obtained as the solution of a minimization problem over the set $\Gamma(\mu,\mu')$ (\cref{wasserstein_distance}),
the Conditional OT distance $d$ can be obtained as the solution of minimization problems over the sets $\Pi^\Leb(\mu, \mu')$ and $\Gamma^\diag(\mu, \mu')$.

\begin{prop} \label{prop:d_characterization}
    Let $\mu, \mu' \in \Pp_2^\Leb([0, 1] \times \Om)$ then:
    \begin{align*}
        d(\mu, \mu')^2 & = \min_{\gamma \in \Pi^\Leb (\mu, \mu')} \int_{[0,1] \times \Om^2} \| \om - \om' \|^2 \d \gamma(s, \om, \om') \\
        & = \min_{\gamma \in \Gamma^\diag (\mu, \mu')} \int_{([0, 1] \times \Om)^2} \| \om - \om' \|^2 \d \gamma(s,\om,s',\om')  .
    \end{align*}
    We denote respectively by $\Pi^\Leb_o(\mu, \mu')$ and $\Gamma^\diag_o(\mu, \mu')$ the set of optimal couplings in both minimization problems. Then for $\gamma \in \Pi^\Leb_o(\mu,\mu')$ we have for $\d s$-a.e. $s \in [0, 1]$:
    \begin{align*}
        \gamma(.|s) \in \Gamma_o(\mu(.|s), \mu'(.|s)) \quad \text{i.e.} \quad \int_{\Om^2} \| \om-\om' \|^2 \d \gamma(\om, \om'|s) = \Ww_2(\mu(.|s), \mu'(.|s))^2.
    \end{align*}
\end{prop}

\begin{proof}
    Our proof technique is similar to the one of~\cite[Prop.3.3]{hosseini2023conditional} and relies on the possibility of choosing an optimal transport plan $\gamma(.|s) \in \Gamma_o(\mu(.|s), \mu'(.|s))$ for every $s \in [0, 1]$ in a measurable way. 

    We show equality with the first minimization problem on $\Pi^\Leb(\mu,\mu')$, equality between the two minimization problems then comes from~\cref{coupling_correspondence}. Assume there exists a Borel map $\gamma : s \mapsto \gamma(.|s) \in \Pp(\Om^2)$ (where $\Pp(\Om^2)$ is equipped with the narrow topology) such that $\gamma(.|s) \in \Gamma_o(\mu(.|s), \mu'(.|s))$ for every $s \in [0, 1]$. With such a map one can define a Borel probability measure that we also denote by $\gamma$ over $[0, 1] \times \Om$ which is the measure whose disintegration w.r.t. the Lebesgue measure on $[0,1]$ is $\lbrace \gamma(.|s) \rbrace_{s \in [0, 1]}$. In other words the measure $\gamma$ is defined as:
    \begin{align*}
        \int_{[0, 1] \times \Om^2} f(s, \om, \om') \d \gamma(s, \om, \om') \eqdef \int_0^1 \int_{\Om^2} f(s, \om, \om') \d \gamma(\om, \om'|s) \d s , \quad \forall f \in \Cc_b([0, 1] \times \Om^2).
    \end{align*}
    Such $\gamma$ will be a solution to our first optimization problem as we have:
    \begin{align*}
        \int_{[0, 1] \times \Om^2} \| \om-\om' \|^2 \d \gamma(s, \om, \om') = d(\mu, \mu')^2 \leq \inf_{\gamma \in \Pi^\Leb (\mu, \mu')} \int_{[0,1] \times \Om^2} \| \om - \om' \|^2 \d \gamma(s, \om, \om') .
    \end{align*}

    To show the existence of such $\gamma$ we use a measurable selection result, that is considering the set-valued mapping $s \in [0, 1] \mapsto \Gamma_o(\mu(.|s), \mu'(.|s)) \subset \Pp(\Om^2)$ we show it admits a measurable section. Consider the set:
    \begin{align*}
        \Gg^* \eqdef \lbrace (s,\gamma), \, \gamma \in \Gamma_o(\mu(.|s), \mu'(.|s)) \rbrace \subset [0, 1] \times \Pp_2(\Om^2) 
.    \end{align*}
    Using~\cite[Thm.6.9.6]{bogachev2007measure}, as for every $s \in [0, 1]$ the set $\Gamma_o(\mu(.|s), \mu'(.|s))$ is narrowly compact,  it is sufficient to show that $\Gg^* \in \Bb([0, 1] \times \Pp_2(\Om^2))$. Let $\lbrace f_n \rbrace_{n \in \NN}$ be dense in $\Cc_b(\Om)$ for the compact-open topology. Then, for every $n \in \NN$ as the mapping $(s,\gamma) \mapsto \int_{\Om^2} f_n \d \gamma - \int_\Om f_n \d \mu(.|s)$ is measurable, so are the sets:
    \begin{align*}
        \Gg_n & \eqdef \lbrace (s, \gamma), \, \int_{\Om^2} f_n(\om) \d \gamma(\om,\om')  = \int_\Om f_n(\om) \d \mu(\om|s) \rbrace \\
        \Gg_n' & \eqdef \lbrace (s,\gamma), \, \int_{\Om^2} f_n(\om') \d \gamma(\om,\om') = \int_\Om f_n(\om') \d \mu'(\om'|s)  \rbrace .
    \end{align*}
    Also, as the mapping $(s, \gamma) \mapsto \Ww_2(\mu(.|s), \mu'(.|s))^2 - \int_{\Om^2} \| \om-\om' \|^2 \d \gamma$ is measurable by~\cref{prop:distance}, so is the set:
    \begin{align*}
        \Gg_o \eqdef \lbrace (s, \gamma), \, \int_{\Om^2} \|\om-\om'\|^2 \d \gamma = \Ww_2(\mu(.|s), \mu'(.|s))^2 \rbrace \in \Bb([0,1] \times \Pp_2(\Om^2)) .
    \end{align*}
    Finally we have that $\Gg^* = \Gg_o \cap \left( \bigcap_{n \in \NN} \Gg_n \cap \Gg_n' \right)$ is a Borel set, which completes the proof.
\end{proof}

\begin{rem}[Comparison with $\Ww_2$] \label{rem:topology}
Note that, for $\mu, \mu' \in \Pp^\Leb_2([0,1] \times \Om)$, we have that $\Gamma^\diag(\mu,\mu') \subset \Gamma(\mu, \mu')$. Hence from the previous result, it follows:
\begin{align*}
    \Ww_2(\mu, \mu') \leq d(\mu, \mu')
\end{align*}
and the topology induced by $d$ on $\Pp^\Leb_2([0, 1] \times \Om)$ is stronger than the Wasserstein topology. It is in fact strictly stronger and, 
for example, the sequence $\mu_n = \int_0^1 \delta_{(-1)^{\lfloor 2 n s \rfloor}} \d s$ and the measure $\mu = \frac{1}{2} \int_0^1 (\delta_1 + \delta_{-1})\d s$ in $\Pp_2^\Leb([0,1] \times \RR)$ are such that $\Ww_2(\mu_n, \mu) \to 0$ but $d(\mu_n, \mu) \geq 1$.
\end{rem}

The following result states that the metric space $(\Pp_2^\Leb([0,1] \times \Om), d)$ is complete.

\begin{prop}[Completeness] \label{prop:d_complete}
    $(\Pp_2^\Leb([0,1] \times \Om), d)$ is a complete metric space.
\end{prop}

\begin{proof}
    The proof is analogous to the proof of completeness of the Wasserstein space $\Pp_2([0,1] \times \Om)$ (see~\cite[Thm.6.18]{villani2009optimal})

    Let $(\mu_n)_{n \geq 0}$ be a Cauchy sequence in $\Pp_2^\Leb([0,1] \times \Om)$. Then for any $\nu \in \Pp_2^\Leb([0,1] \times \Om)$
    \begin{align*}
        d(\nu, \mu_n) \leq d(\nu,\mu_0) + \sum_{1 \leq i \leq n} d(\mu_i, \mu_{i-1}) \leq d(\nu,\mu_0) + \sum_{i \geq 1} d(\mu_i, \mu_{i-1}) < \infty,
    \end{align*}
    implying that the sequence $(\mu_n)$ is bounded. Hence it is tight and by Prokhorov's theorem, it admits a subsequence $(\mu_{n_k})_{k \geq 0}$ converging narrowly to some $\mu_\infty \in \Pp_2^\Leb([0,1] \times \Om)$. Then by narrow lower-semicontinuity of $d$ (\cref{lem:d_lsc}) we have for every $l \geq 0$:
    \begin{align*}
        d(\mu_\infty, \mu_{n_l}) \leq \liminf_{k \to \infty} d(\mu_{n_k}, \mu_{n_l})
    \end{align*}
    and by taking the $\limsup$ w.r.t. $l$:
    \begin{align*}
        \limsup_{l \to \infty} d(\mu_\infty, \mu_{n_l}) \leq \limsup_{\substack{k \to \infty \\ l \to \infty}} d(\mu_{n_k}, \mu_{n_l}) = 0 .
    \end{align*}
    Hence $(\mu_{n_k})$ $d$-converges to $\mu_\infty$, implying that the whole sequence $(\mu_n)$ $d$-converges to $\mu_\infty$.
\end{proof}

\begin{lem}[narrow lower-semicontinuity of $d$] \label{lem:d_lsc}
    Let $(\mu_n)_{n \geq 0}$ and $(\nu_n)_{n \geq 0}$ be sequences in $\Pp_2^\Leb([0,1] \times \Om)$ such that $(\mu_n, \nu_n) \xrightarrow[n \to \infty]{} (\mu,\nu)$ narrowly for some $\mu, \nu \in \Pp_2^\Leb([0,1] \times \Om)$. Then:
    \begin{align*}
        d(\mu,\nu) \leq \liminf_{n \to \infty} d(\mu_n, \nu_n).
    \end{align*}
\end{lem}

\begin{proof}
    Up to extraction of a subsequence one can consider $d(\mu_n, \nu_n) \to \liminf d(\mu_n, \nu_n)$. Then for every $n \geq 0$ consider some $\gamma_n \in \Gamma_o^\diag(\mu_n,\nu_n)$. In particular $\gamma_n \in \Gamma(\mu_n,\nu_n)$ and by~\cite[Lem.4.4]{villani2009optimal} the sequence $(\gamma_n)$ is tight. Hence it admits a subsequence $(\gamma_{n_k})_{k \geq 0}$ narrowly converging to some $\gamma$ which is in $\Gamma^\diag(\mu,\nu)$ by the properties of narrow convergence. Thus applying~\cite[Lem.4.3]{villani2009optimal} and using the characterization of $d$ in~\cref{prop:d_characterization}:
    \begin{align*}
        d(\mu, \nu)^2 \leq \int_{([0,1] \times \Om)^2 } \| \om-\om' \|^2 \d \gamma \leq \liminf_{k \to \infty} \int_{([0,1] \times \Om)^2 } \| \om-\om' \|^2 \d \gamma_{n_k} = \liminf_{n \to \infty} d(\mu_n, \nu_n)^2,
    \end{align*}
    from which the result follows.
\end{proof}

\subsection{Dynamical formulation of Conditional Optimal Transport}

We here analyze the properties of absolutely continuous curves in the metric space $\Pp_2^\Leb([0,1] \times \Om)$ equipped with the Conditional OT metric. Similarly to the (unconstrained) Wasserstein metric, we show that absolutely continuous curves obey a certain continuity equation.
This characterization will be crucial for defining the gradient flow equation used in the training of our \NODE{} model.

\paragraph{Absolutely continuous curves for the Wasserstein distance $\Ww_2$}

For $T > 0$, consider $I = (0, T)$ an open interval and $(\mu_t)_{t \in I}$ a family of probability measures on the Euclidean space $\RR^p$.
Given a Borel velocity field $v: (t,x) \in I \times \RR^p \mapsto v_t(x) \in \RR^p$ such that $\int_I \int_{\RR^p} \| v_t \| \d \mu_t \d t < \infty$, we say that $(\mu_t)_{t \in I}$ satisfies the continuity equation $\partial_t \mu_t + \div(v_t \mu_t)$ in the weak sense if:
\begin{align} \label{continuity_weak}
    \int_I \int_{\RR^p} \left( \partial_t \varphi(t,x) + \langle \nabla \varphi(t,x), v_t(x) \rangle \right) \d \mu_t(x) \d t = 0, \quad \forall \varphi \in \Cc^1_c(I \times \RR^p).
\end{align}
Equivalently (\cite[Prop.4.2]{santambrogio2015optimal}), when the mapping $t \mapsto \mu_t$ is narrowly continuous, this is equivalent to saying that for every $\varphi \in \Cc^1_c(\RR^p)$ the map $t \mapsto \mu_t(\varphi) \eqdef \int \varphi \d \mu_t$ is absolutely continuous and satisfies:
\begin{align*}
    \frac{\d}{\d t} \mu_t(\varphi) = \int \langle \nabla \varphi, v_t \rangle \d \mu_t, \quad \text{for $\d t$-a.e. $t \in I$.}
\end{align*}
An important property of the Wasserstein space $\Pp_2(\RR^p)$ endowed with the distance $\Ww_2$ is the characterization of absolutely continuous curves: a narrowly continuous curve $(\mu_t)_{t \in I}$ is absolutely continuous in $\Pp_2(\RR^p)$ if and only if it is solution to the continuity equation~\cref{continuity_weak} for some velocity field $v$ with $\int_I \| v_t \|_{L^2(\mu_t)} \d t < \infty$~\parencite[Thm.8.3.1]{ambrosio2008gradient}. We refer to the book by~\textcite{ambrosio2008gradient} for a detailed study of absolutely continuous curves in $(\Pp_2(\RR^p), \Ww_2)$.

\paragraph{Absolutely continuous curves for the Conditional OT distance}

Similarly to the characterization of absolutely continuous curves in the Wasserstein space $\Pp_2(\RR^p)$, an adaptation of~\cite[Thm.8.3.1]{ambrosio2008gradient} provides an analogous characterization of absolutely continuous curves in $\Pp_2^\Leb([0,1] \times \Om)$ equipped with the Conditional OT distance $d$. This characterization allows us to (formally) provide the metric space $(\Pp_2^\Leb([0,1] \times \Om), d)$ with a kind of ``differential structure'' by seeing tangent vectors as velocity fields.
This identification will be crucial for defining the gradient flow equation, which will take the form of a continuity equation with an appropriate velocity field (\cref{def:gradient_flow}).

\begin{prop}[adapted from~\cite{ambrosio2008gradient}, Thm.8.3.1] \label{prop:characterization_absolute_continuity}
    Assume $\Om = \RR^p$.
    Let $I = (0,T)$ for some $T > 0$ and let $(\mu_t)_{t \in I}$ be an absolutely continuous curve in $\Pp^\Leb_2([0, 1] \times \Om)$.
    Then there exists a unique Borel velocity field $v : (t,s,\om) \in I \times [0, 1] \times \Om \mapsto v_t(s,\om) \in \Om$ such that for a.e. $t \in I$:
    \begin{align*}
        v_t \in L^2(\mu_t), \quad \| v_t \|_{L^2(\mu_t)} \leq \left|\frac{\d}{\d t} \mu_t \right| ,
    \end{align*}
    and $\mu$ is a weak solution of the continuity equation:
    \begin{align} \label{continuity}
        \partial_t \mu_t + \div((0, v_t) \mu_t ) = 0 \quad \text{over $I \times [0,1] \times \Om$.}
    \end{align}
    We will refer to such $v_t$ as the \emph{tangent velocity field} of the curve $(\mu_t)_{t \in I}$.
    Conversely, if $(\mu_t)_{t \in I}$ is a narrowly continuous curve satisfying~\cref{continuity} for a Borel velocity field $v_t$ with $\| v_t \|_{L^2(\mu_t)} \in L^1(I)$, then $(\mu_t)_{t \in I}$ is absolutely continuous and $| \frac{\d}{\d t} \mu_t | \leq \| v_t \|_{L^2(\mu_t)}$ for a.e. $t \in I$.
\end{prop}

\begin{proof}
    \proofpart{AC curve $\Rightarrow$ Continuity equation.}
    
    Note that this part is the easiest as $d$-absolute continuity implies $\Ww_2$-absolute continuity for which the result is well-known, originally proven in~\cite[Thm.8.3.1]{ambrosio2008gradient}. 
    Therefore we here only adapt this proof to our specific setting.

    Up to reparameterization, one can assume without loss of generality that $\left| \frac{\d}{\d t} \mu_t \right| \in L^\infty(I)$.

    First we show that for $\varphi \in \Cc^1_c(I \times [0,1] \times \Om)$ the map $t \mapsto \mu_t(\varphi) \eqdef \int_0^1 \int_\Om \varphi \d \mu_t$ is absolutely continuous. Indeed, for $t, u \in I$, introducing a coupling $\gamma_{t,u} \in \Pi^\Leb_o(\mu_t, \mu_u)$ we have:
    \begin{align*}
        \left| \mu_t(\varphi) - \mu_u(\varphi) \right| \leq  \left| \int_0^1 \int_{\Om^2} (\varphi(s,\om) - \varphi(s,\om')) \d \gamma_{t,u}(s, \om, \om') \right| \leq \| \nabla_\om \varphi \|_\infty d(\mu_t, \mu_u) ,
    \end{align*}
    from which absolute continuity follows. Then considering the map:
    \begin{align*}
        H(s, \om, \om') \eqdef
        \left\{
            \begin{array}{cc}
                \| \nabla_\om \varphi(s,\om) \| & \text{if $\om = \om'$},  \\
                & \\
                \frac{| \varphi(s,\om) - \varphi(s,\om') |}{\| \om-\om' \|} & \text{else}, 
            \end{array}
        \right.
    \end{align*}
    we have for every $t, u \in I$:
    \begin{align*}
        \frac{| \mu_t(\varphi) - \mu_u(\varphi) |}{| t-u |}
        \leq \frac{1}{| t-u |} \int_0^1 \int_{\Om^2} \| \om-\om' \| H(s, \om, \om') \d \gamma_{t,u}(s,\om,\om')
        \leq \frac{d(\mu_t, \mu_u)}{|t-u|} \| H \|_{L^2(\gamma_{t,u})}
    \end{align*}
    As $u \to t$ we have $d(\mu_u, \mu_t) \to 0$ and by the properties of $L^2$ spaces~\parencite[\emph{e.g.}][Prop.3.11]{cannarsa2015introduction} we can take a sequence $u_n \to t$ such that $\Ww_2(\mu_{u_n}(.|s),\mu_t(.|s)) \to 0$ for $\d s$-a.e. $s \in [0, 1]$.
    This implies for those $s \in [0, 1]$ that $\mu_{u_n}(.|s) \to \mu_t(.|s)$ narrowly and that $\gamma_{t,u_n}(.|s) \to \gamma(.|s) \in \Gamma_o(\mu_t(.|s), \mu_t(.|s))$, i.e. the trivial transport plan $\gamma(.|s) = (\Id,\Id)_\# \mu_t(.|s)$.
    Thus we have that $\gamma_{t,u_n} \to (\pi^1, \pi^2, \pi^2)_\# \mu_t$ narrowly since, by Lebesgue's theorem, given a bounded continuous function $f \in \Cc_b([0,1] \times \Om \times \Om)$:
    \begin{align*}
        \int f \d \gamma_{t,u_n} = \int_0^1 \left( \int_{\Om^2} f(s,\om,\om') \d \gamma_{t,u_n}(\om,\om'|s) \right) \d s \xrightarrow[n \to \infty]{} \int_0^1 \left( \int_\Om f(s,\om,\om) \d \mu_{t}(\om|s) \right) \d s.
    \end{align*}
    Hence, at a point where $t \mapsto \mu_t$ is metrically differentiable:
    \begin{align*}
        \limsup_{u \to t} \frac{| \mu_t(\varphi) - \mu_u(\varphi) |}{| t-u |} \leq \left| \frac{\d}{\d t} \mu_t \right| \| \nabla_\om \varphi \|_{L^2(\mu_t)} .
    \end{align*}
    Consider $\mmu = \int_I \mu_t \d t \in \Pp(I \times [0,1] \times \Om)$ the measure whose disintegration w.r.t. Lebesgue's measure on $I$ is $(\mu_t)_{t \in I}$. Then for $\varphi \in \Cc^1_c(I \times [0,1] \times \Om)$ we have:
    \begin{align*}
         \int_I \int_{[0,1] \times \Om} \partial_t \varphi(t,s,\om) \d \mu_t(s,\om) \d t & = \lim_{h \to 0} \int_I \int_{[0,1] \times \Om} \frac{\varphi(t,s,\om) - \varphi(t-h,s,\om)}{h} \d \mu_t(s,\om) \d t \\
         & = \lim_{h \to 0} \int_I \frac{1}{h} \left( \int_{[0,1] \times \Om} \varphi(t,s,\om) \d \mu_t(s,\om) - \int_{[0,1] \times \Om} \varphi(t,s,\om) \d \mu_{t-h}(s,\om)  \right) \d t .
    \end{align*}
    Thus by the previous inequality and Fatou's lemma:
    \begin{align*}
        \left|  \int_I \int_{[0,1] \times \Om} \partial_t \varphi(t,s,\om) \d \mu_t(s,\om) \d t \right| \leq \left( \int_I \left| \frac{\d}{\d t} \mu_t \right|^2 \d t \right)^{1/2} \left( \int_{I \times [0,1] \times \Om} \| \nabla_\om \varphi(t,s,\om) \|^2 \d \mmu(t,s,\om) \right)^{1/2}.
    \end{align*}
    Consider the subspace $V \eqdef \lbrace \nabla_\om \varphi, \, \varphi \in \Cc^1_c(I \times [0,1] \times \Om) \rbrace$ and $\Vv$ its closure in $L^2(I \times [0,1] \times \Om, \mmu)$. Then by the previous inequality the linear functional $\Aa: V \to \RR$ defined by:
    \begin{align*}
        \Aa(\nabla_\om \varphi) \eqdef - \int_{I \times [0,1] \times \Om} \partial_t \varphi(t,s,\om) \d \mmu(t,s,\om)
    \end{align*}
    is continuous on $V$ and thus, by Hahn-Banach's theorem, can be extended to a unique continuous linear functional on $\Vv$. Therefore, by Lax-Milgram's theorem, the minimization problem:
    \begin{align*}
        \min \left\{ \frac{1}{2} \int_{I \times \RR^{p+1}} \| w(t,s,\om) \|^2 \d \mmu(t,s,\om) - \Aa(w), \, w \in \Vv \right\}
    \end{align*}
    admits a unique solution $v \in \Vv$ which is characterized by the property that:
    \begin{align*}
        \int_{I \times \RR^{p+1}} \langle v(t,s,\om), \nabla_\om \varphi(t,s,\om) \rangle \d \mmu(t,s,\om) = \Aa ( \nabla_\om \varphi ), \quad \forall \varphi \in \Cc^1_c(I \times [0,1] \times \Om) ,
    \end{align*}
    which is the desired continuity equation by definition of $\Aa$.

    Finally, let $(\nabla_\om \varphi_n) \subset V$ be a sequence converging to $v \in L^2(\mmu)$. Considering an interval $J \subset I$ and some $\eta \in \Cc^1_c(J)$ with $0 \leq \eta \leq 1$ we have by the previous arguments:
    \begin{align*}
        \int_{I \times [0,1] \times \Om} \eta(t) \| v(t,s,\om) \|^2 \d \mmu (t,s,\om) & = \lim_{n \to \infty} \int_{I \times [0,1] \times \Om} \eta \langle v, \nabla_\om \varphi_n \rangle \d \mmu \\
        & = \lim_{n \to \infty}  \Aa ( \nabla_\om (\eta \varphi_n) ) \\
        & \leq \left( \int_J \left| \frac{\d}{\d t} \mu_t \right|^2 \d t \right)^{1/2} \lim_{n \to \infty} \left( \int_{J \times [0,1] \times \Om} \| \nabla_\om \varphi_n \|^2 \d \mmu \right)^{1/2} \\
        & = \left( \int_J \left| \frac{\d}{\d t} \mu_t \right|^2 \d t \right)^{1/2}\left( \int_{J \times [0,1] \times \Om} \| v \|^2 \d \mmu \right)^{1/2} .
    \end{align*}
    Hence approximating the characteristic function of $J$ with such an $\eta$ we get:
    \begin{align*}
        \int_J \int_{[0,1] \times \Om} \| v_t \|^2 \d \mu_t \d t \leq  \int_J \left| \frac{\d}{\d t} \mu_t \right|^2 \d t ,
    \end{align*}
    implying $\| v_t \|_{L^2(\mu_t)} \leq \left| \frac{\d}{\d t} \mu_t \right|$ for a.e. $t \in I$.

    \proofpart{Continuity equation $\Rightarrow$ AC curve.}

    This part of the proof is new as, according to~\cite[Thm.8.3.1]{ambrosio2008gradient}, the continuity equation only ensures $\Ww_2$-absolute continuity as explained in~\cref{rem:topology}.
    We show here that the specific form of the velocity field ensures $d$-absolute continuity.
    
    For $(t,s) \in I \times [0,1]$, we denote by $v_{t,s}$ the Borel vector field $v_{t,s} : \om \in \Om \mapsto v_t(s,\om)$. Note that by Jensen's inequality:
    \begin{align*}
        \int_I  \int_0^1 \| v_{t,s} \|_{L^2(\mu_{t}(.|s))} \d s \d t \leq \int_I \| v_t \|_{L^2(\mu_t)} \d t < +\infty,
    \end{align*}
    and we have that for $\d s$-a.e. $s \in [0, 1]$, $t \mapsto \| v_{t,s} \|_{L^2(\mu_{t}(.|s))} \in L^1(I)$. Also if $\varphi \in \Cc_c^1(I \times \Om)$ and $\chi \in \Cc_c^1([0,1])$ then by definition of the continuity equation:
    \begin{align*}
        \int_I \int_0^1 \int_\Om \left( \partial_t \varphi + \langle \nabla_\om \varphi, v_{t,s} \rangle \right) \chi(s) \d \mu_t(.|s) \d s \d t = 0.
    \end{align*}
    Hence if $J \subset [0, 1]$ is an interval, approaching the characteristic function of $J$ with $\chi$ we get:
    \begin{align*}
        \int_I \int_J \int_\Om \left( \partial_t \varphi + \langle \nabla_\om \varphi, v_{t,s} \rangle \right) \d \mu_t(.|s) \d s \d t = 0,
    \end{align*}
    and hence for $\d s$-a.e. $s \in [0, 1]$:
    \begin{align*}
        \int_I \int_\Om \left( \partial_t \varphi + \langle \nabla_\om \varphi, v_{t,s} \rangle \right) \d \mu_t(.|s) \d t = 0.
    \end{align*}
    Now if we consider $(\varphi_n)$ a countable dense sequence in $\Cc^1_c(I \times \Om)$ endowed with the usual topology then we can find a set $\Lambda \subset [0, 1]$ of full Lebesgue's measure such that for every $s \in \Lambda$ the above equation holds for every test function $\varphi \in \Cc_c^1(I \times \Om)$. In other words we have shown that, for $\d s$-a.e. $s \in [0,1]$, $\mu_t(.|s)$ solves the continuity equation:
    \begin{align*}
        \partial_t \mu_t(.|s) + \div(v_{t,s} \mu_t(.|s)) = 0 \quad \text{over $I \times \Om$.}
    \end{align*}
    
    Note that, without loss of generality, we can consider the curve $(\mu_t(.|s))_{t \in I}$ to be narrowly continuous. Indeed, as it is a solution of the continuity equation we know that the curve $(\mu_t(.|s))_{t \in I}$ admits a narrowly continuous representative $\Tilde{\mu}_{t}(.|s)$ \parencite[Lem.8.1.2]{ambrosio2008gradient} and that this representative is characterized by that for every $\varphi \in \Cc^1_c(\Om)$ and every $t \in I$:
    \begin{align*}
        \Tilde{\mu}_{t}(.|s)(\varphi) = \int_0^t \int_\Om \left( \chi'(u) \varphi + \chi(u) \langle \nabla_\om \varphi, v_{u,s} \rangle \right) \d \mu_{u}(.|s) \d u ,
    \end{align*}
    where $\chi \in \Cc^1(I)$ is any function such that $\chi = 0$ on a neighbourhood of $0$ and $\chi (u) = 1$ for $u \geq t$ (the definition does not depend on $\chi$ by definition of the continuity equation). Then it follows that for any $t \in I$ and any $f \in \Cc_b([0,1] \times \Om)$ the map $s \mapsto \Tilde{\mu}_{t}(.|s)(f)$ is measurable and integrating w.r.t. $s$ we get that $\Tilde{\mu}_t \eqdef \int_0^1 \Tilde{\mu}_{t}(.|s) \d s$ defines a probability measure on $[0,1] \times \Om$ whose disintegration is $\lbrace \Tilde{\mu}_{t}(.|s) \rbrace_{s\in [0,1]}$. Moreover for any $\varphi \in \Cc_c^1([0,1] \times \Om)$ we have:
    \begin{align*}
        \Tilde{\mu}_t(\varphi) = \int_0^t \int_0^1 \int_\Om \left( \chi'(u) \varphi + \chi(u) \langle \nabla_\om \varphi, v_{u,s} \rangle \right) \d \mu_{u}(.|s) \d s \d u = \mu_t(\varphi)
    \end{align*}
    and hence in fact the equality $\Tilde{\mu}_t = \mu_t$.
    
    Then using~\cite[Thm.8.3.1]{ambrosio2008gradient} with the assumption that $t \mapsto \mu_t(.|s)$ is narrowly continuous we hence have that for $\d s$-a.e. $s \in [0,1]$ the curve $t \in I \mapsto \mu_t(.|s)$ is absolutely continuous and
    \begin{align*}
        \Ww_2(\mu_{t_1}(.|s),\mu_{t_2}(.|s))^2 \leq (t_2 - t_1) \int_{t_1}^{t_2} \int_\Om \| v_{t,s} \|^2 \d \mu_t(.|s) \d t, \quad \forall t_1 < t_2 \in I.
    \end{align*}
    Hence integrating w.r.t. $s \in [0, 1]$ gives:
    \begin{align*}
        d(\mu_{t_1}, \mu_{t_2})^2 \leq (t_2-t_1) \int_{t_1}^{t_2} \int_{[0, 1] \times \Om} \| v_t \|^2 \d \mu_t \d t, \quad \forall t_1 < t_2 \in I,
    \end{align*}
    which shows that $(\mu_t)_{t \in I}$ is $d$-absolutely continuous with $\left| \frac{\d}{\d t} \mu_t \right| \leq \| v_t \|_{L^2(\mu_t)}$ for a.e. $t \in I$.
\end{proof}

\begin{rem}
    Note that, to study absolutely continuous curves, we introduce the supplementary time variable $t \geq 0$. This time variable will model the optimization time in the~\cref{def:gradient_flow} of the gradient flow equation. It is not to be confused with the \NODE{} flow time $s \in [0, 1]$.
\end{rem}

As a consequence of~\cref{prop:characterization_absolute_continuity} we recover two useful results about absolutely continuous curves in $\Pp_2^\Leb([0,1] \times \Om)$.
Those are stated in the following~\cref{lem:approximation_along_curves,lem:differentiability_d}.
The first result concerns approximation along absolutely continuous curves. It states that the tangent velocity field $(v_t)_{t \in I}$ defined in~\cref{prop:characterization_absolute_continuity} indeed furnishes a first-order approximation of the curve $(\mu_t)_{t \in I}$ at every time $t \in I$.
This will be particularly useful for differentiating quantities related to $\mu_t$ (\cref{cor:flow_differential,cor:loss_differential}). The second result is an application and gives the derivative of the square-distance $d(\mu_t, \mu')^2$ along an absolutely continuous curve $(\mu_t)_{t \in I}$.

\begin{lem}[Adapted from \protect{\cite[Prop.8.4.6]{ambrosio2008gradient}}] \label{lem:approximation_along_curves}
    Let $(\mu_t)_{t \in I}$ be an absolutely continuous curve in $(\Pp^\Leb_2([0, 1] \times \Om), d)$ and let $v : I \times \RR^{p+1} \to \RR^p$ be the unique velocity field satisfying the conclusions of~\cref{prop:characterization_absolute_continuity}. Then, for $\d t$-a.e.\@ $t \in I$ and for any choice of $\gamma^h_t \in \Pi^\Leb_o(\mu_{t+h}, \mu_t)$, it holds:
    \begin{align*}
        \lim_{h \to 0} \left( \pi^1, \pi^2, \frac{1}{h}(\pi^3-\pi^2) \right)_\# \gamma^h_t = \left( \pi^1,\pi^2, v_t ) \right)_\# \mu_t \quad \text{in $\Ww_2([0,1] \times \Om \times \Om)$}
    \end{align*}
    and
    \begin{align*}
        \lim_{h \to 0} \frac{d(\mu_{t+h},(\Id+h(0,v_t))_\# \mu_t)}{|h|} = 0.
    \end{align*}
\end{lem}

\begin{proof}
    The proof only needs to be slightly adapted from the one of~\cite[Lem.8.4.6]{ambrosio2008gradient} but we rewrite it here for completeness.
    
    Let $(\varphi_n)_{n \geq 0}$ be a countable dense sequence in $\Cc^1_c([0,1] \times \Om)$. Then for $\d t$-a.e. $t \in I$ we have $\lim_{h \to 0} \frac{1}{|h|} d(\mu_{t+h}, \mu_t) = \left| \frac{\d}{\d t} \mu_t \right|$ and for every $n \geq 0$:
    \begin{align*}
        \lim_{h \to 0} \frac{\mu_{t+h}(\varphi_n) - \mu_t(\varphi_n)}{h} = \int \langle \nabla_\om \varphi_n, v_t \rangle \d \mu_t.
    \end{align*}
    Introducing some $\gamma^h_t \in \Pi^\Leb_o(\mu_{t+h}, \mu_t)$ we consider:
    \begin{align*}
        \nu^h \eqdef \left( \pi^1, \pi^2, \frac{1}{h}(\pi^3-\pi^2) \right)_\# \gamma^h_t.
    \end{align*}
    Then for any sequence $(h_n)$ converging to $0$ the sequence $(\nu^{h_n})$ is tight in $\Pp([0,1] \times \Om \times \Om)$ and we can consider a narrow limit point $\nu^0$. The marginal of $\nu^h$, and hence of $\nu^0$, on $[0,1] \times \Om$ is $\mu_t$ which allows writing by disintegration $\nu^0 = \int \nu^0_{s,\om} \d \mu_t(s,\om)$.
    Then we have for every $n \geq 0$:
    \begin{align*}
        \frac{\mu_{t+h}(\varphi_n) - \mu_t(\varphi_n)}{h} & = \frac{1}{h} \int \left( \varphi_n(s,\om') - \varphi_n(s,\om) \right) \d \gamma^h_t(s,\om,\om') \\
        & = \frac{1}{h} \int \left( \varphi_n(s,\om + h z) - \varphi_n(s,\om) \right) \d \nu^h(s,\om,z)
    \end{align*}
    and taking the limit $h \to 0$, by Lebesgue's theorem:
    \begin{align*}
        \int \langle \nabla_\om \varphi_n, v_t \rangle \d \mu_t = \int_{[0,1] \times \Om} \int_\Om \langle z, \nabla_\om \varphi_n(s,\om) \rangle \d \nu^0_{s,\om}(z) \d \mu_t(s,\om).
    \end{align*}
    For $(s, \om) \in [0,1] \times \Om$, let us denote by $\Tilde{v}_t(s,\om) \eqdef \int_\Om z \d \nu^0_{s,\om}(z)$ the first moment of $\nu^0_{s,\om}$. Then from the last equality and by a density argument it follows:
    \begin{align*}
        \div( (0, \Tilde{v}_t - v_t) \mu_t ) = 0,
    \end{align*}
    and in particular the continuity equation~\cref{continuity} is satisfied with the vector field $(0, \Tilde{v}_t)$.
    Let us now show:
    \begin{align*}
        \int_{[0,1] \times \Om} \int_\Om \| z \|^2 \d \nu^0_{s,\om}(z) \d \mu_t(s,\om) \leq \left| \frac{\d}{\d t} \mu_t \right|^2.
    \end{align*}
    Indeed we have:
    \begin{align*}
        \int_{[0,1] \times \Om} \int_\Om \| z \|^2 \d \nu^0_{s,\om}(z) \d \mu_t(s,\om) & \leq \liminf_{h \to 0} \int_{[0,1] \times \Om \times \Om} \| z \|^2 \d \nu^h(s, \om, z) \\
        & = \liminf_{h \to 0} \int_{[0,1] \times \Om \times \Om} \frac{1}{h^2} \| \om' - \om \|^2 \d \gamma^h_t(s,\om,\om') \\
        & = \liminf_{h \to 0} \frac{d(\mu_{t+h}, \mu_t)^2}{h^2} = \left| \frac{\d}{\d t} \mu_t \right|^2. 
    \end{align*}
    Whence by definition of $\Tilde{v}_t$ and Jensen's inequality:
    \begin{align*}
        \| \Tilde{v}_t \|_{L^2(\mu_t)} \leq \left| \frac{\d}{\d t} \mu_t \right| = \| v_t \|_{L^2(\mu_t)}
    \end{align*}
    from which it follows that $\Tilde{v}_t = v_t$ in $L^2(\mu_t)$ because of the minimality of $\| v_t \|_{L^2(\mu_t)}$ and the strict convexity of the $L^2$-norm. Moreover the above inequality is strict whenever $\nu^0_{s,\om}$ is not a Dirac mass in a set of $\mu_t$ positive measure. This implies that $\nu^0_{s,\om}$ is a Dirac mass for $\d \mu_t$-a.e. $(s,\om) \in [0, 1] \times \Om$ and that $\nu^0 = (\pi^1, \pi^2, v_t)_\# \mu_t$. This proves the narrow convergence of $\nu^h$ to the desired measure and together with the convergence of the second moments we have $\Ww_2$ convergence.

    Let us now estimate the distance between $\mu_{t+h}$ and $(\pi^1,\pi^2 + h (0, v_t))_\# \mu_t$ with the coupling $\gamma \eqdef (\pi^1,\pi^2 + h(0, v_t), \pi^3 )_\# \gamma^h_t \in \Pi^\Leb((\pi^1,\pi^2 +h(0, v_t))_\# \mu_t, \mu_{t+h})$. By the preceding result:
    \begin{align*}
        \frac{d((\pi^1,\pi^2 +h v_t)_\# \mu_t, \mu_{t+h})^2}{h^2} & \leq \int_{[0,1] \times \Om \times \Om} \frac{1}{h^2} \| \om + h v_t(s,\om) - \om' \|^2 \d \gamma^h_t(s,\om,\om') \\
        & = \int_{[0,1] \times \Om \times \Om} \| v_t(s,\om) - z \|^2 \d \nu^h(s,\om,z) \xrightarrow[h \to 0]{} 0 .
    \end{align*}
\end{proof}

\begin{lem}[Adapted from \protect{\cite[Thm.8.4.7]{ambrosio2008gradient}}] \label{lem:differentiability_d}
    Let $(\mu_t)_{t \in I}$ be an absolutely continuous curve in $\Pp^\Leb_2([0, 1] \times \Om)$, let $v : I \times [0,1] \times \om \to \Om$ be its \emph{tangent vector field} and let $\mu' \in \Pp^\Leb_2([0, 1] \times \Om)$. Then for $\d t$-a.e. $t \in I$:
    \begin{align*}
        \frac{\d}{\d t} d(\mu_t, \mu_t')^2 = 2 \int_0^1 \int_\Om \langle \om - \om', v_t(s, \om) \rangle \d \gamma(s, \om, \om'), \quad \forall \gamma \in \Pi^\Leb_o(\mu_t, \mu_t'). 
    \end{align*}
\end{lem}

\begin{proof}
    Having shown~\cref{lem:approximation_along_curves}, the proof is the same as the one of~\cite[Thm.8.4.7]{ambrosio2008gradient}.
\end{proof}

\section{Gradient flow dynamics} \label{sec:gradient_flow}

To train the \NODE{} model of~\cref{def:NODE} we consider performing \emph{gradient flow} on the parameter $\mu$ for the risk $L$ and w.r.t. the Conditional OT metric described in the previous section.
However the parameter set $\Pp_2^\Leb([0,1] \times \Om)$
equipped with the distance $d$ lacks a proper differential structure.
We will thus in this section give meaning to the notion of gradient flow of $L$. First, motivated by formal computations we will introduce a definition of gradient flow that is consistent with the one proposed by~\textcite{chen2018neural} for the training of \NODEs{} of finite width. Then, we will show this definition to be equivalent to the notion of \emph{curve of maximal slope} from the theory of gradient flow in metric spaces~\parencite{ambrosio2008gradient,santambrogio2017euclidean}. Finally, this equivalence will allow us to show well-posedness results for the gradient flow equation.

\subsection{Backward equation and adjoint variables}

The computation of the gradient will make use of a new ODE linked to~\cref{forward_weak}. This ODE should be understood as running backward over the time variable $s \in [0, 1]$ with the initial condition at $s = 1$. In the same way~\cref{forward_weak} models the processing of the data by a ResNet of infinite depth, the \emph{adjoint variables} $p$ solutions to this \emph{backward ODE} should be considered as modeling the quantities calculated when performing back-propagation over a deep ResNet. 

\begin{defn}[Adjoint variable] \label{def:adjoint_variable}
    Let $\mu \in \Pp_2^\Leb([0, 1] \times \Om)$ and $(x,y) \in \RR^{d+d'}$. Let $(x_\mu(s))_{s \in[0,1]}$ be the solution to~\cref{forward} with parameter $\mu$ and $x_\mu(0) = x$.
    Then we call \emph{adjoint variable associated to $\mu$, $x$ and $y$} the solution $(p_{\mu,x,y}(s))_{s \in[0, 1]}$ to the \emph{backward ODE}:
    \begin{equation} \label{backward}
        \forall s \in [0,1], \quad p_{\mu,x,y}(s) = \nabla_x \ell(x_\mu(1), y) + \int_s^1 \D_x F_{\mu(.|r)}(x_\mu(r))^\top p_{\mu,x,y}(r) \d r.
    \end{equation}
    When no ambiguity, the dependence of $p$ w.r.t. $\mu$, $x$ and $y$ is omitted and we simply write $p(s)$.
\end{defn}

The following proposition states the well-posedness of the backward equation under suitable assumptions on the feature map $\fmap$ and gives a useful representation of the adjoint variables.

\begin{prop} \label{prop:backward_wellposed}
    Let $\mu \in \Pp_2^\Leb([0, 1] \times \Om)$ and $(x,y) \in \RR^{d+d'}$. Assume $\fmap$ satisfies~\cref{fmap_assumption1,fmap_assumption2,fmap_assumption3}. Then there exists a unique solution to~\cref{backward} which is given by:
    \begin{equation} \label{adjoint_expr}
        \forall s \in[0,1], \quad p_{\mu,x,y}(s) = \Phi_{\mu,x}(s)^{-\top} \Phi_{\mu,x}(1)^\top \nabla_x \ell(x_\mu(1),y).
    \end{equation}
    where we define $\Phi_{\mu,x}$ to be the (matrix) solution of the linear ODE:
    \begin{equation} \label{Phi}
        \forall s \in [0, 1], \quad \Phi_{\mu,x}(s) = \Id + \int_0^s \D_x F_{\mu(.|r)}(x_\mu(r)) \Phi_{\mu,x}(r) \d r .
    \end{equation}
    When no ambiguity we simply write $\Phi_\mu(s)$ or even $\Phi(s)$.
\end{prop}

\begin{proof}
    Note that~\cref{backward} is a non-autonomous linear ODE w.r.t. the variable $p$. Thus, the representation~\cref{adjoint_expr} follows from the existence and uniqueness of $\Phi$ and to prove the result it suffices to show the map $s \mapsto \D_x F_{\mu(.|s)}(x(s))$ is integrable.

    First, as $\fmap$ is continuously differentiable w.r.t. $x$ with integrable differential, for almost every fixed $s \in [0, 1]$ the map $x \mapsto F_{\mu(.|s)}(x)$ is continuously differentiable with differential given by:
    \begin{align*}
        \D_x F_{\mu(.|s)}(x) = \int_\Om \D_x \fmap_\om(x) \d \mu(\om|s) .
    \end{align*}
    Moreover, by continuity of $s \mapsto x(s)$ the integrand $\D_x \fmap_\om (x(t))$ is measurable and so is the map $s \mapsto \D_x F_{\mu(.|s)}(x(s))$. Finally integrability follows as $\D_x \fmap_\om(x(s))$ has 2-growth w.r.t. $\om$ and $\int_\Om \| \om \|^2 \d \mu(\om|s)$ is integrable on $[0, 1]$.  
\end{proof}

The following result gives an alternate point of view on the adjoint variable $p$. Geometrically, it follows from~\cref{adjoint_expr} that $p$ lives in the co-tangent space of the flow $x$. Therefore in the case of a general (not necessarily with finite support) data distribution $\Dd \in \Pp(\RR^d \times \RR^{d'})$ it is convenient to see $p$ as the gradient of a potential $\psi$ over the variables $(x,y) \in \RR^d \times \RR^{d'}$.  

\begin{lem}
    Let $\mu \in \Pp_2^\Leb([0, 1] \times \Om)$. Then for every $(x,y) \in \RR^{d+d'}$ the associated adjoint variable $p$ can be expressed for every $s \in [0, 1]$ as:
    \begin{equation} \label{adjoint_expr_potential}
        p_{\mu,x,y}(s) = \nabla_x \psi_\mu (s, x_\mu(s), y) ,
    \end{equation}
    where $\psi_\mu$ is the unique solution to the transport equation:
    \begin{equation} \label{potential}
        \partial_s \psi_\mu + \langle \nabla_x \psi_\mu, F_{\mu(.|s)} \rangle = 0, \quad \psi_\mu(1,x,y) = \ell(x, y), \, \forall (x, y) \in \RR^d \times \RR^{d'}.
    \end{equation}
\end{lem}

\begin{proof}
    The solution to the transport equation can be given in the characteristic form:
    \begin{equation*}
        \forall s \in [0, 1], x \in \RR^d, \quad \psi_\mu(s, x_\mu(s), y) = \psi_\mu(1, x_\mu(1), y) = \ell (x_\mu(1), y).
    \end{equation*}
    One can then check that the r.h.s. of~\cref{adjoint_expr_potential} is indeed a solution of~\cref{backward}.
\end{proof}

\subsection{The gradient flow equation}

We motivate here by formal computations a definition of a gradient flow equation for the risk $L$. This \emph{adjoint sensitivity analysis} consists in using a Lagrangian form of the risk minimization problem to obtain an expression of the gradient w.r.t. the parameter $\mu$.

As the inputs are processed by our model through~\cref{forward} one can consider for a parameter $\mu \in \Pp_2^\Leb([0, 1] \times \Om)$ and  every time $s \in [0, 1]$ the distribution $\rho_\mu(.|s) \eqdef (x \mapsto x_\mu(s), \Id)_\# \Dd$ of the data at time $s$. Then $\lbrace \rho_\mu(.|s) \rbrace_{s \in [0,1]}$ is a narrowly continuous solution to the continuity equation:
\begin{equation} \label{forward_continuity}
    \partial_s \rho^*_\mu(.|s) + \div_x (F_{\mu(.|s)} \rho^*_\mu(.|s)) = 0
\end{equation}
and the risk associated to $\mu$ is $ L(\mu) = \int_{\RR^d \times \RR^{d'}} \ell(x,y) \d \rho_\mu(x,y|1) = \rho_\mu(.|1)(\ell)$.
We introduce a Lagrange multiplier $\psi$ to penalize the above continuity equation. For a parameter $\mu \in \Pp_2^\Leb([0, 1] \times \Om)$, a measurable family $\rho = \lbrace \rho(.|s) \rbrace_{s \in [0,1]}$ of probability measures on $\RR^d \times \RR^{d'}$ and a smooth test function $\psi : [0, 1] \times \RR^d \times \RR^{d'} \to \RR$, consider the Lagrangian $\Ll$ defined as:
\begin{equation}
    \Ll(\mu, \rho, \psi) \eqdef \rho(.|1)(\ell) - \rho(.|1)(\psi(1)) - \rho(.|0)(\psi(0)) + \int_0^1 \int_{\RR^{d+d'}} \left( \partial_s \psi + \langle \nabla_x \psi, F_{\mu(.|s)}\rangle \right) \d \rho(.|s) \d s .
\end{equation}
Using the definition of $F$ and changing the order of integration, the variation of $\Ll$ w.r.t. $\mu$ is given for every $\rho$ and $\psi$ by:
\begin{align*}
    \frac{\delta \Ll}{\delta \mu}(\mu, \rho, \psi) : (s, \om) \mapsto \int_{\RR^{d+d'}} \langle \nabla_x \psi(s,x, y), \fmap(\om, x) \rangle \d \rho (x, y| s) .
\end{align*}
Also, if $\rho = \rho_\mu$ is the solution of~\cref{forward_continuity} for the parameter $\mu$, we have the relation $\Ll(\mu, \rho_\mu, \psi) = L(\mu)$ for any test function $\psi$. Hence the variation of $L$ w.r.t. $\mu$ is:
\begin{align*}
    \frac{\delta L}{\delta \mu}(\mu) = \frac{\delta \Ll}{\delta \mu}(\mu, \rho_\mu, \psi) + \frac{\delta \Ll}{\delta \rho}(\mu, \rho_\mu, \psi) \frac{\delta \rho_\mu}{\delta \mu}(\mu),
\end{align*}
where the \emph{Lagrange multiplier} $\psi$ can be chosen arbitrarily. Also the variation of $\Ll$ w.r.t. the family of probability measures $\rho$, seen as the probability measure whose disintegration on $[0,1]$ is $\lbrace \rho_s \rbrace_{s \in [0, 1]}$ (with the fixed initial condition $\rho(.|0) = \Dd$), can be formally given by:
\begin{align*}
\frac{\delta \Ll}{\delta \rho}(\mu, \rho, \psi) = \left( \ell - \psi(1) \right) \delta_{s=1} + \partial_s \psi + \langle \nabla_x \psi, F_{\mu(.|s)} \rangle .   
\end{align*}
We observe that taking $\psi = \psi_\mu$ to be a solution of~\cref{potential} cancels $\frac{\delta \Ll}{\delta \rho}$ for every $\rho$ and hence: 
\begin{align*}
    \frac{\delta L}{\delta \mu}(\mu) & = \frac{\delta \Ll}{\delta \mu}(\mu, \rho_\mu, \psi_\mu) .
\end{align*}
By~\cref{prop:characterization_absolute_continuity} we know that for every absolutely continuous curve $(\mu_t)$ passing through $\mu$, its variation at $\mu$ is given by $\partial_t \mu_t = - \div((0,v) \mu)$ for some $v \in L^2(\mu)$. 
A notion of gradient of $L$ (for the ``differential'' structure of $\Pp_2^\Leb([0, 1] \times \Om)$) at $\mu$ could thus be defined as the unique solution to the variational problem:
\begin{align*}
    \nabla L(\mu) \in \, \argmin_{v \in L^2(\mu)} \frac{1}{2} \| v \|^2_{L^2(\mu)} - \langle \nabla_\om \frac{\delta L}{\delta \mu}(\mu), v \rangle_{L^2(\mu)} .
\end{align*}
This problem admits a unique solution $v^* \in L^2(\mu)$, provided that $\nabla_\om \frac{\delta L}{\delta \mu}(\mu) \in L^2(\mu)$, and using the relation~\cref{adjoint_expr_potential} between the adjoint variable $p_\mu$ and the potential $\psi_\mu$ we have:
\begin{align*}
    v^* & = \nabla_\om \frac{\delta L}{\delta \mu}(\mu) = \nabla_\om \frac{\delta \Ll}{\delta \mu}(\mu, \rho_\mu, \psi_\mu) \\
    \text{that is} \quad v^* & : (s, \om) \mapsto \int_{\RR^{d+d'}} \D_\om \fmap(\om, x)^\top \nabla_x \psi_\mu(x, y) \d \rho_\mu (x, y |s) = \EE_{x,y} \D_\om \fmap(\om,x_\mu(s))^\top p_{\mu,x,y}(s) .
\end{align*}

If the above calculations are purely formal, they motivate the following definition of gradient flow for $L$. In particular, this definition will be shown in the next section to be equivalent to the appropriate notion of gradient flow in metric spaces.

\begin{defn}[Gradient flow] \label{def:gradient_flow}
    Let $I \subset \RR$ be an interval.
    For $\mu \in \Pp_2^\Leb([0, 1] \times \Om)$, let us define:
    \begin{align} \label{gradient}
        \nabla L[\mu] : (s, \om) \mapsto \EE_{x,y} \D_\om \fmap(\om,x_\mu(s))^\top p_{\mu,x,y}(s) .
    \end{align}
    We say that a curve $t \in I \mapsto \mu_t \in \Pp^\Leb_2([0,1] \times \Om)$ is a \emph{gradient flow} for $L$ if it is locally absolutely continuous curve satisfying the continuity equation~\cref{continuity} with $v_t = - \nabla L[\mu_t]$ for $\d t$-a.e. $t \in I$.
\end{defn}

The following result is a useful representation formula for the gradient flow curves defined by~\cref{def:gradient_flow}: for every $t \geq 0$ the gradient flow $\mu_t$ at time $t$ is the pushforward of the initialization $\mu_0$ by a flow-map. The proof relies on classical results from transport equation theory~\parencite{ambrosio2008transport}.

\begin{prop}\label{prop:representation_curve}
    Assume $\fmap$ is twice continuously differentiable and satisfies~\cref{fmap_assumption1,fmap_assumption2,fmap_assumption3}. Let $(\mu_t)_{t \geq 0}$ be a gradient flow for the risk $L$ and consider for every $t \geq 0$ the vector field:
    \begin{align*}
        V_t : (s,\om) \mapsto \left( 0, \nabla L[\mu_t](s,\om) \right) = \left( 0, \EE_{x,y} \D_\om \fmap(\om,x_{\mu_t}(s))^\top p_{\mu_t, x, y}(s) \right) \in \RR \times \Om.
    \end{align*}
    Then for every $t \geq 0$ we have $\mu_t = (X_t)_\# \mu_0$ where $X_t$ is the flow-map solution of the ODE:
    \begin{align} \label{gradient_flow_characteristic}
        \frac{\d}{\d t} X_t(s,\om) = V_t(X_t(s,\om)), \quad X_0 = \Id.
    \end{align}
\end{prop}

\begin{proof}
    The existence and uniqueness of the flow-map $X_t$ for every $t \geq 0$ follows from the assumptions on $\fmap$ (in particular linear growth and local Lipschitz continuity of $\D_\om \fmap$ w.r.t. $\om$) and classical theory of ODEs. The flow-map representation of $\mu_t$ then follows from~\cite[Thm.3.2]{ambrosio2008transport} as for any initial value $(s,\om) \in [0,1] \times \Om$ the set of solutions to the ODE is the singleton $\lbrace (X_t(s,\om))_{t \geq 0} \rbrace$.
\end{proof}

\paragraph{Consistency with the adjoint gradient flow}

A case of particular interest for numerical applications is when the measure $\mu$ is discretized and approximated at every time $s \in [0, 1]$ by an empirical distribution.
Given $m \geq 0$ and $\theta = (\theta^j)_{1 \leq j \leq m} \in L^2([0, 1], \Om)^m$, we define the associated empirical distribution $\mu_\theta \in \Pp_2^\Leb([0,1] \times \Om)$ by:
\begin{align}
    \text{for $\d s$-a.e. $s \in [0,1]$}, \quad \mu_\theta(.|s) \eqdef \frac{1}{m} \sum_{j=1}^m \delta_{\theta^j(s)} \d s .
\end{align}
\emph{i.e.} $\mu_\theta$ is the measure whose disintegration at any time $s \in [0,1]$ is the empirical measure $\frac{1}{m} \sum_{j=1}^m \delta_{\theta^j(s)}$. Then we denote by $L(\theta) \eqdef L(\mu_\theta)$ the risk associated to $\theta$.
In the original work of~\textcite{chen2018neural}, the authors propose to train the Neural ODE parameterized by $\theta$ and minimize $L(\theta)$ by performing gradient descent for the \emph{adjoint gradient} defined as:
\begin{align*}
    \nabla_{\theta^j} L(\theta) \eqdef \EE_{x,y} \D_\om \fmap(\theta^j, x(s))^\top p(s),
\end{align*}
where $x$ and $p$ are respectively the solutions of~\cref{forward_weak,backward} for the parameter $\mu_\theta$.
One can observe that the adjoint gradient is the one calculated by~\cref{gradient} when $\mu = \mu_\theta$.
Given sufficient regularity assumptions on the feature map $\fmap$, we have by~\cref{prop:representation_curve} that  $(\mu_t)_{t \geq 0}$ is a gradient flow in the sense of~\cref{def:gradient_flow} with $\mu_0 = \mu_{\theta_0}$ if and only if $\mu_t = \mu_{\theta_t}$ for every $t \geq 0$ and $(\theta_t)_{t \geq 0}$ is a gradient flow for the above adjoint gradient.

\subsection{Gradient flows as curves of maximal slope}

There exists a large body of mathematical works devoted to the generalization of the classical theory of gradient flows to functionals over metric spaces.
\textcite{ambrosio2008gradient} give an in-depth presentation of this theory. Complementary and more synthetic presentations are given by~\textcite{ambrosio2013user,santambrogio2017euclidean}. Based on those works, we introduce here another definition of gradient flows for the risk $L$ which is the one of \emph{curves of maximal slope} and show it coincides with the definition from the previous section.

\subsubsection{Curves of maximal slope in metric spaces}

When $Z$ is a Euclidean space, and $f$ is a smooth function the gradient flow of $f$ is defined as the solution of the ODE $\frac{\d}{\d t} z_t = -\nabla f(z_t)$.
Then such a gradient flow satisfies:
\begin{align*}
    \frac{\d}{\d t} f(z_t) = - \| \nabla f(z_t) \|^2 = - \frac{1}{2} \left( \| \frac{\d}{\d t} z_t \|^2 + \| \nabla f(z_t) \|^2 \right),
\end{align*}
whereas for any other smooth curve $(y_t)$ we have by Young's inequality:
\begin{align} \label{young}
    \frac{\d}{\d t} f(y_t) = - \langle \nabla f(y_t), \frac{\d}{\d t} y_t \rangle \geq - \frac{1}{2} \left( \| \frac{\d}{\d t} y_t \|^2 + \| \nabla f(y_t) \|^2 \right) 
\end{align}
with equality if and only if $\frac{\d}{\d t} y_t = - \nabla f(y_t)$. Hence, we see that imposing equality in the above inequality gives a characterization of gradient flow curves in the Euclidean case.
The definition of \emph{curves of maximal slope} is based on the generalization of this characterization to metric spaces. For example, a generalization of the speed's norm $\| \dd{t} z_t \|$ is given by the metric derivative $\left| \dd{t}z_t \right|$.
To make sense of the gradient's norm $\| \nabla f(z) \|$ we need to introduce the notion of \emph{upper gradient}.

\begin{defn}[Upper gradient {{\cite[Def.1.2.1]{ambrosio2008gradient}}}] \label{def:upper_gradient}
    Let $(Z, d)$ be a complete metric space and $f: Z \to \RR$ be a function. The map $g : Z \to [0, +\infty]$ is an \emph{upper gradient} for $f$ if for every absolutely continuous curve $(z_t)_{t \in I}$ we have that $t \mapsto g(z_t)$ is measurable and:
    \begin{align*}
        \left| f(z_{t_1}) - f(z_{t_2}) \right| \leq \int_{t_1}^{t_2} g(z_t) \left| \frac{\d}{\d t} z_t \right| \d t, \quad \forall t_1 \leq t_2 \in I.
    \end{align*}
    When there is no ambiguity, an upper gradient of $f$ will simply be denoted by $| \nabla f|$.
\end{defn}

Given an upper gradient for $f$, the definition of curves of maximal slope consists in imposing equality in~\cref{young}.

\begin{defn}[Curve of maximal slope~{{\cite[Def.1.3.2]{ambrosio2008gradient}}}] \label{def:maximal_slope}
    Let $(Z, d)$ be a complete metric space, $I \subset \RR$ be an interval and $f : Z \to \RR$ a function with $|\nabla f|$ an upper gradient for $f$. We say that $(z_t)_{t \in I}$ is a \emph{curve of maximal slope} for $f$ w.r.t. $|\nabla f|$ if it satisfies:
    \begin{enumerate}
        \item $(z_t)_{t \in I}$ is locally absolutely continuous,
        \item the map $t \mapsto f(z_t)$ is non-increasing,
        \item for $\d t$-a.e. $t \in I$ it holds $\frac{\d}{\d t} f(z_t) \leq - \frac{1}{2} \left( \left| \frac{\d}{\d t} z_t \right|^2  + | \nabla f |^2(z_t) \right)$.
    \end{enumerate}
    Moreover, if $\lim_{t \to \inf I} z_t = z$ exists, then we say $(z_t)$ is a \emph{curve of maximal slope starting at $z$}.
\end{defn}

\begin{rem}[About the various definitions of \emph{curves of maximal slope}]
    There are various definitions of the notion of curve of maximal slope in metric spaces, see for example~\cite[Sec.4]{ambrosio2013user} for a discussion about the various definitions and their relations. Our definition is the same as the one used in~\cite[Def.2.12]{hauer2019kurdyka}. In particular, it implies the following \emph{Energy Dissipation Equality}~\parencite[Prop.2.14]{hauer2019kurdyka}:
    \begin{align} \label{EDE} \tag{EDE}
        f(z_{t_1}) - f(z_{t_2}) = \frac{1}{2} \int_{t_1}^{t_2} \left( \left| \frac{\d}{\d t} z_t \right|^2  + | \nabla f |^2(z_t) \right) \d t, \quad \forall t_1, t_2 \in I.
    \end{align}
    This definition differs from the one exposed in~\cite[Def.2.2]{schiavo2023local} (see also~\cite[Def.4.4]{muratori2020gradient}) as the map $t \mapsto f(z_t)$ need not be locally absolutely continuous. The difference between these two definitions is discussed in~\cite[Rem.2.6]{schiavo2023local} but observe that for our purpose \emph{(i)} the loss $L$ will be shown to be locally Lipschitz in~\cref{cor:loss_lipschitz}, hence implying that $L(\mu_t)$ is locally absolutely continuous, \emph{(ii)} the gradient norm $\| \nabla L[\mu] \|_{L^2(\mu)}$ will be shown to be an upper gradient in~\cref{prop:upper_gradient}. For these reasons, we need not here make the distinction between these two definitions and prefer weaker assumptions.
\end{rem}

Note that there are \emph{a priori} no reasons for~\cref{def:gradient_flow,def:maximal_slope} to define the same notion of gradient flow for the risk $L$.
In particular, the first definition uses the existence of the adjoint variable $p$ and thus some regularity on $\fmap$. In contrast, the second definition requires an upper gradient which is yet unspecified for $L$.
Taking appropriate assumptions on the feature map $\fmap$, we show in the rest of this section that the risk $L$ is sufficiently regular for the two definitions to coincide. This will be the content of~\cref{thm:equivalence_GF_CMS}.

\subsubsection{Curve of maximal slope for the risk $L$}
Provided with~\cref{def:maximal_slope}, we seek to characterize the curves of maximal slope for the risk $L$ in the metric space $\Pp_2^\Leb([0,1] \times \Om)$.
We first show the \NODE's output is a locally Lipschitz function of the parameter $\mu \in \Pp^\Leb_2([0, 1] \times \Om)$.

\begin{assumption}[local Lipschitz continuity w.r.t. $\om$] \label{fmap_assumption2}
    Assume that $\fmap : \Om \times \RR^d \to \RR^d$ is locally Lipschitz w.r.t. $\om$ with a Lipschitz constant that grows at most linearly w.r.t. $\Om$: for every $R \geq 0$ there exists a constant $C(R)$ s.t.:
    \begin{align*}
        \forall x \in B(0, R), \, \forall \om, \om' \in \Om, \quad \| \fmap(\om, x) - \fmap (\om', x) \| \leq C(R) ( 1+\max(\|\om\|, \|\om'\|) ) \| \om - \om' \| .
    \end{align*}
\end{assumption}

\begin{lem}[local Lipschitz continuity of the flow] \label{lem:flow_lipschitz}
    Assume $\fmap$ satisfies~\cref{fmap_assumption1,fmap_assumption2} and consider some input $x \in \RR^d$. Then the map $\mu \in \Pp^\Leb_2([0, 1] \times \Om) \mapsto (x_\mu(s))_{s \in [0, 1]} \in \Cc([0, 1], \RR^d)$ is locally Lipschitz. More precisely,  for every $\Ee \geq 0$ there exists a constant $C = C(\Ee)$ s.t.:
    \begin{align*}
        \sup_{s \in [0, 1]} \| x_\mu(s) - x_{\mu'}(s) \| \leq C d(\mu, \mu'),
    \end{align*}
    for every $\mu, \mu' \in \Pp^\Leb_2([0, 1] \times \Om)$ with $\Ee_2(\mu), \Ee_2(\mu') \leq \Ee$. Moreover, the constant $C$ can be chosen uniformly over $x$ in a compact set.
\end{lem}

\begin{proof}
    Consider $x \in \RR^d$, $\Ee \geq 0$ and $\mu, \mu'$ such as in the statement.
    We denote by $(x(s))_{s \in [0, 1]}$ and $(x'(s))_{s \in [0, 1]}$ the flow associated to $x$  and to the parameters $\mu$ and $\mu'$ respectively.
    Let $R \geq 0$ be such that $\|x \| \leq R$.
    Then by~\cref{prop:flow_wellposed} the trajectories $x, x'$ are uniformly bounded by some $R' = R'(R, \Ee)$ . Then using~\cref{forward} we have for every $s \in [0, 1]$:
    \begin{align*}
        \| x(s) - x'(s) \| \leq & \| x(0) - x'(0) \| + \| \int_0^s \int_\Om \fmap(\om, x(r)) \d \mu(\om|r) \d r - \int_0^s \int_\Om \fmap(\om, x'(r)) \d \mu'(\om|r) \d r   \| \\
        \leq & \| x(0) - x'(0) \| + \int_0^s \int_\Om \| \fmap(\om,x(r)) - \fmap(\om, x'(r)) \| \d \mu(\om|r) \d r \\
        & + \int_0^s \|  \int_\Om \fmap(\om, x'(r)) \d (\mu - \mu')(\om|r) \| \d r . 
    \end{align*}
    For the first integral note that using the local Lipschitz continuity of $\fmap$ w.r.t. $x$ in~\cref{fmap_assumption1} we have for every $r \in [0, 1]$:
    \begin{align*}
        \int_\Om \| \fmap(\om,x(r)) - \fmap(\om, x'(r)) \| \d \mu(\om|r) \leq C_1 \int_\Om (1+\| \om \|^2)\d \mu(\om|r) \| x(r) - x'(r) \|,
    \end{align*}
    where $C_1 =  C_1(R, \Ee)$. For the second integral, note that at fixed $r \in [0, 1]$, if $\gamma \in \Gamma_o(\mu(.|r), \mu'(.|r))$ is an optimal coupling between $\mu(.|r)$ and $\mu'(.|r)$ then:
    \begin{align*}
          \int_\Om \fmap(\om, x'(r)) \d (\mu - \mu')(\om | r) = \int_{\Om^2} ( \fmap(\om, x'(r)) - \fmap(\om', x'(r)) ) \d \gamma(\om, \om')\,.
    \end{align*}
    Using the local Lipschitz continuity of $\fmap$ w.r.t. $\om$ in~\cref{fmap_assumption2} and the optimality of $\gamma$:
    \begin{align*}
        \| \int_\Om \fmap(\om, x'(r)) \d (\mu - \mu')(\om | r) \| & \leq \int_{\Om^2} \|  \fmap(\om, x'(r)) - \fmap(\om', x'(r)) \| \d \gamma(\om, \om') \\
        & \leq  \int_{\Om^2} C_2 (1 + \max(\|\om\|, \|\om'\|)) \| \om - \om' \|  \d \gamma(\om, \om') \\
        & \leq \sqrt{3} C_2 \left(1 + \Ee_2(\mu(.|r)) + \Ee_2(\mu'(.|r)) \right)^{1/2} \Ww_2(\mu(.|r), \mu'(.|r)) ,
    \end{align*}
    where $C_2 = C_2(R, \Ee)$. Integrating those inequalities gives by Grönwall's lemma that for $s \in [0, 1]$:
    \begin{align*}
        \| x(s) - x'(s) \|  \leq \sqrt{3} e^{C_1 (1+\Ee_2(\mu))}  C_2 \int_0^s  \left(1 + \Ee_2(\mu(.|r)) + \Ee_2(\mu'(.|r)) \right)^{1/2} \Ww_2(\mu(.|r), \mu'(.|r)) \d r \leq C_3 d(\mu, \mu') ,
    \end{align*}
    where $C_3 = C_3(R, \Ee)$.
\end{proof}

As an immediate corollary of the above proposition we get that, provided the loss function $\ell$ is itself locally Lipschitz, then the risk $L$ is also a locally Lipschitz function of $\mu$.

\begin{cor}[local Lipschitz continuity of the risk] \label{cor:loss_lipschitz}
    Assume that $\fmap$ satisfies~\cref{fmap_assumption1,fmap_assumption2} and $\ell$ is locally Lipschitz w.r.t. $x$. Then the risk map $L : \mu \in \Pp_2^\Leb([0,1] \times \Om) \mapsto L(\mu)$ is locally Lipschitz.
\end{cor}

Assuming more regularity on the map $\fmap$, one can express the first variation $\delta x$ of the flow map with respect to a variation of the parameter transported by a velocity field $v \in L^2(\mu)$.

\begin{assumption}[Differentiability of $\fmap$] \label{fmap_assumption3}
    Assume that $\fmap$ is continuously differentiable and s.t.
    \begin{enumerate}
        \item $\D_x \fmap$ grows at most quadratically with $\om$: for every $R \geq 0$ there exists a constant $C(R)$ s.t.
        \begin{align*}
           \forall x \in B(0, R), \, \forall \om \in \Om, \quad  \| \D_{x} \fmap(\om, x) \| \leq C(R) (1+\|\om\|^2) .
        \end{align*}
        
        \item $\D_{\om} \fmap$ grows at most linearly with $\om$: for every $R \geq 0$ there exists a constant $C = C(R)$ s.t.
        \begin{align*}
            \forall x \in B(0, R), \, \forall \om \in \Om, \quad \| \D_{\om} \fmap(\om,x) \| \leq C(R)(1+\| \om \|).
        \end{align*}
    \end{enumerate}
    
\end{assumption}

\begin{prop} \label{prop:flow_differential}
    Assume $\fmap$ satisfies~\cref{fmap_assumption1,fmap_assumption2,fmap_assumption3}. Consider $\mu \in \Pp_2^\Leb([0,1] \times \Om)$ and a velocity field $v : [0,1] \times \Om \to \Om$ in $L^2(\mu)$. For $t \in \RR$, define $\mu_t \eqdef (\Id + t(0,v))_\# \mu$. Then, for $x \in \RR^d$, $(x_{\mu_t})_{t \in \RR}$ is differentiable in $\Cc([0,1], \RR^d)$ at $t = 0$ and $\delta x \eqdef \frac{\d}{\d t} x_{\mu_t} |_{t=0}$ is the solution to:
    \begin{align} \label{flow_variation}
        \forall s \in [0, 1], \, \delta x(s) = \int_0^s \D F_{\mu(.|r)}(x_\mu(r)) \delta x(r) \d r + \int_0^s \int_\Om \D_\om \fmap(\om, x_\mu(r)) v(r,\om) \d \mu(\om|r) \d r.
    \end{align}
\end{prop}

\begin{proof}
     First, thanks to~\cref{fmap_assumption3}, for $\nu \in \Pp_2(\Om)$ the map $F_\nu : \RR^d \to \RR^d$ is differentiable with $\D F_\nu : x \mapsto \int_\Om \D_x \fmap(\om, x) \d \nu(\om)$.
    Also, $\delta x$ is well-defined as the unique solution of~\cref{flow_variation} and:
    \begin{align*}
        \forall s \in [0, 1], \quad \delta x(s) =  \int_0^s \int_\Om \Phi_{\mu,x}(s) \Phi_{\mu,x} (r)^{-1}  \D_\om \fmap(\om, x(r)) v(r,\om) \d \mu(r,\om).
    \end{align*}
    
    For simplicity, in the rest of the proof we will denote by $x_t \eqdef x_{\mu_t}$ for any $t \in \RR$.
    Let us then show that $\delta x$ is the derivative of $x_t$ at $t = 0$. For $t \neq 0$, consider the normalized increment:
    \begin{align*}
        z_t \eqdef \frac{1}{t} (x_t - x_0) \in \Cc([0, 1], \RR^d) .
    \end{align*}
    Then we have by definition of $x_t$ and $x_0$ that for every $s \in [0, 1]$:
    \begin{align*}
        z_t(s) = & \frac{1}{t} \int_0^s \int_\Om \fmap(\om,x_t(r)) \d \mu_t(r, \om) - \frac{1}{t} \int_0^s \int_\Om \fmap(\om, x_0(r)) \d \mu(r,\om) \\
        = & \frac{1}{t} \int_0^s \int_\Om \left(  \fmap(\om + t v(r,\om),x_t(r)) - \fmap(\om, x_0(r)) \right) \d \mu(r,\om) \\
        = & \int_0^s \int_\Om \left( \int_0^1 \D_x \fmap(\om,x_0(r)+utz_t(r)) \d u \right) \cdot z_t(r) \d \mu(r,\om) \\
        & +  \int_0^s \int_\Om \left( \int_0^1 \D_\om \fmap (\om + u t v(r,\om), x_t(r)) \d u \right) \cdot v(r,\om) \d \mu(r,\om).
    \end{align*}
    Hence $z_t$ is solution of the linear ODE $z_t(s) = \int_0^s \left( A_t(r) \cdot z_t(r) + b_t (r) \right) \d r$ where we defined for $\d r$-a.e. $r \in [0,1]$:
    \begin{align*}
        A_t(r) & \eqdef \int_\Om \int_0^1 \D_x \fmap(\om,x_0(r)+utz_t(r)) \d u \d \mu(\om|r) \\
        b_t(r) & \eqdef \int_\Om \int_0^1 \D_\om \fmap (\om + u t v(r,\om), x_t(r)) \cdot v(r,\om) \d u \d \mu(\om|r)
    \end{align*}
    and in order to prove that $z_t \to \delta x$ in $\Cc([0,1], \RR^d)$ as $t \to 0$ it suffices to show that $A_t$ and $b_t$ converge respectively in $L^1([0,1])$ to:
    \begin{align*}
        A(r) \eqdef \int_\Om \D_x \fmap(\om, x_0(r)) \d \mu(\om|r), \quad \text{and} \quad b(r) \eqdef \int_\Om \D_\om \fmap (\om, x_0(r)) \cdot v(r) \d \mu(\om|r).
    \end{align*}
    Indeed, note that $d(\mu_t, \mu) \leq t \| v \|_{L^2(\mu)}$ and the family $(z_t)_{t \in [-1,1]}$ is bounded $\Cc([0,1], \RR^d)$ by~\cref{lem:flow_lipschitz}.
    Thus for $t \in \RR$:
    \begin{align*}
        \int_0^1 \left| A_t(r) - A(r) \right| \d r \leq \int_0^1 \int_\Om \left| \int_0^1 \D_x \fmap (\om,x_0(r)+utz_t(r)) \d u - \D_x \fmap (\om, x_0(r))  \right| \d \mu(r,\om) \xrightarrow[t \to 0]{} 0
    \end{align*}
    where~\cref{fmap_assumption3} allows us to bound the integrand by an integrable function and to apply Lebesgue's theorem, showing that $A_t \to A$ as $t \to 0$ in $L^1([0,1])$.
    Similarly for $b_t$:
    \begin{align*}
        \int_0^1 \left| b_t(r) - b(r) \right| \d r \leq \int_0^1 \int_\Om \left| \int_0^1 \D_\om \fmap (\om +utv(r,\om), x_t(r)) \d u - \D_\om \fmap (\om, x_0(r)) \right| \| v(r,\om) \| \d \mu(r,\om) \xrightarrow[t \to 0]{} 0.
    \end{align*}
\end{proof}

A direct consequence of the previous result is the differentiability of the flow map and consequently of the risk along absolutely continuous curves.

\begin{cor}[Differentiability of the flow] \label{cor:flow_differential}
    Assume $\fmap$ satisfies~\cref{fmap_assumption1,fmap_assumption2,fmap_assumption3}. Let $I \subset \RR$ be an interval and consider $(\mu_t)_{t \in I}$ an absolutely continuous curve in $\Pp^\Leb_2([0,1] \times \Om)$ satisfying the continuity equation:
    \begin{align*}
        \partial_t \mu_t + \div ((0,v_t) \mu_t) = 0 \quad \text{over $I \times [0, 1] \times \Om$.}
    \end{align*}
    Consider some $x \in \RR^d$. Then $(x_{\mu_t})_{t \in I}$ is an absolutely continuous curve in $\Cc([0,1], \RR^d)$ and is differentiable in $\Cc([0,1], \RR^d)$ for $\d t$-a.e. $t \in I$ with $\delta x_t \eqdef \frac{\d }{\d t} x_{\mu_t}$ the solution to:
    \begin{align} \label{flow_differential}
        \forall s \in [0, 1], \, \delta x_t(s) = \int_0^s \D F_{\mu_t(.|r)}(x_{\mu_t}(r)) \delta x_t(r) \d r + \int_0^s \int_\Om \D_\om \fmap(\om, x_{\mu_t}(r)) v_t(r,\om) \d \mu_{t}(\om|r) \d r.
    \end{align}
\end{cor}

\begin{proof}
    For $t \in I$ we use the shortcut notation $x_t \eqdef x_{\mu_t}$. The fact that $(x_t)_{t \in I}$ is absolutely continuous follows from~\cref{lem:flow_lipschitz} stating that the flow map is locally Lipschitz. To prove the result it hence suffices to show that $\delta x_t$ is the derivative of $t \mapsto x_t$ in $\Cc([0,1], \RR^d)$.

    Note that, without loss of generality we can consider $v_t$ to be the (uniquely defined) tangent velocity field of the curve $(\mu_t)_{t \in I}$. Indeed if $\Tilde{v}_t$ is the tangent velocity field then we have by~\cref{prop:characterization_absolute_continuity} that in the sense of distributions:
    \begin{align*}
        \div((0, v_t - \Tilde{v}_t) \mu_t) = 0.
    \end{align*}
    Hence for every $x \in \RR^d$ and every $s \in [0, 1]$:
    \begin{align*}
        \int_0^s \int_\Om \D_\om \fmap(\om, x_t(r)) v_t(r,\om) \d \mu_t(r,\om) = \int_0^s \int_\Om \D_\om \fmap(\om, x_t(r)) \Tilde{v}_t(r,\om) \d \mu_t(r,\om) 
    \end{align*}
    and the definition of $\delta x_t$ stays unchanged.
    Then, assuming $v_t$ is the tangent velocity field to the curve $\mu_t$, we can consider a subset $\Lambda \subset I$ of full Lebesgue measure such that the conclusions of~\cref{lem:approximation_along_curves} hold. For every $t \in \Lambda$ and every $h \neq 0$ consider $\Tilde{\mu}^h_t \eqdef (\Id +h(0,v_t))_\# \mu_t$ and $\Tilde{x}^h_t$ the associated flow. Then by~\cref{prop:flow_differential}:
    \begin{align*}
        \| \frac{x_{t+h} - x_t}{h} - \delta x_t \|_{\Cc([0,1])} \leq \|  \frac{\Tilde{x}^h_t - x_t}{h} - \delta x_t \|_{\Cc([0,1])} + \| \frac{x_{t+h} - \Tilde{x}^h_t}{h} \|_{\Cc([0,1])} \xrightarrow[h \to 0]{} 0
    \end{align*}
    where the first term goes to $0$ by application of~\cref{prop:flow_differential}.The second term also goes to $0$ by the fact that the flow map is locally Lipschitz, thus $\| x_{t+h} - \Tilde{x}^h_t \|_{\Cc([0,1])} \leq C d(\mu_{t+h}, \Tilde{\mu}^h_t)$ for some constant $C$ and $\frac{1}{h} \| x_{t+h} - \Tilde{x}^h_t \|_{\Cc([0,1])} \to 0$ by~\cref{lem:approximation_along_curves}. Note that, as $\Lambda \subset I$ is independent of $x$, it follows that the curve $t \mapsto x_t$ is differentiable at every $t \in \Lambda$ for every $x \in \RR^d$.
\end{proof}

\begin{cor}[Differentiability of the loss] \label{cor:loss_differential}
    Assume $\fmap$ satisfies~\cref{fmap_assumption1,fmap_assumption2,fmap_assumption3} and $\ell$ is continuously differentiable. Let $I \subset \RR$ be an interval and $(\mu_t)_{t \in I}$ be as in~\cref{cor:flow_differential}.
    Then $(L(\mu_t))_{t \in I}$ is absolutely continuous and for almost every $t \in I$:
    \begin{align*}
        \frac{\d}{\d t} L(\mu_t) = \int_{[0, 1] \times \Om} \langle \EE_{x,y} \D_\om \fmap(\om, x_{\mu_t}(s))^\top p_{\mu_t,x,y}(s), v_{t}(s,\om) \rangle \d \mu_t(s, \om) .
    \end{align*}
\end{cor}

\begin{proof}
    First, the fact that $t \mapsto L(\mu_t)$ is absolutely continuous follows from the fact that $\mu \mapsto L(\mu)$ is locally Lipschitz, as shown in~\cref{cor:loss_lipschitz}. It remains to show the formula for its derivative.

    For $t \in I$ and $(x,y) \in \RR^d \times \RR^{d'}$ use the shortcut notations $x_t \eqdef x_{\mu_t}$, $p_t \eqdef p_{\mu_t, x, y}$ and $\Phi_t \eqdef \Phi_{\mu_t, x}$. By the proof of~\cref{cor:flow_differential} we know that there exists a subset $\Lambda \subset I$ of full Lebesgue measure such that for every $t \in \Lambda$, the map $t \mapsto x_t$ is differentiable at $t$ for every $x \in \RR^d$. By Lebesgue theorem, the map $t \mapsto L(\mu_t)$ is differentiable at every $t \in \Lambda$ with:
    \begin{align*}
        \frac{\d}{\d t} L(\mu_t) = \EE_{x,y} \langle \nabla_x \ell (x_t(1),y), \delta x_t(1) \rangle,
    \end{align*}
    where, at fixed $x \in \RR^d$, $\delta x_t$ verifies~\cref{flow_differential} and is given by
    $$
        \delta x_t(1) = \int_{[0,1] \times \Om} \Phi_t(1) \Phi_t(s)^{-1} \D_\om \fmap(\om, x_t(s)) v_t(s, \om) \d \mu_t(s,\om). 
    $$
    Also, the adjoint variable $p_t$ is given by $p_t(s) =  \Phi_t(s)^{- \top} \Phi_t(1)^\top \nabla_x \ell(x_t(1), y)$. Hence by changing the order of integration, we see that for $t \in \Lambda$:
    \begin{align*}
         \frac{\d }{\d t} L(\mu_t) = \EE_{x,y} \langle \nabla_x \ell(x_t(1), y), \delta x_t(1) \rangle = \int_{[0, 1] \times \Om} \langle \EE_{x,y} \D_\om \fmap(\om, x_t(s))^\top p_t(s), v_t(s, \om) \rangle \d \mu_t(s,\om).
    \end{align*}
\end{proof}

Thanks to~\cref{cor:flow_differential} (which can be seen as a chain-rule formula) one can show that the gradient norm $\| \nabla L[\mu] \|_{L^2(\mu)}$ gives an upper gradient for the risk $L$ in the sense of~\cref{def:upper_gradient}. Moreover the following~\cref{prop:upper_gradient} shows it corresponds to the notion of \emph{local slope} (\cite[Def.1.2.4]{ambrosio2008gradient}). The result relies on the following lemma.

\begin{lem}[Continuity of the adjoint variable $p$] \label{lem:adjoint_continuous}
    Assume $\fmap$ satisfies~\cref{fmap_assumption1,fmap_assumption2,fmap_assumption3}.
    Then, for fixed $(x,y) \in \RR^{d+d'}$, the map $\mu \mapsto p_{\mu, x, y} \in \Cc([0,1] \times \RR^d)$ is $d$-continuous on $\Pp_2^\Leb([0,1] \times \Om)$.
\end{lem}

\begin{proof}
    Let $\mu \in \Pp_2^\Leb([0,1] \times \Om)$ and consider a sequence $(\mu_n)_{n \geq 0}$ in $\Pp_2^\Leb([0,1] \times \Om)$ such that $d(\mu_n, \mu) \to 0$.
    Fix a pair $(x,y) \in \RR^{d+d'}$ and use the shortcuts $x_n \eqdef x_{\mu_n}$ (resp. $x \eqdef x_\mu$) and $p_n \eqdef p_{\mu_n, x, y}$ (resp. $p \eqdef p_{\mu, x, y}$). By~\cref{lem:flow_lipschitz} we already have $x_n \to x$ in $\Cc([0,1])$ and we show now that $p_n \to p$ in $\Cc([0,1])$ using Ascoli's theorem.

    Remark that by the assumptions on $\fmap$, all the trajectories $x$, $x_n$, $p$ and $p_n$ stay in a bounded set $B(0,R)$ for some $R \geq 0$. Also, as $\mu_n \to \mu$, we have that the sequence $(\mu_n)$ has uniformly integrable second moment and for every $\eps > 0$ we can find a $k \geq 0$ s.t.
    $$
    \int_{\| \om \| \geq k} (1 + \| \om \|^2) \d \mu, \, \sup_{n \geq 0} \int_{\| \om \| \geq k} (1 + \| \om \|^2) \d \mu_n \leq \eps.
    $$
    Then for $n \geq 0$ and $s_1 \leq s_2 \in [0,1]$ we have by~\cref{backward} and the assumptions on $\fmap$:
    \begin{align*}
        \| p_n(s_2) - p_n(s_1) \| & \leq \int_{s_1}^{s_2} \int_\Om \| \D_x \fmap(\om, x_n(r)) \| \| p_n(r) \| \d \mu_n(r,\om) \\
        & \leq C \int_{s_1}^{s_2} \int_\Om (1 + \| \om \|^2) \d \mu_n(r,\om) \\
        & \leq C ( \eps + (1+k^2) |s_2 - s_1|),
    \end{align*}
    where $C = C(R)$. Hence the sequence $(p_n)_{n \geq 0}$ is equicontinuous and, up to a subsequence, we have $p_n \to \Bar{p} \in \Cc([0,1])$. Let us then show $\Bar{p} = p$. Indeed using the initial condition we have for $n \geq 0$ and $s \in [0,1]$:
    \begin{align*}
        p_n(s) = \nabla_x \ell(x_n(1),y) + \int_s^1 \int_\Om \D_x \fmap(\om, x_n(r))^\top p_n(r) \d \mu_n(r,\om).
    \end{align*}
    First we have $\nabla_x \ell(x_n(1), y) \xrightarrow[n \to \infty]{} \nabla_x \ell(x(1),y)$. Also, note that by the assumptions on $\fmap$ we have $\D_x \fmap(\om, x_n(r))^\top p_n(r) \leq C(1+\| \om \|^2)$ and $\D_x \fmap(\om, x_n(r))^\top p_n(r) \to \D_x \fmap(\om, x(r))^\top \Bar{p}(r)$ locally uniformly over $[0,1] \times \Om$. Hence by the properties of $\Ww_2$-convergence, we can take the limit in the above equation to obtain:
    \begin{align*}
        \Bar{p}(s) = \nabla_x \ell(x(1),y) + \int_s^1 \int_\Om \D_x \fmap(\om, x(r))^\top \Bar{p}(r) \d \mu(r,\om),
    \end{align*}
    \emph{i.e.} $\Bar{p} = p$ by uniqueness of the solutions to~\cref{backward}.
\end{proof}

\begin{lem}[Continuity of $\| \nabla L(\mu) \|_{L^2(\mu)}$ ] \label{lem:gradient_continuous}
    Assume $\fmap$ satisfies~\cref{fmap_assumption1,fmap_assumption2,fmap_assumption3}. Then the map $\mu \mapsto \| \nabla L[\mu] \|_{L^2(\mu)}$ is $d$-continuous on $\Pp_2^\Leb([0,1] \times \Om)$.
\end{lem}

\begin{proof}
    Let $\mu \in \Pp_2^\Leb([0,1] \times \Om)$ and consider a sequence $(\mu_n)_{n \geq 0}$ in $\Pp_2^\Leb([0,1] \times \Om)$ such that $d(\mu_n, \mu) \to 0$.
    For an input $x \in \RR^d$, denote by $x_n$ (resp. $x$) the flow associated to $\mu_n$ (resp. $\mu$) and starting from $x$.
    Similarly introduce the adjoint variables $(p_n)_{n \geq 0}$ and $p$.
    Then by~\cref{lem:flow_lipschitz} and~\cref{lem:adjoint_continuous} we have that $x_n \to x$ and $p_n \to p$ in $\Cc([0,1], \RR^d)$. As a consequence the sequence of continuous maps 
    $$
        f_n : (r, \om) \mapsto \EE_{x,y} \D_\om \fmap(\om, x_n(r))^\top p_n(r)
    $$ 
    converges locally uniformly to the map $f : (r,\om) \mapsto \EE_{x,y} \D_\om \fmap(\om, x(r))^\top p(r)$ and is uniformly bounded by a function of linear growth. As $d$-convergence implies $\Ww_2$-convergence and by the properties of $\Ww_2$-convergence (\cite[Thm.6.9]{villani2009optimal}) this implies:
    \begin{align*}
        \| \nabla L[\mu_n] \|^2_{L^2(\mu_n)} = \int_{[0,1] \times \Om} \| f_n \|^2 \d \mu_n \xrightarrow[n \to \infty]{} \int_{[0,1] \times \Om} \| f \|^2 \d \mu = \| \nabla L[\mu] \|^2_{L^2(\mu)}. 
    \end{align*}
\end{proof}

\begin{prop}[$\| \nabla L(\mu) \|_{L^2(\mu)}$ is an upper-gradient] \label{prop:upper_gradient}
    Assume $\fmap$ satisfies~\cref{fmap_assumption1,fmap_assumption2,fmap_assumption3}. Let $\mu \in \Pp_2^\Leb([0,1] \times \Om)$, then $\| \nabla L[\mu] \|_{L^2}(\mu)$ is the \emph{local slope} of the risk $L$ at $\mu$, that is:
    \begin{align} \label{local_slope}
        \| \nabla L[\mu] \|_{L^2(\mu)} = \limsup_{\nu \to \mu} \frac{\left( L(\mu) - L(\nu) \right)^+}{d(\mu, \nu)}.
    \end{align}
    Moreover, it is an upper-gradient in the sense of~\cref{def:upper_gradient}.
\end{prop}

\begin{proof}
    The last part of the result follows from~\cref{prop:characterization_absolute_continuity}, since if $(\mu_t)_{t \in I}$ is an absolutely continuous curve then it satisfies the continuity equation with a vector field $v$ such that $\| v_t \|_{L^2(\mu_t)} \leq \left| \frac{\d}{\d t} \mu_t \right|$ for a.e. $t \in I$. Hence by~\cref{cor:loss_differential} and Cauchy-Schwarz we have:
\begin{align*}
    \forall t_1 \leq t_2 \in I, \quad \left| L(\mu_{t_1}) - L(\mu_{t_2}) \right| \leq \int_{t_1}^{t_2} \| \nabla L[\mu_t] \|_{L^2(\mu_t)} \left| \frac{\d}{\d t} \mu_t \right| \d t.
\end{align*}

    Let us then show~\cref{local_slope}.
    Consider some parameter $\mu \in \Pp_2^\Leb([0,1] \times \Om)$ and denote by $| \nabla L |(\mu)$ the r.h.s. of~\cref{local_slope}.
    Then for $\eps > 0$, by continuity of $\| \nabla L[\mu] \|_{L^2(\mu)}$ (\cref{lem:gradient_continuous})
    and by definition of $| \nabla L |(\mu)$ 
    one can find a $\nu \in \Pp_2^\Leb([0,1] \times \Om)$ s.t.:
    \begin{align*}
        \frac{\left( L(\mu) - L(\nu) \right)^+}{d(\mu,\nu)}  \geq \left| \nabla L \right| (\mu) - \eps, \quad \text{and} \quad \left| \| \nabla L[\mu] \|_{L^2(\mu)} - \| \nabla L[\nu'] \|_{L^2(\nu')} \right| \leq \eps, \, \text{if $d(\mu,\nu') \leq d(\mu,\nu)$}.
    \end{align*}
    Consider $(\mu_t)_{t \in [0,1]}$ a constant speed geodesics with endpoints $\mu_0 = \mu$ and $\mu_1 = \nu$ (such a geodesic can easily be constructed by similarity with classical Wasserstein geodesics, see~\cite[Thm.7.2.2]{ambrosio2008gradient}). Then by~\cref{prop:characterization_absolute_continuity} the tangent velocity field $v$ of the curve $(\mu_t)$ satisfies for $\d t$-a.e. $t \in [0,1]$, $\| v_t \|_{L^2(\mu_t)} \leq \left| \frac{\d}{\d t} \mu_t \right| = d(\mu,\nu)$ and using~\cref{cor:loss_differential}:
    \begin{align*}
        L(\nu) & = L(\mu) + \int_0^1 \langle \nabla L[\mu_t], v_t \rangle_{L^2(\mu_t)} \d t \\
       & \leq L(\mu) + d(\mu,\nu) \int_0^1 \| \nabla L[\mu_t] \|_{L^2(\mu_t)} \d t \\
       & \leq L(\mu) + d(\mu,\nu) \| \nabla L[\mu] \|_{L^2(\mu)} + \eps.
    \end{align*}
    Similarly we have $L(\mu) \leq L(\nu) + d(\mu,\nu) \| \nabla L[\mu] \|_{L^2(\mu)} + \eps$ and hence $\left| \nabla L \right|(\mu) \leq \| \nabla L[\mu] \|_{L^2(\mu)} + \eps$.

    For the converse inequality consider for $t \in \RR$ the parameter $\mu_t = (\Id + t(0, \nabla L[\mu]))_\# \mu$. Then, by~\cref{prop:flow_differential} with $v = \nabla L[\mu]$, the map $t \mapsto L(\mu_t)$ is differentiable at $t = 0$ and applying the same calculations as in~\cref{cor:loss_differential}:
    \begin{align*}
        \frac{\d}{\d t} L[\mu_t] \Bigr|_{t=0} = \langle \nabla L[\mu], v \rangle_{L^2(\mu)} = \| \nabla L[\mu] \|^2_{L^2(\mu)}.
    \end{align*}
    Hence observing that $d(\mu_t,\mu) \leq t \| v \|_{L^2(\mu_t)}$ we have $\liminf_{t \to 0^+} \frac{(L(\mu_t)-L(\mu))^+}{d(\mu_t,\mu)} \geq \| \nabla L[\mu] \|_{L^2(\mu)}$.
\end{proof}

As a consequence of the previous result, we will from now on only consider as upper gradient of $L$ the one given for every $\mu \in \Pp_2^\Leb([0,1] \times \Om)$ by:
\begin{align} \label{upper_gradient}
    | \nabla L |(\mu) \eqdef \| \nabla L[\mu] \|_{L^2(\mu)} = \left( \int_{[0,1] \times \Om} \| \EE_{x,y} \D_\om \fmap (\om, x(s))^\top p(s) \|^2 \d \mu(s,\om) \right)^{1/2}.
\end{align}
Note that the vector field $\nabla L[\mu]$ was used in~\cref{def:gradient_flow} to define the notion of \emph{gradient flow} whereas the upper-gradient $|\nabla L|(\mu)$ is used in the~\cref{def:maximal_slope} of curves of maximal slope. The following theorem is the main result of this section and shows these two notions coincide.

\begin{thm} \label{thm:equivalence_GF_CMS}
    Assume $\fmap$ satisfies~\cref{fmap_assumption1,fmap_assumption2,fmap_assumption3} and $\ell$ is smooth. Let $I \subset \RR$ be an open interval. Then a curve $(\mu_t)_{t \in I}$ is a \emph{gradient flow} in the sense of~\cref{def:gradient_flow} if and only if it is a \emph{curve of maximal slope} for $L$ in the sense of~\cref{def:maximal_slope}.
\end{thm}

\begin{proof}
    \proofpart{Gradient flows are curves of maximal slope.}
    Let $(\mu_t)_{t \in I}$ be a gradient flow for $L$ in the sense of~\cref{def:gradient_flow}. Then $(\mu_t)$  is a locally absolutely continuous curve satisfying the continuity equation $\partial_t \mu_t + \div(v_t \mu_t) = 0$ with $v_t = - \nabla L[\mu_t]$. Hence by~\cref{prop:characterization_absolute_continuity} we have for a.e. $t \in I$:
    \begin{align*}
        \left| \frac{\d}{\d t} \mu_t \right| \leq \| v_t \|_{L^2(\mu_t)} = \| \nabla L[\mu_t] \|_{L^2(\mu_t)} .
    \end{align*}
    Also, by~\cref{cor:flow_differential}, $(L(\mu_t))_{t \in I}$ is absolutely continuous with for a.e. $t \in I$:
    \begin{align*}
        - \frac{\d}{\d t} L(\mu_t) = \langle v_t, \nabla L[\mu_t] \rangle_{L^2(\mu_t)} = \| \nabla L[\mu_t] \|^2_{L^2(\mu_t)} .
    \end{align*}
    Thus recalling that $| \nabla L |(\mu) = \| \nabla L[\mu] \|_{L^2(\mu)}$ we get~\cref{def:maximal_slope} by putting together the two previous equations.

    \proofpart{Curves of maximal slope are gradient flows.}
    Let $(\mu_t)_{t \in I}$ be a curve of maximal slope for $L$ in the sense of~\cref{def:maximal_slope}. Then in particular $(\mu_t)_{t \in I}$ is locally absolutely continuous in $(\Pp^\Leb_2([0, 1] \times \Om), d)$ and by~\cref{prop:characterization_absolute_continuity} there exists a Borel velocity field $v : I \times [0, 1] \times \Om \to \Om$ such that $\mu$ satisfies the continuity equation:
    \begin{align*}
        \partial_t \mu_t + \div((0, v_t) \mu_t) = 0 , \quad \text{over $I \times [0,1] \times \Om$,}
    \end{align*}
    and such that the metric derivative satisfies $|\frac{\d}{\d t} \mu_t| \geq \| v_t \|_{L^2(\mu_t)}$ for a.e. $t\in I$. Hence it follows from~\cref{cor:loss_differential} that $(L(\mu_t))_{t \in I}$ is absolutely continuous and for a.e. $t \in I$:
    \begin{align*}
        - \frac{\d }{\d t} L(\mu_t) = - \langle v_t, \nabla L[\mu_t] \rangle .
    \end{align*}
    Using the EDE condition we thus have:
    \begin{align*}
         - \langle v_t, \nabla L[\mu_t] \rangle \geq \frac{1}{2} (|\frac{\d}{\d t} \mu_t |^2 + | \nabla L |^2 (\mu_t) ) \geq \frac{1}{2} ( \| v_t\|^2_{L^2(\mu_t)} + \| \nabla L[\mu_t] \|^2_{L^2(\mu_t)} )
    \end{align*}
    from which it follows by Young's inequality that $v_t = - \nabla L[\mu_t]$ in $L^2(\mu_t)$ for a.e. $t \in I$.
\end{proof}

Note that, although it does not appear in~\cref{def:gradient_flow}, the above equivalence shows that if $(\mu_t)_{t \in I}$ is a gradient flow for $L$ then $\left| \frac{\d}{\d t} \mu_t \right| = \| \nabla L[\mu_t] \|_{L^2(\mu_t)}$ \emph{i.e.} $\nabla L[\mu_t]$ is in fact the (uniquely defined) tangent velocity field of the curve $(\mu_t)_{t \in I}$.

\subsection{Existence, uniqueness, and stability of gradient flow curves} \label{subsec:existence}

We show here the well-posedness result for the gradient flow equation of the risk $L$, namely we show the existence, uniqueness, and stability of gradient flow curves starting from any initialization $\mu_0 \in \Pp^\Leb_2([0, 1] \times \Om)$. For the ``existence'' part we will rely on classical results from the theory of gradient flows in metric spaces showing the convergence of proximal sequences to a curve known as \emph{(Generalized) Minimising Movements}~\cite{de1993new}.
For the ``uniqueness'' part we will show that gradient flow trajectories are stable, that is if two initializations $\mu_0, \mu_0'$ are close (in the sense of the metric $d$), then the emanating gradient flow curves $(\mu_t)_{t \geq 0}, (\mu_t')_{t \geq 0}$ stay close in finite time. 

\subsubsection{Existence}

We proceed to show the existence of gradient flow curves as defined in~\cref{def:gradient_flow}. For that purpose, we need a strengthening of~\cref{fmap_assumption1}. Notably, \cref{fmap_assumption_existence} allows showing the flow map $\mu \mapsto x_\mu$ is continuous for the topology of narrow convergence over $\Pp_2^\Leb([0,1] \times \Om)$.

\begin{assumptionc}{A}\label{fmap_assumption_existence}
    For some $\alpha \in [1, 2)$ we assume that:
    \begin{enumerate}
        \item The feature map $\fmap$ has $\alpha$-growth w.r.t. $\om$, locally w.r.t. $x$. That is for every compact $K \subset \RR^d$ there exists a constant $C = C(K)$ s.t:
        \begin{align*}
            \forall x \in K, \forall \om \in \Om, \quad \| \fmap(\om,x) \| \leq C(1+\|\om\|^\alpha) .
        \end{align*}

        \item  The feature map $\fmap$ is continuously differentiable and its differential $\D_x \fmap$ w.r.t. $x$ has $\alpha$-growth w.r.t. $\om$, locally w.r.t. $x$. That is for every compact $K \subset \RR^d$ there exists a constant $C = C(K)$ s.t.:
        \begin{align*}
            \forall x \in K, \forall \om \in \Om, \quad \| \D_x \fmap(\om,x) \| \leq C(1+\|\om \|^\alpha).
        \end{align*}
    \end{enumerate}
\end{assumptionc}

\begin{thm}[Existence of curves of maximal slope] \label{thm:existence_curve}
    Assume $\fmap$ satisfies~\cref{fmap_assumption1,fmap_assumption2,fmap_assumption3} and~\cref{fmap_assumption_existence}. Let $\mu_0 \in \Pp_2^\Leb([0, 1] \times \Om)$. Then there exists a curve of maximal slope $(\mu_t)_{t \in [0, +\infty)}$ starting from $\mu_0$ and $\left( \left| \frac{\d}{\d t} \mu_t \right| \right)_{t \geq 0} \in L^2_\loc([0,+\infty))$.
\end{thm}

\begin{proof}
    The result follows from the successive application of~\cite[Thm.2.2.3 and Thm.2.3.3]{ambrosio2008gradient}, the first result ensuring the existence of \emph{Generalized Minimizing Movements} and the second result stating that these curves are curves of maximal slope for the local slope. The proof proceeds by verifying the assumptions of these theorems. We consider here $\Pp_2^\Leb([0, 1] \times \Om)$ equipped with the topology induced by the distance $d$ and with the (weaker) topology of narrow convergence, denoted by $\tau$. Note that $(\Pp_2^\Leb([0,1] \times \Om), d)$ is a complete metric space (\cref{prop:d_complete}) and that the distance $d$ is $\tau$-lower-semicontinuous (\cref{lem:d_lsc}).

    \proofpart{$d$-bounded sets are $\tau$-relatively compact.}
        This property is verified as $d$-bounded sets are tight and hence $\tau$-relatively compact by Prokhorov's theorem.

    \proofpart{$L$ is $\tau$-continuous on $d$-bounded sets.}
    Let $(\mu_n)$ be a $d$-bounded sequence in $\Pp^\Leb_2([0, 1] \times \Om)$ s.t.\@ $\mu_n \xrightarrow{\tau} \mu$ for some $\mu \in \Pp^\Leb_2([0,1] \times \Om)$ and let us show that $L(\mu_n) \to L(\mu)$. Take $x \in \RR^d$ and denote by $x_n \eqdef x_{\mu_n}$ the flow trajectory starting from $x$ and associated to $\mu_n$. By Lebesgue's theorem, it suffices to show that $x_n(1) \to x_\mu(1)$. Using Ascoli's theorem, we will proceed by showing that $x_n \to x_\mu$ in $\Cc([0,1], \RR^d)$.
    
    From the $d$-boundedness and the proof of~\cref{prop:flow_wellposed}, it follows that the trajectories $x_n$ stay in a compact set. Moreover given $s_1 < s_2 \in [0, 1]$ we have using the $\alpha$-growth assumption:
    \begin{align*}
        \| x_n(s_2) - x_n(s_1) \| \leq \int_{[s_1,s_2] \times \Om} \| \fmap(\om,x_n(r)) \| \d \mu_n(r,\om) \leq \int_{[s_1,s_2] \times \Om} C(1+\| \om \|^\alpha) \d \mu_n(r,\om).
    \end{align*}
    Also as the sequence $(\mu_n)$ is $d$-bounded it has uniformly integrable $\alpha$-moments. Given $\eps > 0$ we can thus find a $k \geq 0$ such that, for every $n \geq 0$, $\int_{\| \om \| \geq k} C(1+\|\om\|^\alpha) \d \mu_n \leq \eps$. Using this in the previous inequality and the fact that the marginal of $\mu_n$ on $[0,1]$ is the Lebesgue measure gives:
    \begin{align*}
        \forall s_1 < s_2 \in [0,1], \quad \| x_n(s_2) - x_n(s_1) \| \leq \eps + C(1+k^\alpha) (s_2 - s_1).
    \end{align*}
    Thus the trajectories $(x_n)$ are equicontinuous and, by Arzela-Ascoli's theorem, we have (up to a subsequence) that $x_n \to \Bar{x}$ in $\Cc([0, 1], \RR^d)$.
    
    Let us show that $\Bar{x} = x_\mu$ is the flow generated by $\mu$. This will conclude this part of the proof as it will imply $x_n(1) \to x_\mu(1)$ and then $L(\mu_n) \to L(\mu)$ by Lebesgue's convergence theorem. Considering $s \in [0, 1]$ we have:
    \begin{align*}
        x_n(s) = x & + \int_{[0,1] \times \Om} \ind_{r \leq s} \fmap (\om, x_n(r)) \d \mu_n(r, \om) \\
        = x & + \int_{[0,1] \times \Om} \ind_{r \leq s} \fmap (\om, \Bar{x}(r)) \d \mu(r,\om) \\
        & + \int_{[0,1] \times \Om} \ind_{r \leq s} \left( \fmap(\om, x_n(r)) - \fmap (\om, \Bar{x}(s)) \right) \d \mu_n(r,\om) \tag{L1} \label{L1} \\
        & + \int_{[0,1] \times \Om} \ind_{r \leq s} \fmap(\om, \Bar{x}(r)) \d (\mu_n - \mu)(r,\om) \tag{L2} \label{L2} .
    \end{align*}
    In this last equality, we need to show that~\ref{L1} and~\ref{L2} vanish as $n \to \infty$. In~\ref{L1} we have that the integrand has $\alpha$-growth, hence given $\eps > 0$ we have using the uniform integrability of $\|\om\|^\alpha$ on the sequence $(\mu_n)$ that for every $n \geq 0$:
    \begin{align*}
         \| \int \ind_{r \leq s} \left( \fmap(\om, x_n(r)) - \fmap (\om, \Bar{x}(r)) \right) \d \mu_n(r,\om) \| \leq \eps + \int \ind_{r \leq s, \| \om \| \leq k} \| \fmap(\om, x_n(r)) - \fmap(\om, \Bar{x}(r)) \| \d \mu_n(r,\om).
    \end{align*}
    Then as $\fmap$ is locally-Lipschitz w.r.t. $x$ and $x_n \to \Bar{x}$ we have that the integrand on the r.h.s. converges uniformly to 0 and hence:
    \begin{align*}
        \limsup_{n \to \infty} \| \int_{[0,1] \times \Om} \ind_{r \leq s} \left( \fmap(\om, x_n(r)) - \fmap (\om, \Bar{x}(s)) \right) \d \mu_n(r,\om) \| \leq \eps.
    \end{align*}
    In~\ref{L2} the integrand is not continuous so we can't simply apply the definition of narrow convergence and need to leverage the fact that $(\mu_n)$ is a $d$-bounded sequence in $\Pp_2^\Leb([0, 1] \times \Om)$. Note that for every $(r,\om) \in [0, 1] \times \Om$, $\| \fmap(\om, \Bar{x}(r)) \| \leq C (1+ \|\om\|^\alpha)$. Given $\eps > 0$ and using the uniform integrability of $\|\om\|^\alpha$ we can thus have a $k \geq 0$ such that:
    \begin{align*}
        \sup_{n \geq 0} \int_{\|\om\| \geq k} C(1+\|\om\|^\alpha) \d \mu_n, \, \int_{\|\om\| \geq k} C(1+\|\om\|^\alpha) \d \mu \, \leq \eps.
    \end{align*}
    Then for $\delta > 0$ we can find a continuous function $\varphi : [0, 1] \times \Om \to \RR^d$ s.t. $\| \varphi(r,\om) \| \leq C(1+\|\om\|^\alpha)$, $\varphi(r,\om) = \fmap(\om,\Bar{x}(r))$ for every $r \leq s$ and $\varphi(r,\om) = 0$ whenever $r \geq s + \delta$. Considering such a function $\varphi$ we have for~\ref{L2}:
    \begin{align*}
        \| \int \ind_{r \leq s} \fmap(\om, \Bar{x}(r)) \d (\mu_n - \mu)(r,\om) \| & \leq \| \int \varphi \d (\mu_n-\mu) \| + \int \| \ind_{r \leq s} \fmap(\om, \Bar{x}(r)) - \varphi(r,\om) \| \d (\mu_n + \mu)(r,\om) \\
        & \leq  \| \int \varphi \d (\mu_n-\mu) \| + 4 \eps + C(1+k^\alpha) \delta
    \end{align*}
    where we used the fact that, in the second term, the integrand has $\alpha$-growth and is only non-zero for $r \in [s, s+\delta]$. Hence having chosen $\delta$ sufficiently small gives:
    \begin{align*}
        \limsup_{n \to \infty} \| \int \ind_{r \leq s} \fmap(\om, \Bar{x}(r)) \d (\mu_n - \mu)(r,\om) \| \leq \limsup_{n \to \infty} \| \int \varphi \d (\mu_n-\mu) \| + 5 \eps \leq 5 \eps,
    \end{align*}
    where the first $\limsup$ is 0 by definition of narrow convergence. We have thus shown that for every $s \in [0, 1]$, taking the limit as $n \to \infty$:
    \begin{align*}
        \Bar{x}(s) = x + \int_0^s \fmap(\om, \Bar{x}(r)) \d \mu(\om|r) \d r,
    \end{align*}
    \emph{i.e.} $\Bar{x} = x_\mu$ is the flow trajectory associated to $\mu$ and starting from $x$.
    
    \proofpart{$\| \nabla L(\mu) \|_{L^2(\mu)}$ is $\tau$-lower semi-continuous on $d$-bounded subsets.}
    We previously considered the upper-gradient $|\nabla L|$ defined in~\cref{upper_gradient} as $|\nabla L|(\mu) \eqdef \| \nabla L[\mu] \|_{L^2(\mu)}$. However, \textcite[Thm.2.3.3]{ambrosio2008gradient} state that the obtained curve of maximal slope is a curve of maximal slope for another definition of gradient, referred to as \emph{relaxed slope}. To show these two notions coincide here we show that the map $\mu \mapsto  \|\nabla L[\mu] \|_{L^2(\mu)}$ is $\tau$-lower semi-continuous on $d$-bounded subsets (cf.~\cite[Rm.2.3.4]{ambrosio2008gradient}).

    As before consider a $d$-bounded sequence $(\mu_n)$ in $\Pp_2^\Leb([0,1] \times \Om)$ which narrowly converges to some $\mu \in \Pp_2^\Leb([0,1] \times \Om)$. Then we previously have shown that for every $x \in \RR^d$ we have $x_{\mu_n} \to x_\mu$ in $\Cc([0,1], \RR^d)$. Proceeding with similar arguments one could show the same for the adjoint variable $p$, that is $p_n \eqdef p_{\mu_n,x,y} \to p_{\mu,x,y}$ in $\Cc([0,1], \RR^d)$. Then using a generalization of Fatou's lemma with varying measures (\emph{e.g.}~\cite[Thm.2.4]{feinberg2020fatou}) we have:
    \begin{align*}
        \liminf_{n \to \infty} \| \nabla L[\mu_n] \|^2_{L^2(\mu_n)} & = \liminf_{n \to \infty} \int_{[0,1] \times \Om} \| \EE_{x,y} \D_\om \fmap(\om,x_n(r))^\top p_n(r) \|^2 \d \mu_n(r,\om) \\
        & \geq \int_{[0, 1] \times \Om} \liminf_{\substack{n \to \infty \\ (r',\om') \to (r,\om)}} \| \EE_{x,y} \D_\om \fmap(\om',x_n(r'))^\top p_n(r') \|^2 \d \mu(r,\om) \\
        & = \| \nabla L[\mu] \|^2_{L^2(\mu)},
    \end{align*}
    which is the desired property.
\end{proof}

\subsubsection{Uniqueness}

We present here a uniqueness result for solutions of the gradient flow equation, which is the content of the following~\cref{thm:uniqueness_curve}. The proof is standard and relies on the Lipschitz continuity of the gradient vector field $\nabla L [\mu]$ w.r.t. the measure $\mu$. It uses the following~\cref{fmap_assumption_uniqueness} on the feature map $\fmap$ to ensure local Lipschitz continuity of the adjoint variable $p$ (\cref{lem:adjoint_lipschitz}). 

\begin{assumptionc}{B} \label{fmap_assumption_uniqueness}
    Assume that $\fmap$ is twice continuously differentiable with
    $\D^2_{\om, \om} \fmap$ uniformly bounded, $\D^2_{\om,x} \fmap$ having linear growth and $\D^2_{x,x} \fmap$ having quadratic growth w.r.t. $\om$. Namely, for every $R \geq 0$ there exists a constant $C = C(R)$ s.t. for every $x,x' \in B(0,R)$ and every $\om, \om' \in \Om$:
    \begin{align*}
        \| \D^2_{\om, \om} \fmap(\om,x) \| \leq C(R),  \quad \| \D^2_{\om,x} \fmap(\om,x) \| \leq C(R)(1+\|\om\|), \quad \| \D^2_{x, x} \fmap(\om,x) \| \leq C(R) (1+\|\om\|^2).
    \end{align*}  
\end{assumptionc}

\begin{thm}[Uniqueness of curves of maximal slope] \label{thm:uniqueness_curve}
    Assume $\fmap$ satisfies~\cref{fmap_assumption1,fmap_assumption2,fmap_assumption3,fmap_assumption_uniqueness} and that $\nabla_x \ell$ is locally Lipschitz w.r.t. $x$. 
    Let $\mu_0 \in \Pp^\Leb_2([0, 1] \times \Om)$. Then the gradient flow for the risk $L$ starting from $\mu_0$, if it exists, is unique.
\end{thm}

\begin{proof}
    Let $(\mu_t)_{t \geq 0}$ and $(\mu_t')_{t \geq 0}$ be two gradient flow curves for the risk $L$ starting from $\mu_0$. We will proceed to show that $d(\mu_t, \mu_t') = 0$ for every $t \geq 0$.
    
    We use the shorter notations $v_t \eqdef \nabla L[\mu_t]$, $v_t' \eqdef \nabla L[\mu_t']$ to refer to the \emph{tangent vector fields} of $\mu$ and $\mu'$ respectively. Observe that the map $t \mapsto d(\mu_t, \mu_t')^2$ is locally absolutely continuous and by~\cref{lem:differentiability_d} its differential is given at almost every $t \geq 0$ by:
    \begin{align*}
        \frac{\d}{\d t} d(\mu_t, \mu_t')^2 = 2 \int_0^1 \int_\Om \langle \om'-\om, v_t'(s,\om') - v_t(s, \om) \rangle \d \gamma_t(s,\om,\om'),
    \end{align*}
    where $\gamma_t \in \Pi^\Leb_o(\mu_t, \mu_t')$ can be any optimal coupling.

    Let $T \geq 0$ and denote by $\Ee \eqdef \sup_{t \in [0,T]} \max(\Ee_2(\mu_t), \Ee_2(\mu_t')) < \infty$.
    Fix some $t \in [0,T]$ and consider $(x,y)$ in the support of the data distribution $\Dd$ with the shortcuts $x_t \eqdef x_{\mu_t}$, $p_t \eqdef p_{\mu_t,x,y}$ and similarly $x_t', p_t'$ for $\mu_t'$. As the data distribution has compact support, we have that there exists some $R = R(\Ee)$ such that $\| x_t(s) \|, \| x_t'(s) \|, \| p_t(s) \|, \| p_t'(s) \| \leq R$. Then using~\cref{lem:flow_lipschitz,lem:adjoint_lipschitz} as well as the assumptions on $\fmap$ we have for every $(s, \om, \om') \in [0, 1] \times \Om^2$:
    \begin{multline} \label{gradient_lipschitz}
        \| \D_\om \fmap(\om, x_t(s))^\top p_t(s) - \D_\om \fmap(\om', x_t'(s))^\top p_t'(s) \| \\
        \leq  \| \D_\om \fmap(\om, x_t(s))^\top p_t(s) - \D_\om \fmap(\om', x_t(s))^\top p_t(s) \| \\
        \shoveright{
        + \| \D_\om \fmap (\om', x_t(s))^\top p_t(s) - \D_\om \fmap(\om', x_t'(s))^\top p_t'(s)  \|
        }
        \\
        \leq \| \D_\om \fmap(\om, x_t(s)) - \D_\om \fmap(\om', x_t(s)) \| \| p_t(s) \| \\
        \shoveright{
        + \| \D_\om \fmap(\om', x_t(s)) - \D_\om \fmap(\om', x_t'(s)) \| \| p_t(s) \|
        }
        \\
        \shoveright{
        + \| \D_\om \fmap(\om', x_t'(s)) \| \| p_t(s) - p_t'(s) \|
        }
        \\
        \leq C_1 \| \om - \om' \|
        + C_1 (1+\|\om'\|) d(\mu_t, \mu_t')
        + C_1 (1+\|\om'\|) d(\mu_t, \mu_t') ,
    \end{multline}
    with $C_1 = C_1(\Ee)$ some constant. Fixing some $\gamma_t \in \Pi^\Leb_o(\mu_t, \mu_t')$, using that $2 \langle a, b \rangle \leq \| a \|^2 + \| b\|^2$ and integrating the previous inequality over $(x,y)$ and $(s, \om, \om')$ we get:
    \begin{align*}
        \frac{\d}{\d t} d(\mu_t, \mu_t')^2 & \leq \int_{[0, 1] \times \Om^2} \left( \| \om-\om' \|^2 + \| v_t(s, \om) - v_t'(s, \om') \|^2 \right) \d \gamma_t(s, \om, \om') \leq C_2 d(\mu_t, \mu_t')^2,
    \end{align*}
    for some constant $C_2 = C_2(\Ee)$. We can then conclude using Grönwall's inequality to:
    \begin{align*}
        \forall t \in [0,T], \quad d(\mu_t, \mu_t')^2 \leq e^{C_2 t} d(\mu_0, \mu_0)^2 = 0 .
    \end{align*}
\end{proof}

The above proof relied on the following lemma, which shows that the adjoint variable map $\mu \mapsto p_{\mu, x, y}$ is locally Lipschitz under~\cref{fmap_assumption_uniqueness}.

\begin{lem} \label{lem:adjoint_lipschitz}
    Assume $\fmap$ satisfies~\cref{fmap_assumption1,fmap_assumption2,fmap_assumption3,fmap_assumption_uniqueness} and that $\nabla_x \ell$ is locally Lipschitz w.r.t. $x$. Then, for fixed $(x,y) \in \RR^d \times \RR^{d'}$, the adjoint variable map
    $$
    \mu \in \Pp^\Leb_2([0, 1] \times \Om) \mapsto p_{\mu,x,y} \in \Cc([0, 1], \RR^d)
    $$
    is locally Lipschitz.
    Namely, for every $\Ee \geq 0$ there exists a constant $C = C(\Ee)$ such that:
    \begin{align*}
        \sup_{s \in [0,1]} \| p_{\mu,x,y}(s) - p_{\mu',x,y}(s) \| \leq C  d(\mu, \mu')
    \end{align*}
    for every parameterization $\mu, \mu'$ such that $\Ee_2(\mu), \Ee_2(\mu') \leq \Ee$. Moreover, the constant $C$ can be chosen uniformly over $(x,y)$ in a compact subset.
\end{lem}

\begin{proof}
    Consider $(x,y) \in \RR^d \times \RR^{d'}$, $\Ee \geq 0$ and parameterizations $\mu, \mu'$ as in the proposition. We denote by $(x(s))$ and $(x'(s))$ the forward flows and $(p(s))$ and $(p'(s))$ the backward flows associated to $x, y$ and to $\mu$ and $\mu'$ respectively.
    Let $R \geq 0$ be such that $\| x \| + \| y \|\leq R$. Using~\cref{prop:flow_wellposed} and~\cref{backward} we can assume that the trajectories $x$, $x'$, $p$ and $p'$ are uniformly bounded by some $R' = R'(R, \Ee)$. Then we get from~\cref{backward} that at every $s \in [0, 1]$:
    \begin{multline*}
        \| p(s) - p'(s) \| \leq \| p(1) - p'(1) \| \\ + \int_s^1 \| \int_\Om \D_x \fmap(\om, x(r))^\top p(r) \d \mu(\om|r)  - \int_\Om \D_x \fmap(\om, x'(r))^\top  p'(r)  \d \mu'(\om|r) \| \d r .
    \end{multline*}
    Fixing some $r \in [s, 1]$ the integrand on the r.h.s. can be decomposed as:
    \begin{align*}
         & \| \int_\Om \D_x \fmap(\om, x(r))^\top p(r) \d \mu(\om|r)  - \int_\Om \D_x \fmap(\om, x'(r))^\top  p'(r)  \d \mu'(\om|r) \| \\
         \leq & \int_\Om \| \D_x \fmap(\om, x(r)) - \D_x \fmap(\om, x'(r)) \| \| p(r) \| \d \mu(\om|r)
        + \int_\Om \| \D_x \fmap(\om, x'(r)) \| \| p(r) - p'(r) \| \d \mu(\om|r) \\
        & + \| \int_\Om \D_x \fmap(\om, x'(r)) p'(r)  \d (\mu' - \mu)(\om|r) \| \\
        \eqdef &  I_1(r) + I_2(r) + I_3(r) .
    \end{align*}
    Then using the assumptions on $\fmap$ and~\cref{lem:flow_lipschitz}:
    \begin{align*}
        I_1(r) \leq C_1 d(\mu, \mu') \int_\Om (1+ \| \om \|^2) \d \mu(\om|r), \quad I_2(r) \leq C_2 \| p(r) - p'(r) \| \int_\Om (1+\| \om \|^2) \d \mu(\om|r) ,
    \end{align*}
    with $C_1 = C_1 (R,\Ee)$ and $C_2 = C_2(R, \Ee)$. For $I_3$, introducing an optimal coupling $\gamma \in \Gamma_o(\mu(.|r), \mu'(.|r))$:
    \begin{align*}
        I_3(r) & \leq \int_{\Om^2} \| \D_x \fmap(\om, x'(r)) - \D_x \fmap(\om', x'(r)) \| \| p'(r) \| \d \gamma(\om, \om') \\
        & \leq C_3 (1+\Ee_2(\mu(.|r))+\Ee_2(\mu'(.|r)))^{1/2} \Ww_2(\mu(.|r), \mu'(.|r)), 
    \end{align*}
    with $C_3 = C_3(R, \Ee)$. Assembling all the previous inequalities we get by Grönwall's lemma:
    \begin{align*}
        \| p(s) - p'(s) \|
        & \leq e^{C_2 (1 + \Ee_2(\mu))} \left( \| p(1) - p'(1) \| + d(\mu, \mu') (C_1 (1+ \Ee_2(\mu)) + C_3 (1+\Ee_2(\mu)+\Ee_2(\mu'))^{1/2} ) \right)
    \end{align*}
    To conclude it suffices to note that by definition $p(1) = \nabla_x \ell(x(1), y)$ and $p'(1) = \nabla_x \ell (x'(1), y)$ and using~\cref{lem:flow_lipschitz} with the assumptions on $\ell$:
    \begin{align*}
        \| p(1) - p'(1) \| \leq C_4 d(\mu, \mu'), \quad \text{where $C_4 = C_4(R, \Ee)$.}
    \end{align*}
\end{proof}

\subsubsection{Stability}

We now turn to a stability result on the gradient flow equation. The following~\cref{thm:stability_curve} is stronger than the above~\cref{thm:uniqueness_curve}. It implies that if a sequence of initializations $(\mu_0^n)_{n \geq 0}$ $d$-converges to some initialization $\mu_0$ then the associated gradient flows $(\mu^n_t)_{n \geq 0}$ $d$-converge to $\mu_t$, uniformly over finite time intervals.
For simplicity we consider here that $\ell$ is the square loss $\ell : (x,y) \mapsto \frac{1}{2} \| x-y \|^2$, but the result could be extended to any other loss satisfying $\| \nabla_x \ell \| \leq \varphi(\ell)$ for a concave increasing function $\varphi$.
We also consider the following supplementary assumptions, allowing us to control the growth of $\Ee_2(\mu_t)$ along the gradient flow (\cref{lem:energy_bound}).

\begin{assumptionc}{C} \label{fmap_assumption_stability}
    Assume that $\fmap$ is continuously differentiable and such that $\D_x \fmap$ is uniformly bounded and $\D_\om \fmap$ is of linear growth w.r.t. $\om$. Namely, there exists an absolute constant $C$ s.t.:
    \begin{align*}
        \forall x \in \RR^d, \om \in \Om, \quad \| \D_x \fmap(\om,x) \| \leq C, \quad \| \D_\om \fmap (\om,x) \| \leq C (1+\|\om\|).
    \end{align*}
\end{assumptionc}

\begin{thm}[Stability of curves of maximal slope] \label{thm:stability_curve}
    Assume $\fmap$ satisfies~\cref{fmap_assumption1,fmap_assumption2,fmap_assumption3,fmap_assumption_uniqueness,fmap_assumption_stability} and assume $\ell$ is the square loss. 
    Let $(\mu_t)_{t \geq 0}, (\mu_t')_{t \in t \geq 0}$ be gradient flow curves for the risk $L$ starting from $\mu_0, \mu_0' \in \Pp^\Leb_2([0, 1] \times \Om)$ respectively and let $\Ee_0$ be such that $\Ee_2(\mu_0), \Ee_2(\mu_0') \leq \Ee_0$. Then for every $T \geq 0$ there exists a constant $C = C(\Ee_0, T)$ such that:
    \begin{align*}
        \forall t \in [0, T], \quad d(\mu_t, \mu_t') \leq e^{Ct} d(\mu_0, \mu_0') .
    \end{align*}
\end{thm}

\begin{proof}
    Let $T \geq 0$. By the energy bound of~\cref{lem:energy_bound} below, we know that we can find a $\Ee = \Ee(\Ee_0,T)$ such that for every $t \in [0, T]$:
    \begin{align*}
        \Ee_2(\mu_t), \Ee_2(\mu_t') \leq \Ee.
    \end{align*}
    Then using~\cref{fmap_assumption_uniqueness} and proceeding as in the proof of the above~\cref{thm:uniqueness_curve} (cf.~\cref{gradient_lipschitz})  we get a constant $C = C(\Ee)$ such that for $\d t$-a.e. $t \in [0,T]$:
    \begin{align*}
        \dd{t} d(\mu_t,\mu_t')^2 \leq C d(\mu_t, \mu_t')^2 ,
    \end{align*}
    which gives the result using Grönwall's inequality.
\end{proof}

In the above proof, we used the following technical result giving an upper bound on the energy $\Ee_2(\mu_t)$ along a gradient flow curve $(\mu_t)_{t \geq 0}$.

\begin{lem} \label{lem:energy_bound}
    Assume $\fmap$ satisfies~\cref{fmap_assumption1,fmap_assumption2,fmap_assumption3,fmap_assumption_uniqueness,fmap_assumption_stability} and $\ell$ is the square loss.
    Let $(\mu_t)_{t \geq 0}$ be a gradient flow for the risk $L$ and let $\Ee \geq 0$ be s.t. $\Ee_2(\mu_0) \leq \Ee$. Then there exists a constant $C = C(\Ee)$ such that $\Ee_2(\mu_t) \leq e^{Ct}(\Ee_2(\mu_0) + Ct)$ for every $t \geq 0$.
\end{lem}

\begin{proof}
    For $(x,y) \in \RR^d \times \RR^{d'}$ use the shortcuts $x_t \eqdef x_{\mu_t}$, $p_t \eqdef p_{\mu_t,x,y}$.
    Note that the map $t \mapsto \Ee_2(\mu_t) = d(\mu_t, \Leb([0,1]) \otimes \delta_0)^2$ is locally absolutely continuous and that by~\cref{lem:differentiability_d} its derivative is given at $\d t$-a.e. $t \geq 0$ by:
    \begin{align*}
        \dd{t} \Ee_2(\mu_t) = 2 \int_{[0,1] \times \Om} \langle \om, \EE_{x,y} \D_\om \fmap(\om, x_t(s))^\top p_t(s) \rangle \d s.
    \end{align*}
    By~\cref{fmap_assumption_stability} there exists an absolute constant $C_1$ such that $\| \D_x F_{\mu_t(.|s)} \| \leq C_1$ and hence $\| p_t(s) \| \leq e^{C_1} \| p_t(1) \|$ for every $s \in [0,1]$. Using the initial condition on $p_t(1)$ and the fact that $\ell$ is the quadratic loss:
    \begin{align*}
        \EE_{x,y} \| p_t(s) \| \leq e^{C_1} \EE_{x,y} \| x_t(1) - y \| \leq C_2 \sqrt{L(\mu_t)} \leq C_2 \sqrt{L(\mu_0)},
    \end{align*}
    for some universal constant $C_2$. Then using~\cref{fmap_assumption_stability} we have $\| \D_\om \fmap (\om, x) \| \leq C_1(1+\|\om\|)$, and using previous inequality we get:
    \begin{align*}
        \dd{t} \Ee_2(\mu_t) \leq 2 C_1 C_2 \left( \Ee_2(\mu_t) + \sqrt{\Ee_2(\mu_t)} \right) \sqrt{L(\mu_0)} .
    \end{align*}
    Observing that $L(\mu_0) \leq C_3$ for some constant $C_3 = C_3(\Ee)$, the result follows by Grönwall's inequality
\end{proof}

\section{Convergence analysis and (local) Polyak-\L{}ojasiewicz property} \label{sec:convergence}

This section is devoted to the analysis of the convergence of the gradient flow dynamic we defined in the previous section. The questions we will try to answer are the following:
\begin{center}
    \emph{
    Given an initial parametrization $\mu_0 \in \Pp^\Leb_2([0, 1] \times \Om)$, does the risk $L(\mu_t)$ converge to $0$ as $t \to \infty$ ? Does the parameterization converge to an optimal parameterization $\mu_\infty$?
    }
\end{center}

\subsection{Polyak-\L{}ojasiewicz inequality}

Our approach to show convergence of the gradient flow is to show that the risk satisfies the following local Polyak-\L{}ojasiewicz (P-\L{}) property around well-chosen parameterizations. The P-\L{} inequality provides a lower bound on the ratio between the square gradient's risk $\left|  \nabla L \right|^2$ and the risk $L$. It thus prevents the existence of spurious critical points and guarantees that the risk decreases at a constant rate along gradient flow. 
\begin{defn}[Local P-\L{} property]
    We say that the risk $L$ is $(R, m)$-P-\L{} around a parameterization $\mu \in \Pp^\Leb_2([0,1]\times \Om)$ if it satisfies:
    \begin{align} \label{PL_inequality}
        \forall \mu' \in B(\mu, R), \quad | \nabla L |^2(\mu') \geq m L(\mu').
    \end{align}
\end{defn}
\begin{rem}
    Different formulations of the $(R,m)$-P-\L{} property have been introduced in the literature. For example~\cite{chatterjee2022convergence} introduced the local ratio:
    $
        \alpha(\mu_0,R) \eqdef \inf_{\mu' \in B(\mu_0,R), L(\mu') > 0} \frac{\| \nabla L(\mu') \|^2}{L(\mu')} .
    $
\end{rem}
A direct consequence of the $(R, m)$-P-\L{} property~\cref{PL_inequality} is that the risk admits no spurious critical points around $\mu$ as all critical points are global minimizers. A second consequence is a local convergence result: if the loss at initialization is sufficiently small then both the risk and the parameterization converge, respectively to $0$ and to some global minimizer $\mu_*$ of the risk.

\begin{thm}[{{\cite[Cor.1.5]{schiavo2023local}}}] \label{thm:local_convergence}
    Assume that $L$ is $(R, m)$-P-\L{} around an initial parameterization $\mu_0 \in \Pp^\Leb_2([0, 1] \times \Om)$. If the initial risk satisfies:
    \begin{align} \label{general_convergence_condition}
        L(\mu_0) < \frac{m R^2}{4}
    \end{align}
    then, for some $T \in (0, +\infty]$, if $(\mu_t)_{t \in [0, T)}$ is a curve of maximal slope for $L$ we have
    \begin{align*}
        L(\mu_t) \leq e^{-mt}L(\mu_0), \quad \forall t \in [0, T).
    \end{align*}
    Moreover $\mu_T \eqdef \lim_{t \to T} \mu_t$ exists and satisfies $d(\mu_T, \mu_0) \leq R$.
\end{thm}

The above convergence result is open, meaning that if its assumptions are satisfied for some initialization $\mu_0$ then it is also the case for any initialization $\mu_0'$ sufficiently close to $\mu_0$.
This is the content of the following~\cref{prop:PL_open}. 

\begin{prop} \label{prop:PL_open}
    Assume that $L : \Pp^\Leb_2([0,1]\times \Om) \to \RR_+$ satisfies the assumptions of~\cref{thm:local_convergence} around some $\mu_0$.
    Then there exists a neighborhood $\Uu$ of $\mu_0$ such that any initialization $\mu_0' \in \Uu$ also satisfies these assumptions.
\end{prop}

\begin{proof}
    Note that, by definition of the $(R,m)$-P-\L{} property, if it is satisfied around $\mu_0$ then, for any $\delta \in (0,R)$ a $(R-\delta,m)$-P-\L{} property is satisfied around $\mu_0'$ for any $\mu_0' \in B(\mu_0, \delta)$. Moreover the condition~\cref{general_convergence_condition} is open: if it is satisfied at $\mu_0$ with $R$ and $m$ then it is satisfied at any $\mu_0' \in B(\mu_0, \delta)$ with $R-\delta$ and $m$ provided $\delta$ is sufficiently small.
\end{proof}

\subsection{Convergence for general architecture}

We explain here how one can prove convergence of the gradient flow to a minimizer of the risk $L$ when training the \NODE{} model. More precisely, we give sufficient conditions for $L$ to satisfy a $(R, m)$-P-\L{} property around the gradient flow initialization, and as a consequence of~\cref{thm:local_convergence}, these ensure convergence to a global minimizer of the risk. Those conditions are first described at a general level and later specified for practical examples of architectures and initializations.

Specifically, we will consider the case where the data distribution is the empirical distribution $\Dd = \frac{1}{N} \sum_{i=1}^N \delta_{(x^i, y^i)}$ for a dataset $\lbrace (x^i, y^i) \rbrace_{1 \leq i \leq N} \in (\RR^d \times \RR^d)^N$. In this case, the risk $L$ for a parameterization $\mu$ is given by the \emph{empirical risk}:
\begin{align*}
    L(\mu) = \frac{1}{2N} \sum_{i=1}^N \ell(x^i_\mu(1), y^i),
\end{align*}
where $x^i_\mu$ is the flow of~\cref{forward} starting at $x^i$ and with parameterization $\mu$, which we will simply denote by $x^i$ when no ambiguity. Similarly we denote by $p^i_\mu \eqdef p_{\mu, x^i, y^i}$, or simply $p^i$ when no ambiguity, the associated adjoint variable of~\cref{backward}. We also suppose that $\fmap$ satisfies at least~\cref{fmap_assumption1,fmap_assumption2,fmap_assumption3}.
In this setting, using~\cref{upper_gradient} we get for any $\mu \in \Pp^\Leb_2([0,1] \times \Om)$:
\begin{align*}
    | \nabla L |^2 (\mu) & = \int_0^1 \int_\Om \| \frac{1}{N} \sum_{i = 1}^N \D_\om \fmap(\om, x^i (s))^\top p^i(s) \|^2 \d \mu(s,\om) \\
    & = \frac{1}{N^2} \int_0^1  \left( \sum_{1 \leq i,j \leq N} p^i(s)^\top K[\mu(.|s)](x^i(s),x^j(s)) p^j(s) \right) \d s ,
\end{align*}
where for $\nu \in \Pp_2(\Om)$ we define the kernel $K[\nu] : \RR^d \times \RR^d \to \RR^{d \times d}$ as:
\begin{align} \label{K}
    \forall x,x' \in \RR^d, \quad K[\nu](x,x') \eqdef \int_\om \D_\om \fmap(\om,x) \D_\om \fmap(\om, x')^\top \d \nu(\om) .
\end{align}
For $\nu \in \Pp_2(\Om)$ and a point cloud $\zz = (z^i)_{1 \leq i \leq N} \in (\RR^d)^N$ we will denote by $\KK[\nu,\zz] \in \RR^{dN \times dN}$ the \emph{kernel matrix} associated to $K[\nu]$ and defined as the block matrix:
\begin{align} \label{K_matrix}
    \KK[\nu,\zz] \eqdef \left( K[\nu](z^i, z^j) \right)_{1 \leq i,j \leq N} .
\end{align}
In particular, we see that the conditioning $\lambda_{\min} \left( \KK[\mu, \zz] \right)$ of the kernel matrix will play an important role in proving a local P-\L{} property for the risk $L$. Indeed, in terms of the kernel matrix $\KK$, the square gradient can be written :
\begin{align} \label{square_gradient}
    | \nabla L |^2(\mu) = \frac{1}{N^2} \int_0^1 \langle \pp_\mu(s), \KK[\mu(.|s),\xx_\mu(s)] \pp_\mu(s) \rangle \d s,
\end{align}
where for every $s \in [0,1]$ we defined the point cloud $\xx_\mu(s) \eqdef (x^i_\mu(s)) \in (\RR^d)^N$ and where we concatenated the adjoint variables into $\pp_\mu(s) \eqdef (p^i_\mu(s))_{1 \leq i \leq N} \in \RR^{dN}$. In the following, when no ambiguity we will simply denote by $\xx = \xx_\mu \in \Cc([0,1], (\RR^d)^N)$ and $\pp = \pp_\mu \in \Cc([0,1], \RR^{dN})$.

\begin{rem}[Kernels, RKHS and feature representation] \label{rem:RKHS_feature_representation}
    Due to its definition, the above function $K[\nu]$ is obviously a (vector valued) \emph{positive kernel} over $\RR^d$, meaning that:
    \begin{align*}
        \forall (x^i)_{1 \leq i \leq N}, (\alpha^i)_{1 \leq i \leq N} \in (\RR^d)^N, \quad \sum_{1 \leq i,j \leq N} \langle \alpha^i,  K[\nu](x^i,x^i)\alpha^j \rangle \geq 0.
    \end{align*}
    It is a classical result that every such kernel defines a unique structure of \emph{Reproducing Kernel Hilbert Space (RKHS)} over $\RR^d$, a Hilbert space of function for which the evaluation function is continuous~\parencite{carmeli2010vector}.
    
    Note that the kernel $K[\nu]$ is here directly given by a \emph{feature representation}, that is a representation of the form $K[\nu](x,y) = \chi(x)^* \chi(y)$ with a map $\chi : \RR^d \to \Ll(\RR^d, \Hh)$ for some Hilbert space $\Hh$. Here one can for example consider $\Hh = L^2(\nu)$ and $\chi : x \mapsto \D_\om \fmap(.,x)^\top$. If such a representation always defines a positive kernel, one can conversely show that such a representation always exists whenever $K$ is a positive kernel~\parencite{carmeli2010vector}. This representation can however not be expected to be unique and corresponds with a certain square root of $K[\nu]$ viewed as an integral operator~\parencite{bach2017equivalence}.
\end{rem}

To obtain a P-\L{} inequality for the risk $L$ we will from now on assume that the loss function $\ell$ itself is P-\L{} w.r.t. the $x$ variable.
\begin{assumption} \label{assumption:loss_PL}
    We assume that the loss $\ell$ satisfies a P-\L{} inequality w.r.t. $x$ that is:
    \begin{align*}
        \forall (x,y) \in \RR^d \times \RR^{d'}, \quad \| \nabla_x \ell(x, y) \|^2 \geq 2 \ell(x,y).
    \end{align*}
    This assumption is in particular satisfied by the quadratic loss $\ell(x,y) = \frac{1}{2} \| x-y \|^2$  in regression problems or locally by the cross entropy loss $\ell(x,y) = \log \left( \frac{\sum_j y[j] \exp(x[j])}{\sum_j \exp(x[j])} \right)$ in classification.
\end{assumption}

\begin{lem} \label{lem:conditioning_PL}
    Assume $\fmap$ satisfies~\cref{fmap_assumption1,fmap_assumption2,fmap_assumption3} and $\ell$ satisfies~\cref{assumption:loss_PL}. Consider $\mu \in \Pp_2^\Leb([0,1] \times \Om)$. Then there exists a constant $C = C(\Ee_2(\mu))$ s.t.:
    \begin{align} \label{conditioning_PL}
        | \nabla L |^2(\mu) \geq \frac{2e^{-C}}{N} \left( \int_0^1 \lambda_{\min}(\KK[\mu(.|s),\xx_\mu(s)]) \d s \right) L(\mu) ,
\end{align}
\end{lem}

\begin{proof}
Thanks to~\cref{fmap_assumption3} and to the definition of $p^i_\mu$ there exists a constant $C = C(\Ee_2(\mu))$ such that for every $1 \leq i \leq N$ we have the estimate:
\begin{align*}
    \| p^i_\mu(s) \|^2 \geq e^{-C} \| p^i_\mu(1) \|^2 , \quad \forall s \in [0, 1]. 
\end{align*}
Using that $p^i_\mu(1) = \nabla_x \ell(x^i_\mu(1), y^i)$ and with the previous~\cref{assumption:loss_PL} we have
$$
\| \pp_\mu(s) \|^2 \geq e^{-C} \| \pp_\mu(1) \|^2 \geq 2N e^{-C} L(\mu) \, .
$$
Putting this lower bound in~\cref{square_gradient} then gives the result.
\end{proof}

The above~\cref{lem:conditioning_PL} shows that the conditioning $\lambda_{\min}(\KK)$ of the kernel matrix provides a lower bound on the ratio between the square gradient and the risk: having $\lambda_{\min}(\KK)$ strictly positive directly implies the P-\L{} inequality~\cref{PL_inequality} for the risk. The above quantity could for example be computed numerically during training. However, at this point, it is not clear how one can be sure, before training, that the P-\L{} inequality will hold along the gradient flow. We investigate this problem in the next section for a special kind of architecture.
Nonetheless, a direct corollary of~\cref{lem:conditioning_PL} is that gradient flow converges if it stays bounded and we assume the kernel matrix stays well-conditioned.
\begin{cor}
    Assume $\fmap$ satisfies~\cref{fmap_assumption1,fmap_assumption2,fmap_assumption3} and $\ell$ satisfies~\cref{assumption:loss_PL}. For an initialization $\mu_0 \in \Pp_2^\Leb([0,1] \times \Om)$, let $(\mu_t)_{t \geq 0}$ be a gradient flow of $L$ starting from $\mu_0$. If there exists a constant $C > 0$ s.t., for every $t \geq 0$, $\Ee_2(\mu_t) \leq C$ and $\int_0^1 \lambda_{\min}(\KK[\mu_t(.|s),\xx_{\mu_t}(s)]) \d s \geq C^{-1}$, 
    then the gradient flow converges in the sense that $\mu_t \xrightarrow{t \to +\infty} \mu_\infty \in \Pp_2^\Leb([0,1] \times \Om)$ and there exists a constant $C' > 0$ s.t. for every $t \geq 0$:
    \begin{align*}
        L(\mu_t) \leq e^{-C' t} L(\mu_0).
    \end{align*}
\end{cor}

\subsection{Convergence for SHL residuals}

We study here in more detail the case of \emph{Single Hidden Layer (SHL)} perceptrons for which we show the kernel matrix is well-conditioned and the risk satisfies a P-\L{} property around well-chosen initializations. The SHL architecture is an instance of~\cref{measure_parameterization} which we define by setting the parameter space $\Om = \RR^d \times \RR^d \times \RR$ and the map:
\begin{align} \label{fmap_SHL}
    \fmap : ((u,w,b), x) \in \Om \times \RR^d \mapsto u \activation (w^\top x + b) , 
\end{align}
where $\activation : \RR \to \RR$ is an activation function. We will generically make the following assumption on the activation $\activation$ to ensure that our previous assumptions on the feature map $\fmap$ are satisfied.

\begin{assumption} \label{activation_assumption}
    The activation $\activation : \RR \to \RR$ is a twice continuously differentiable function with a uniformly bounded derivative.
    Defining $M = M(\activation) \eqdef |\activation(0)| + \| \activation' \|_\infty$ we have for $(x,\om) \in \RR^d \times \Om$:
    \begin{align} \label{activation_assumption_growth}
        \| \fmap(\om,x) \| \leq M (1+\|x\|)(1+\|\om\|^2), \quad \| \D_\om \fmap(\om,x) \| \leq M(1+\|x\|)(1+\|\om\|), \quad \| \D_x \fmap(\om,x) \| \leq M\|\om\|^2.
    \end{align}
    Note that this assumption ensures that~\cref{fmap_assumption1,fmap_assumption2,fmap_assumption3} are satisfied. It does not however imply~\cref{fmap_assumption_existence,fmap_assumption_uniqueness}. Still we are able to show that~\cref{thm:existence_curve,thm:uniqueness_curve} both hold for SHL architectures (cf.~\cref{prop:existence_SHL,prop:uniqueness_SHL} in~\cref{sec:appendix}).
\end{assumption}

\begin{rem}
    \Cref{activation_assumption} is in particular satisfied for the popular choices that are $\activation = \tanh$ or any smooth approximation of $\relu$ such as $\gelu$ or $\swish$, but considering $\relu$ activation itself is expected to create two kinds of issues.
    First, the non-differentiability of $\relu$ at $0$ could create singularities in the continuity equation.
    As a consequence, while existence of solutions to the gradient flow equation (\cref{def:gradient_flow}) might still hold, one should not expect those solutions to be unique (\cref{thm:uniqueness_curve}).
    Then, and perhaps most importantly, those solutions might not coincide with curves of maximal slope.
    Indeed, a cornerstone of our analysis is~\cref{thm:equivalence_GF_CMS}, identifying gradient flow curves (\cref{def:gradient_flow}) with curves of maximal slopes for the risk (\cref{def:maximal_slope}).
    This result requires minimal regularity on $\fmap$ and allows showing existence and uniqueness of gradient flow curves in~\cref{subsec:existence} and convergence in~\cref{sec:convergence}.
\end{rem}

Recall the definition of the kernel $K$ in~\cref{K}. In the case of a SHL, the associated kernel can be decomposed into two parts. For $\mu \in \Pp_2(\Om)$ we have $K[\mu] = k^1[\mu] \Id + K^2[\mu]$ where we define for every $x,y \in \RR^d$:
\begin{align} \label{k1}
    k^1[\mu](x,y) & \eqdef  \int_{\RR^d \times \RR^d \times \RR} \activation (w^\top x+b) \activation (w^\top y+b) \d \mu(u,w,b) \\
    K^2[\mu](x,y) & \eqdef \int_{\RR^d \times \RR^d \times \RR} \activation'(w^\top x + b) \activation'(w^\top y +b) (x^\top y+1) (u \otimes u) \d \mu(u,w,b) . \notag
\end{align}

\begin{rem}[Feature distribution and representation of $F_\mu$ in $\Hh_{k^1[\mu]}$]
    Similarly to $K[\mu]$ in~\cref{K} the (scalar) kernel $k^1[\mu]$ is defined in~\cref{k1} directly through a feature representation (cf.~\cref{rem:RKHS_feature_representation}). Note that $k^1[\mu]$ here only depends the marginal of $\mu$ w.r.t. the variable $(w,b)$ which we denote by $\mu^2$ and call the \emph{feature distribution}. We will then equivalently write $k^1[\mu]$ or $k^1[\mu^2]$ to designate the same kernel.

    Due to the linearity of $\fmap$ w.r.t. the outer weights $u$ one may identify $u$ with its conditional expectation $u(w,b) \eqdef \EE_\mu[u|w,b] \in L^2(\mu^2)$. Using this identification in the definition of $F_\mu$ gives
    \begin{align*}
        \forall x \in \RR^d, \quad F_\mu(x) = \int_\Om u \activation(w^\top x + b) \d \mu(u,w,b) = \int_{\RR^d \times \RR} u(w,b)  \activation(w^\top x +b) \d \mu^2(w,b), 
    \end{align*}
    i.e. $F_\mu$ is a function in the RKHS $\Hh_{k^1[\mu]}$ associated with the kernel $k^1[\mu]$. This representation bridges the gap with the work of~\textcite{barboni2022global} who consider \NODEs{} parameterized by vector fields in a (fixed) RKHS,
    which amounts to considering SHL residuals of the form~\cref{fmap_SHL} but where only the outer weights $u$ are trained and the inner weights $w$ are fixed.
    Hence an important improvement of our work is that both inner and outer weights are trained, such as is done in practice. 
    For a parameterization $\mu \in \Pp^\Leb_2([0, 1] \times \Om)$ of the \NODE, the feature distribution $\mu^2(.|s)$ (and hence the function space $\Hh_{k^1[\mu(.|s)]}$) \emph{(i)} is not constant along the flow time $s \in [0, 1]$ and \emph{(ii)} may also evolve along the gradient flow time $t \geq 0$.
\end{rem}

Observing that both $k^1[\mu]$ and $K^2[\mu]$ define positive kernels on $\RR^d$, we have $K[\mu] \geq k^1[\mu] \Id$ in the sense of positive kernels.
Therefore, $\lambda_{\min}(\KK^1[\mu(.|s), \xx_\mu(s)])$ provides a natural lower bound for $\lambda_{\min}(\KK[\mu(.|s), \xx_\mu(s)])$
where, similarly to~\cref{K_matrix} the kernel matrix $\KK^1[\mu, \zz] \in \RR^{N \times N}$ is defined for a point cloud $\zz = (z^i)_{1 \leq i \leq N}$ and for $\mu \in \Pp_2(\Om)$ as:
\begin{align*}
    \KK^1[\mu, \zz] \eqdef \left( k^1[\mu](z^i, z^j) \right)_{1 \leq i,j \leq N} \in \RR^{N \times N}.
\end{align*}
As the following result shows, this lower bound is well behaved w.r.t.\@ the metric $d$ on the parameter set $\Pp_2^\Leb([0,1] \times \Om)$.

\begin{prop} \label{prop:conditioning_lipschitz}
    Assume $\activation$ satisfies~\cref{activation_assumption}. Then the map $\mu \mapsto \int_0^1 \lambda_{\min}(\KK^1[\mu(.|s),\xx_\mu(s)]) \d s$ is locally-Lipschitz on $(\Pp^\Leb_2([0,1]\times \Om), d)$. Moreover there exists some constant $C$ such that if $\mu, \mu'$ are such that $\Ee_2(\mu), \Ee_2(\mu') \leq \Ee$ then :
    \begin{align*}
        \left| \int_0^1 \lambda_{\min}(\KK^1[\mu,\xx_\mu]) -  \int_0^1 \lambda_{\min}(\KK^1[\mu',\xx_{\mu'}]) \right| \leq N C e^{C \Ee_2(\mu_0)} d(\mu, \mu').
    \end{align*}
\end{prop}

\begin{proof}
    Denote by $M$ the constant in~\cref{activation_assumption_growth} and let $R \geq 0$ be such that $\Supp(\Dd) \subset B(0,R)$. We have by~\cref{prop:flow_wellposed} that for $\mu \in \Pp^\Leb_2([0, 1] \times \Om)$ and for $x \in \Supp(\Dd_x)$ the flow verifies:
    \begin{align*}
        \forall s \in [0, 1], \quad \| x_\mu(s) \| \leq e^{M (1 + \Ee_2(\mu))}(R + M(1+ \Ee_2(\mu))) \leq C_1 e^{C_1 \Ee_2(\mu)} ,
    \end{align*}
    where $C_1 = C_1(R, M)$. Using the previous bound on the trajectories as well as the bounds in~\cref{activation_assumption_growth} we see following the proof of~\cref{lem:flow_lipschitz} that if $\Ee_2(\mu), \Ee_2(\mu') \leq \Ee$ then:
    \begin{align*}
        \forall s \in [0, 1], \quad \| x_\mu(s) - x_{\mu'}(s) \| \leq e^{M \Ee} (1+C_1 e^{C_1 \Ee} ) \sqrt{2+4\Ee} d(\mu, \mu') \leq C_2 e^{C_2 \Ee} d(\mu, \mu'),
    \end{align*}
    where $C_2 = C_2(R, M)$. Also, it follows from the assumptions on $\activation$ that, for fixed $\mu \in \Pp_2(\Om)$, the map $(x, y) \in \RR^{2d} \mapsto k^1[\mu](x,y)$
    is locally Lipschitz and
    \begin{align*}
        \forall x,x',y,y' \in \RR^d, \quad \left| k^1[\mu](x,y) - k^1[\mu](x',y') \right| \leq M^2 \Ee_2(\mu) (1+\|x'\|+\|y\|) (\|x-x'\| + \|y-y'\|).
    \end{align*}
    For fixed $x,y \in \RR^d$, the map $\mu \in \Pp_2(\Om) \mapsto k^1[\mu](x,y)$
    is also locally Lipschitz and using~\cref{activation_assumption} we have that if $\Ee_2(\mu), \Ee_2(\mu') \leq \Ee$ then for some constant $C_4$:
    \begin{align*}
        \forall x,y \in \RR^d, \quad \left| k^1[\mu](x,y) - k^1[\mu'](x,y) \right| \leq C_4 (1 + \|x\| + \|y\|) (1+ \sqrt{\Ee}) \Ww_2(\mu, \mu').
    \end{align*}
    Compiling the previous inequalities we have that if $\Ee$ is such that $\Ee_2(\mu), \Ee_2(\mu') \leq \Ee$ then:
    \begin{align*}
        \| \KK^1[\mu, \xx_\mu] - \KK^1[\mu', \xx_{\mu'}] \|_\infty \leq C_5 e^{C_5 \Ee} d(\mu, \mu')
    \end{align*}
    where $C_5 = C_5(R, M)$ and $\| . \|_\infty$ is the supremum norm on matrices.
    Finally, the result follows from the $N$-Lipschitz continuity of the map $S \mapsto \lambda_{\min}(S)$ on the space of $N \times N$ symmetric matrices equipped with $\|.\|_\infty$.
\end{proof}

The following result gives sufficient conditions for the convergence of the gradient flow to a global minimizer of the risk in the case of a \NODE{} with SHL residuals.

\begin{thm} \label{thm:convergence_SHL}
    Assume $\fmap$ is of the form~\cref{fmap_SHL} with an activation $\activation$ satisfying~\cref{activation_assumption} and that $\ell$ satisfies~\cref{assumption:loss_PL}. Then for any $\mu_0 \in \Pp_2^\Leb([0,1] \times \Om)$ there exists a positive constant $C = C(\Ee_2(\mu_0))$ s.t. if
    \begin{align} \label{convergence_condition}
    \lambda_0 \eqdef \int_0^1 \lambda_{\min}(\KK^1[\mu_0, \xx_{\mu_0}]) > 0 \quad \text{and} \quad L(\mu_0) < C N^{-3} \lambda_0^3  ,  
    \end{align}
    then any gradient flow $(\mu_t)_{t \geq 0}$ starting from $\mu_0$ satisfies:
    \begin{align*}
        L(\mu_t) \leq L(\mu_0) \exp(- \frac{C \lambda_0}{N} t), \quad \text{and} \quad \mu_t \xrightarrow{t \to \infty} \mu_\infty \in \Pp_2^\Leb([0, 1] \times \Om).
    \end{align*}
\end{thm}

\begin{proof}
    Let $C_1$ be the universal constant appearing in~\cref{prop:conditioning_lipschitz}. 
    If we denote by $\lambda_0 = \int_0^1 \lambda_{\min}(\KK^1[\mu_0, \xx_{\mu_0}])$ and consider the radius $R = \min \left\{ 1, \frac{1}{2 N C_1} \lambda_0 e^{- C_1 (\sqrt{\Ee_2(\mu_0)}+1)^2}  \right\}$ then we have that for every $\mu \in B(\mu_0, R)$, $\Ee_2(\mu) \leq (\sqrt{\Ee_2(\mu_0)}+1)^2$ and hence by the local Lipschitz property:
\begin{align*}
    \int_0^1 \lambda_{\min}(\KK^1[\mu, \xx_{\mu}]) \geq \frac{\lambda_0}{2} .
\end{align*}
Then, as a consequence of~\cref{conditioning_PL}, we obtain that $L$ satisfies a $(R, m)$-P-\L{} property around $\mu_0$ with $m = N^{-1} e^{-{C_2}} \lambda_0$ and $C_2 = C_2(\Ee_2(\mu_0))$ is a constant depending on $\mu_0$.
Combined with~\cref{thm:local_convergence} we obtain that the condition in~\cref{convergence_condition} is sufficient for the gradient flow initialized at $\mu_0$ to converge to a global minimizer of the risk.

Note that by~\cref{prop:conditioning_lipschitz,conditioning_PL,activation_assumption} we can take the constant $C$ in~\cref{convergence_condition} to be of the form $C = C_3 e^{-C_3 \Ee_2(\mu_0)}$ for some constant $C_3$.
\end{proof}

As one can see in the previous condition, the better the conditioning of the kernel matrix, the better the constants in the local P-\L{} property, and hence the easier it is to satisfy the condition for convergence. This conditioning will depend on the choice of activation and initialization and it is important to keep in mind that the P-\L{} property is not expected to hold around any initialization. For example, there is a saddle at every initialization $\mu_0$  with feature distribution $\mu_0^2 = \delta_{(w,b) = 0}$ whenever $\activation(0) = \activation'(0) = 0$.
In the following, we provide examples of activations $\activation$ and initializations $\mu_0$ for which the kernel matrix is well-conditioned. 

\paragraph{Identity (or \emph{FixUp}) initialization}

It will be particularly convenient to consider initial parameterization of the form $\mu_0 = \Leb([0,1]) \otimes\delta_0 \otimes \mu_0^2$ for some $\mu_0^2 \in \Pp_2(\RR^d \times \RR)$, i.e. parameterization whose disintegration $\mu_{0}(.|s) = \delta_0 \otimes \mu_0^2$ is independent of $s \in [0, 1]$ and has support in $\lbrace 0 \rbrace \times \RR^d \times \RR$.
Such an initialization has been proposed for ResNets in~\cite{zhang2018fixup} and is shown to be associated with robust training and good generalization performances. Moreover, note that such an initialization is particularly natural for \NODEs{}: in this case $F_{\mu_0}$ is identically $0$ and the associated \NODE{} flow is the identity. As a consequence the kernel matrix $\KK^1[\mu_0]$ is independent of $s$ and can be expressed as the block matrix:
\begin{align*}
    \KK^1[\mu_0] = \left( k^1[\mu_0](x^i, x^j) \right)_{1 \leq i,j \leq N}, 
\end{align*}
only depending on the feature distribution $\mu_0^2$ and on the input data distribution $\Dd_x = \frac{1}{N} \sum_{i=1}^N \delta_{x^i}$.

\paragraph{Data separation and strictly positive kernel}

In order to ensure the positivity of the kernel matrix $\KK_{\mu_0}^1$ we will assume that the input data $(x^i)_{1 \leq i \leq N}$ are \emph{separated}, that is there exists some strictly positive \emph{data separation} $\delta > 0$ such that:
\begin{align} \label{data_separation}
    \min_{i \neq j} \| x^i - x^j \| \defeq \delta > 0.
\end{align}
Note that, when sampling $N$ random points uniformly in the unit ball of $\RR^d$, we can expect the data separation to scale with the number of points as $\delta \sim N^{-1/d}$.
For a kernel $k$, the property of having its associated kernel matrix being (strictly) positive on every separated point cloud is a property we refer to as \emph{strict positivity}~\parencite{sriperumbudur2011universality}.
In general, the feature distribution $\mu^2$ having dense support is a sufficient condition to ensure strict positivity. The following proposition is a direct consequence of~\cite[Thm.III.4]{sun2019random} and~\cite[Cor.4.3]{carmeli2010vector}.

\begin{prop} \label{prop:dense_support}
    Assume $\activation$ has linear growth and is not a polynomial. Then if the feature distribution $\mu^2 \in \Pp_2(\RR^d \times \RR)$ has dense support in $\RR^d \times \RR$, the kernel $k^1[\mu^2]$ is strictly positive.
\end{prop}

\paragraph{Positively homogeneous activation with uniform distribution of the features on the sphere}

The kernel $k^1[\mu]$ has been particularly studied in the case of a positively homogeneous activation $\sigma$~\parencite{cho2009kernel,bach2017equivalence}. Motivated by applications in machine learning, a popular choice for such activation is the \emph{Rectified Linear Unit} ($\relu$):
\begin{align*}
    \relu : x \mapsto \max \lbrace x, 0 \rbrace
\end{align*}
However, for $\activation = \relu$, the associated feature map $\fmap$ would only satisfy~\cref{fmap_assumption1,fmap_assumption2} and the only choice of positively homogeneous $\activation$ satisfying~\cref{activation_assumption} would be the trivial choice $\activation = \Id$.

Nonetheless, whatever the choice of activation $\sigma$, \cref{k1} still defines a positive kernel $k^1_\mu$ over $\RR^d$. Properties of this kernel in the case where $\activation$ is a positively homogeneous activation have been extensively investigated in the literature. In the case of $\activation = \relu$ the previous~\cref{prop:dense_support} can be improved thanks to the homogeneity of the activation:

\begin{prop} \label{prop:relu_dense_support}
    Assume $\activation = \relu$. Then if the feature distribution $\mu^2 \in \Pp_2(\RR^d \times \RR)$ has dense support in the sphere $\SS^d$, the associated kernel $k^1[\mu^2]$ is strictly positive.
\end{prop}

\begin{proof}
    The result is a direct application of~\cite[Prop.III.5]{sun2019random} and~\cite[Cor.4.3]{carmeli2010vector}.
\end{proof}

\paragraph{Trigonometric activation with strictly positive feature distribution}

An important case is also the choice of the trigonometric activation $\activation = \cos$ for which, considering $\mu \in \Pp_2(\Om)$, \cref{measure_parameterization} gives:
\begin{align*}
    \forall x \in \RR^d, \quad F_\mu(x) = \int_{\RR^d \times \RR^d \times \RR} u \cos(w^\top x+ b) \d \mu(u,w,b),
\end{align*}
and the definition $k^1[\mu]$ in~\cref{k1} gives:
\begin{align*}
    \forall x,y \in \RR^d, \quad k^1[\mu](x,y) = \int_{\RR^d \times \RR} \cos(w^\top x+b) \cos(w^\top y +b) \d \mu^2(w,b) .
\end{align*}
In the case where $\mu^2 = \mu^w \otimes \Uu([0, \pi])$ for some probability measure $\mu^w \in \Pp_2(\RR^d)$ this last expression can be simplified into:
\begin{align} \label{k1_fourier}
    k^1[\mu](x,y) = \frac{1}{2} \int_{\RR^d} \cos(w^\top(x-y)) \d \mu^w(w) . 
\end{align}
That is $k^1[\mu]$ is a positive translation-invariant kernel over $\RR^d$ whose Fourier Transform is $\mu^w$. It is a well-known theorem of Bochner (see~\cite[Thm.6.6]{wendland2004scattered}) that having a non-negative Fourier Transform is a necessary and sufficient condition for a continuous function to define a positive translation-invariant kernel. Moreover, for some initial feature distributions, lower bounds on the conditioning of the kernel matrix as a function of the data separation are given in~\cite{schaback1995error}.

\begin{cor}
    Let $\fmap$ be of the form~\cref{fmap_SHL} with activation $\activation = \cos$. Assume the input data points $\lbrace x^i \rbrace_{1 \leq i \leq N}$ are located in the ball $B_{\RR^d}(0,R)$ of radius $R > 0$ and have separation $\delta \eqdef \min_{i\neq j} \| x^i - x^j \| > 0$. Consider the initialization $\mu_0 = \Leb([0,1]) \otimes \mu$ for some weight distribution $\mu \in \Pp_2(\Om)$. Then the assumptions of~\cref{thm:convergence_SHL} are satisfied if:
    \begin{itemize}
        \item \underline{Sobolev / Mat{\'e}rn kernel} $\mu = \delta_0 \otimes \mu^w \otimes \Uu([0, \pi])$ with $\mu^w(w) \propto (1+\|w\|^2)^{-\nu}$ for some $\nu > d/2 + 2$ and $L(\mu_0) < C^{-1} N^{-3} \delta^{6 (\nu - d/2)}$, for some constant $C = C(R, \nu, d)$.
        
        \item \underline{Gaussian kernel} $\mu = \delta_0 \otimes \mu^w \otimes \Uu([0, \pi])$ with $\mu^w(w) \propto \exp(-\frac{\| w \|^2}{2 \rho^2})$ for some $\rho > 0$ and $L(\mu_0) < C^{-1} N^{-3} \delta^{-3d} e^{-C\delta^{-2}}$, for some constant $C = C(R, \rho, d)$.

        \item \underline{Random features:} Finally assume $\mu_0 = \Leb([0,1]) \otimes \Hat{\mu}$ where $\Hat{\mu} = M^{-1} \sum_{i=1}^M \delta_{(u_i, w_i, b_i)}$ and $(u_i, w_i, b_i)$ are sampled i.i.d. from a distribution $\mu \in \Pp_2(\Om)$ s.t. $\Leb([0,1]) \otimes \mu$ satisfies the assumptions of~\cref{thm:convergence_SHL}. Then for every $\epsilon > 0$ there exists $M_\epsilon \geq 0$ s.t. the assumptions of~\cref{thm:convergence_SHL} are satisfied with probability greater than $1-\epsilon$ (over the sampling of $\lbrace (u_i, w_i, b_i) \rbrace_{1 \leq i \leq M}$) whenever $M \geq M_\epsilon$.
    \end{itemize}
\end{cor}

\begin{proof}
    This is a consequence of results on the conditioning of translation-invariant kernels of the form~\cref{k1_fourier}. 
    \begin{itemize}
        \item \underline{Sobolev / Mat{\'e}rn kernel:} using~\cref{k1_fourier} the RKHS associated to $k^1[\mu]$ corresponds to the Sobolev space $H^\nu(\RR^d)$ and~\cite{schaback1995error} gives that there exists a constant $C = C(\eps, d)$ s.t.
        $\lambda_{\min}(\KK^1[\mu,\xx]) \geq C^{-1} \delta^{2\nu - d}$.

        \item \underline{Gaussian kernel:} by~\cref{k1_fourier}, $k^1[\mu]$ is the Gaussian kernel $k^1[\mu](x,y) = e^{-\frac{1}{2} \rho^2 \|x-y\|^2}$ and~\cite{schaback1995error} gives that there exists a constant $C = C(\rho, d)$ s.t.
        $\lambda_{\min}(\KK^1[\mu,\xx]) \geq C^{-1}  \delta^{-d} e^{- C \delta^{-2}}$.

        \item \underline{Random features:} the assumptions of~\cref{thm:convergence_SHL} are satisfied with high probability when $M$ tends to infinity as all the involved quantities in~\cref{convergence_condition} are continuous w.r.t. the weight distribution $\mu \in \Pp_2(\Om)$.
    \end{itemize}
\end{proof}

\begin{rem}
    In this section we have leveraged the conditioning of the kernel matrix $\KK^1$, that is the square norm of the gradient w.r.t. the outer weights $u$, to show a Polyak-\L{}ojasiewicz inequality holds along the gradient flow. One might ask to what extent the kernel $K^2$ which takes into account the norm of the gradient w.r.t. the weights $(w, b)$ might help improve our convergence result.
    In fact, this kernel plays a negligible role in our analysis for the following reasons:
    
    We consider a ``Fixup'' initialization where the outer weights $u$ are initialized to $0$ at every layer. 
 Initially proposed in \cite{zhang2018fixup}, this kind of initialization is shown to have favorable properties when training ResNets without normalization layer.
 Observing that $K^2$ is quadratic w.r.t. $u$, we have in this case that $K^2 = 0$ and $\partial_t K^2 = 0$ at $t = 0$. 
   Thus, the kernel $K^2$ can only significantly improve the convergence result for large times in the gradient flow and cannot provide us with a good condition number at the beginning of the flow.
   
In addition, following the lines of Proposition 4.2, one could show the kernel matrix $\KK^2$ (defined analogously as the kernel matrices $\KK$ and $\KK^1$) is locally Lipschitz w.r.t.
$\mu$ with a Lipschitz constant scaling linearly with $N$, under additional mild hypotheses on the measure $\mu$. 
Moreover, Theorem~4.2 ensures that during gradient flow the weight distribution will stay in a ball of radius $R \simeq \lambda_0 / N$ around the weight distribution at initialization. Thus $\lambda_{\min}(\KK^2[\mu])$ will be at most of order $\lambda_0$, which is the same order as $\lambda_{\min}(\KK^1[\mu])$.

As a consequence of these two arguments, the local convergence result cannot be explained by the kernel $K^2$.
\end{rem}

\subsection{Satisfying the convergence condition by lifting and scaling}

The previously derived condition for convergence of the gradient flow notably asks for the loss at initialization to be sufficiently low. We conclude the present work by showing how this condition can always be enforced, that is how, for a given training dataset, one can modify the ResNet architecture in a way such that the convergence conditions are satisfied.
The modification we propose is inspired by the work of~\textcite{chizat2019lazy} and consists in embedding the data in a higher dimensional space and performing a rescaling.

As before, we consider an empirical data distribution $\Dd = \frac{1}{N} \sum_{i=1}^N \delta_{x^i,y^i}$, with data points $(x^i, y^i) \in \RR^d \times \RR^{d'}$.
Consider also respectively the \emph{embedding} and \emph{projection} matrices:
\begin{align*}
    A \eqdef (\Id_d, 0_{d, d'})^\top \in \RR^{(d+d') \times d}, \quad B \eqdef (0_{d', d}, \Id_{d'}) \in \RR^{d' \times (d+d')}.
\end{align*}
Using the matrix $A$ we embed the input variables $x^i \in \RR^d$ in the space $\RR^{d+d'}$ by defining $z^i \eqdef A x^i$. We then consider the \NODE{} model of~\cref{def:NODE} with inputs $(z^i)_{1 \leq i \leq N}$ and the feature map $\fmap$ of the form~\cref{fmap_SHL} with some activation $\activation$. In this case we have $\Om = \RR^{d+d'} \times \RR^{d+d'} \times \RR$ and for $(u,w,b) \in \Om$ and $z \in \RR^{d+d'}$,
\begin{align*}
    \fmap((u,w,b), x) = u \activation (w^\top x +b).
\end{align*}
For an input $z^i = A x^i$ and a parameterization $\mu \in \Pp_2^\Leb([0,1] \times \Om)$ we denote by $z^i_\mu$ the associated flow defined by~\cref{forward_weak}.
Also, for a rescaling factor $\alpha > 0$ we consider the modified loss function $\ell^\alpha$ defined by:
\begin{align*}
    \forall (z,y) \in \RR^{d+d'} \times \RR^{d'}, \quad \ell^\alpha(z,y) \eqdef \frac{1}{2} \| \alpha B z - y \|^2.
\end{align*}
We consider training the parameter $\mu \in \Pp_2^\Leb([0,1] \times \Om)$ by performing gradient flow for the risk $L^\alpha$ defined as:
\begin{align*}
    L^\alpha(\mu) \eqdef \frac{1}{N} \sum_{i=1}^N \ell^\alpha(z^i_\mu(1), y^i) .
\end{align*}
Note that, due to its definition, the loss $\ell^\alpha$ satisfies the P-\L{} inequality $\| \nabla_x \ell^\alpha (x,y) \|^2 \geq 2 \alpha^2 \ell(x,y)$. Thus, analogously to~\cref{conditioning_PL}, we obtain the following P-\L{} inequality for $L^\alpha$:
\begin{align} \label{augmented_conditioning_PL}
    \left| \nabla L^\alpha \right|^2(\mu) \geq \frac{2 \alpha^2 e^{-C}}{N} \left( \int_0^1 \lambda_{\min} (\KK[\mu(.|s), \zz_\mu(s) ]) \d s \right) L^\alpha(\mu) ,
\end{align}
where $\zz_\mu$ is the point cloud $(z^i_\mu)_{1 \leq i \leq N}$, the kernel matrix $\KK$ is defined by~\cref{K_matrix} and $C = C(\Ee_2(\mu))$ is a constant depending on $\mu$. Together with~\cref{thm:convergence_SHL}, the above inequality implies that gradient flow converges to a minimizer of the risk whenever $\alpha$ is sufficiently big.

\begin{prop}
    Let $\mu_0 \in \Pp_2^\Leb([0,1] \times \Om)$ be some initialization of the form $\mu_0 = \Leb([0,1]) \otimes \delta_0 \otimes \mu_0^2$ for some $\mu_0^2 \in \Pp_2(\RR^{d+d'+1})$ such that:
    \begin{align}
        \lambda_0 \eqdef \lambda_{\min}(\KK^1[\mu_0^2, \zz]) > 0,
    \end{align}
    where $\KK^1$ is defined in~\cref{k1}. Then there exists $\alpha_0 > 0$ s.t. if $\alpha > \alpha_0$ then the gradient flow initialized at $\mu_0$ converges to a global minimizer of $L^\alpha$. 
\end{prop}

\begin{proof}
    By~\cref{prop:conditioning_lipschitz}, the map $\mu \mapsto \int_0^1 \lambda_{\min}(\KK^1[\mu, \zz_\mu])$ is locally Lipschitz and using~\cref{augmented_conditioning_PL} and that $\lambda_0 > 0$ a local P-\L{} inequality is satisfied around $\mu_0$. Then note that, as at initialization $L^\alpha(\mu_0) = N^{-1} \sum_{i=1}^N \| y^i \|^2$ is independent of $\alpha$ and as increasing $\alpha$ increases the P-\L{} in~\cref{augmented_conditioning_PL}, the convergence condition in~\cref{general_convergence_condition} is necessarily satisfied for $\alpha$ sufficiently large.
\end{proof}

\section{Conclusion}

We studied in this paper the convergence of the gradient flow optimization scheme for the training of a ResNet model of infinite depth and arbitrary width. Our model is a Neural ODE parameterized by measures on the product set $[0,1] \times \Om$ with the constraint to have its marginal on $[0,1]$ to be the Lebesgue measure. Introducing the framework of Conditional Optimal Transport, we provided this set of parameters with a metric structure modeling the metric implicitly used when training ResNets with automatic differentiation. Leveraging results from the theory of gradient flows in metric spaces~\parencite{ambrosio2008gradient}, this property of the Conditional OT metric allowed us to show that the gradient flow equation, derived formally by adjoint sensitivity analysis and used for the training of Neural ODEs, is equivalent to a curve of maximal slope of the risk. The well-posedness of the gradient flow equation follows from this result. Finally, we showed that convergence of mean field models of ResNet can be proved by using a Polyak-\L{}ojasiewicz inequality. This inequality is satisfied locally around well-chosen initializations for which the residuals have sufficiently (but possibly finitely) many features, ensuring their expressivity. As a consequence, assuming the risk is already sufficiently small at those initializations, the gradient flow provably converges to a global minimizer of the risk with an exponential convergence rate. This is the first result of this type for mean-field models of ResNets with unregularized risk as previous works only showed results of optimality under the assumption of convergence.
Finally, for practical examples of architectures and parameter initialization, we quantified explicitly the convergence condition as a function of the number of data points.

We point out some limitations and possible extensions of our work:
\begin{itemize}
    \item We only consider training our ResNet model with gradient flow, an idealized version of the more realistic gradient descent algorithm. An extension of our results to the case of optimization with gradient descent should probably hold due to the local P-\L{} property of the risk. However, this could be at the cost of considering stronger regularity assumptions on the feature map $\fmap$ to ensure the gradient descent dynamic stays close to the gradient flow dynamic for small step sizes.
    \item We make regularity assumptions on the feature map $\fmap$ that might be improved on. In particular, \cref{fmap_assumption3} assumes $\fmap$ to be at least continuously differentiable which does not allow us to consider SHL residuals with $\relu$ activations. This assumption might be weakened, for example by using the recently introduced notion of \emph{conservative gradient}~\parencite{bolte2021nonsmooth}.
    \item We only considered in our convergence analysis of~\cref{sec:convergence} the case of an empirical data distribution $\Dd = \frac{1}{N} \sum_{i=1}^N \delta_{(x^i,y^i)}$. This assumption is crucial as the P-\L{} constant in~\cref{conditioning_PL} scales as $N^{-1}$ and become degenerate for large $N$. It would therefore be interesting to extend our analysis to the case of a data distribution with density.
    \item An important feature of our convergence analysis is to only leverage information about the gradient w.r.t. the outer weights of the residuals (denoted by the variable $u$) to obtain the Polyak-\L{}ojasiewicz inequality. In doing so, we are unable to provide information about the behavior of the feature distribution during training and unable to ensure that gradient flow will escape the ``kernel regime''. 
    Therefore, leveraging information about the feature dynamic to improve our convergence proof is an exciting perspective.
\end{itemize}

\clearpage

\printbibliography

\clearpage

\appendix

\section{Well-posedness of the gradient flow equation for SHL residuals} \label{sec:appendix}

We justify here why the existence and uniqueness results of~\cref{thm:existence_curve,thm:uniqueness_curve} still apply in the case of the SHL architecture where $\fmap$ is given by~\cref{fmap_SHL}, even if~\cref{fmap_assumption_existence,fmap_assumption_uniqueness} are not satisfied. The idea is to restrict ourselves to compactly supported parameterizations $\mu \in \Pp_2^\Leb([0,1] \times \Om)$ where both assumptions are satisfied $\d \mu$ almost everywhere.

In the rest of this section, we consider the parameter space $\Om = \RR^d \times \RR^d \times \RR$ and the feature map is supposed to be given by~\cref{fmap_SHL} with some activation $\activation$ satisfying~\cref{activation_assumption}. Note in particular that~\cref{fmap_assumption1,fmap_assumption2,fmap_assumption3} are satisfied and that the representation result of~\cref{prop:representation_curve} holds. The following preliminary result states that if the initialization $\mu_0$ is compactly supported, so is a solution $\mu_t$ of the gradient flow at every time $t \geq 0$.

\begin{lem} \label{lem:compact_support}
    Assume $\fmap$ if of the form~\cref{fmap_SHL} with some activation $\activation$ satisfying~\cref{activation_assumption}. Let $\mu_0 \in \Pp_2^\Leb([0,1] \times \Om)$ be some compactly supported initialization with $\Supp(\mu_0) \subset B(0,R_0)$ for some $R_0 \geq 0$. If $(\mu_t)_{t \geq 0}$ is a gradient flow of the risk $L$ then for every $T \geq 0$ there exist $R_T \geq 0$ such that:
    \begin{align}
        \forall t \in [0,T], \quad \Supp(\mu_t) \subset B(0, R_T).
    \end{align}
\end{lem}

\begin{proof}
    Let $(\mu_t)_{t \geq 0}$ be such as in the statement. In view of~\cref{prop:representation_curve} such a gradient flow is given for every $t \geq 0$ by $\mu_t = (X_t)_\# \mu_0$, where $X$ is a solution of~\cref{gradient_flow_characteristic}. Let us then consider some $T \geq 0$. The energy $\Ee_2(\mu_t)$ is a continuous function along the gradient flow time $t$ so that $\Ee \eqdef \sup_{t \in [0,T]} \Ee_2(\mu_t) < \infty$. Then using~\cref{activation_assumption} we have a constant $C = C(\Ee)$ such that for every $t \in [0, T]$ and every $(s,\om) \in [0, 1] \times \Om$:
    \begin{align*}
        \| \dd{t} X_t(s,\om) \| \leq C(1+\| X_t(s,\om) \|).
    \end{align*}
    Hence by Grönwall's lemma:
    \begin{align*}
        \| X_t(s,\om) \| \leq e^{C t}(\| X_0(s,\om) \| + Ct),
    \end{align*}
    from which the result follows by taking $R_T = e^{CT}(R_0 + CT)$.
\end{proof}

Note that $\fmap$ may not satisfy~\cref{fmap_assumption_uniqueness} when considering $\om \in \Om$ but it does when considering $\om$ in bounded regions $B(0,R) \subset \Om$. Hence using the above~~\cref{lem:compact_support} and restricting ourselves to finite time intervals, one can show uniqueness of the gradient flow curves whenever the initialization is compactly supported.

\begin{prop} \label{prop:uniqueness_SHL}
    Assume $\fmap$ if of the form~\cref{fmap_SHL} with some activation $\activation$ satisfying~\cref{activation_assumption}. Let $\mu_0 \in \Pp_2^\Leb([0,1] \times \Om)$ be some compactly supported initialization. Then the gradient flow $(\mu_t)_{t \geq 0}$ of the risk $L$ starting from $\mu_0$, if it exists, is unique.
\end{prop}

\begin{proof}
    Let $(\mu_t)_{t \geq 0}, (\mu_t')_{t \geq 0}$ be two gradient flow curves starting from $\mu_0$. We will proceed to show that $d(\mu_t, \mu_t') = 0$ for every $t \geq 0$.

    Fix some $T \geq 0$. By the above~\cref{lem:compact_support}, we can find some $R \geq 0$ such that for every $t \in [0,T]$ we have $\Supp(\mu_t), \Supp(\mu_t') \subset B(0,R)$. Then note that $\fmap$ satisfies~\cref{fmap_assumption_uniqueness} when restricted to $\om \in B(0,R)$ in the sense that for every compact set $K \subset \RR^d$ there exists a constant $C = C(K,R)$ s.t.  for every $x,x' \in K$ and $\om, \om' \in B(0,R)$:
    \begin{align*}
        \| \D^2_{\om,\om} \fmap(\om,x) \| \leq C, \quad \| \D^2_{\om,x} \fmap (\om,x) \| \leq C (1+\|\om\|), \quad \| \D^2_{x,x} \fmap(\om, x) \| \leq C(1+\|\om\|^2)
    \end{align*}
    Also note that for every $t \in [0,T]$ and every optimal Conditional OT coupling $\gamma_t \in \Gamma^\diag_o(\mu_t, \mu_t')$ we have that $\gamma_t$ is compactly supported with $\Supp(\gamma_t) \subset B(0,R) \times B(0,R)$. Hence proceeding as in the proof of~\cref{thm:uniqueness_curve} (cf.~\cref{gradient_lipschitz}) we find a constant $C = C(R)$ s.t. for $\d t$-a.e. $t \in [0,T]$:
    \begin{align*}
        \dd{t} d(\mu_t, \mu_t')^2 \leq C d(\mu_t, \mu_t')^2 ,
    \end{align*}
    from which it follows by Grönwall's inequality that $d(\mu_t, \mu_t')^2 \leq e^{C t} d(\mu_0, \mu_0)^2 = 0$ 
\end{proof}

Finally, the following result states the existence of a gradient flow curve of the risk $L$ for the SHL architecture when the initialization is compactly supported.

\begin{prop} \label{prop:existence_SHL}
    Assume $\fmap$ if of the form~\cref{fmap_SHL} with some activation $\activation$ satisfying~\cref{activation_assumption}. Let $\mu_0 \in \Pp_2^\Leb([0,1] \times \Om)$ be some compactly supported initialization. Then there exists a gradient flow $(\mu_t)_{t \geq 0}$, defined for every $t \geq 0$, for the risk $L$ starting at $\mu_0$.
\end{prop}

\begin{proof}
    Let $\mu_0$ be such as in the statement and consider some $T_0 \geq 0$. Then by the previous~\cref{prop:uniqueness_SHL}, the gradient flow $(\mu_t)_{t \geq 0}$ starting from $\mu_0$, if it exists, is unique and by~\cref{lem:compact_support} there exists a $R_{T_0} > 0$ such $\Supp(\mu_t) \subset B(0,R_{T_0})$ for every $t \in [0,T_0]$.
    
    It is then easy to modify $\fmap$ into some $\Tilde{\fmap}$ such that $\fmap(\om, x) = \Tilde{\fmap}(\om,x)$ whenever $\om \in B(0,2R_{T_0})$, $\Tilde{\fmap}$ satisfies~\cref{fmap_assumption1,fmap_assumption2,fmap_assumption3} but also~\cref{fmap_assumption_existence}. For example, consider for every $x \in \RR^d$ and $(u,w,b) \in \Om$:
    \begin{align*}
        \Tilde{\fmap}((u,w,b),x) \eqdef \pi(u) \activation (w^\top x +b),
    \end{align*}
    for some smooth map $\pi : \RR^d \to \RR^d$ (depending on $R_{T_0}$) such that $\| \pi \|$ and $\| \D \pi \|$ are uniformly bounded and $\pi(u) = u$ if $\| u \| < 2 R_{T_0}$. 
    Note $\Tilde{L}$ the modified risk associated to the modified feature map $\Tilde{\fmap}$.
    Then~\cref{thm:existence_curve} applies and there exists a gradient flow $(\Tilde{\mu}_t)_{t \geq 0}$ for the modified risk $\Tilde{L}$ starting from $\mu_0$. Consider the the time $T$ defined by:
    \begin{align*}
        T^* = \sup \lbrace T \geq 0, \, \Supp(\Tilde{\mu}_t) \subset B(0, 2 R_{T_0}) \, \forall t \in [0,T] \rbrace.
    \end{align*}
    Note that by the definition of $T^*$ and $\Tilde{\fmap}$, if $T < T^*$ then for every $t \in [0,T]$, $\nabla \Tilde{L}[\Tilde{\mu}_t] = \nabla L[\Tilde{\mu}_t]$ in $L^2(\Tilde{\mu}_t)$ and hence $(\Tilde{\mu}_t)_{t \in [0,T]}$ is a gradient flow for the original risk $L$, starting from $\mu_0$. 
    We show by contradiction that $T^* > T_0$, implying that there exists a gradient flow for $L$ starting from $\mu_0$ and defined up to time $T_0$.

    Assume $T^* \leq T_0$. Consider $\Ee \eqdef \sup_{t \in [0, T_0+1]} \Ee_2(\Tilde{\mu_t})$. For any $T < T^*$, we have that $(\Tilde{\mu}_t)_{t \in [0, T]}$ is a gradient flow for $L$ starting from $\mu_0$ and in particular $\Supp(\Tilde{\mu}_T) \subset B(0,R_{T_0})$. But then, reasoning as in the proof of~\cref{lem:compact_support},, there exists a constant $C = C(\Ee)$ (independent of $T$) such that for every $\eps \in [0,1]$:
    \begin{align*}
        \forall t \in [T, T+\eps], \quad \Supp(\Tilde{\mu}_t) \subset B(0, e^{C \eps}(R_{T_0} + C \eps)),
    \end{align*}
    which is included in $B(0,2 R_{T_0})$ for $\eps$ sufficiently small. Hence, chosing $T$ sufficiently close from $T^*$, we get a $T+\eps > T^*$ such that $\Supp(\tilde{\mu}_t) \subset B(0,2R_{T_0})$ for every $t \in [0, T+\eps]$.
    This is in contradiction with the definition of $T^*$.
\end{proof}

\end{document}